\documentclass[10pt]{article}
\usepackage[T1]{fontenc}
\usepackage[latin9]{inputenc}
\usepackage{geometry}
\usepackage{array}
\usepackage{float}
\usepackage{amsmath}
\usepackage{amsthm}
\usepackage{amssymb}
\usepackage{graphicx}

\makeatletter

\providecommand{\tabularnewline}{\\}
\floatstyle{ruled}
\newfloat{algorithm}{tbp}{loa}
\providecommand{\algorithmname}{Algorithm}
\floatname{algorithm}{\protect\algorithmname}

\newcommand{\lyxaddress}[1]{
\par {\raggedright #1
\vspace{1.4em}
\noindent\par}
}
  \theoremstyle{definition}
  \newtheorem{defn}{\protect\definitionname}
\theoremstyle{plain}
\newtheorem{thm}{\protect\theoremname}
  \theoremstyle{plain}
  \newtheorem{lem}{\protect\lemmaname}
  \theoremstyle{remark}
  \newtheorem{rem}{\protect\remarkname}
  \theoremstyle{plain}
  \newtheorem{ax}{\protect\axiomname}
 \theoremstyle{definition}
  \newtheorem{example}{\protect\examplename}
  \theoremstyle{definition}
  \newtheorem{problem}{\protect\problemname}

\@ifundefined{date}{}{\date{}}
\usepackage{url}

\makeatother

  \providecommand{\axiomname}{Axiom}
  \providecommand{\definitionname}{Definition}
  \providecommand{\examplename}{Example}
  \providecommand{\lemmaname}{Lemma}
  \providecommand{\problemname}{Problem}
  \providecommand{\remarkname}{Remark}
\providecommand{\theoremname}{Theorem}

\begin{document}

\title{A Differential Privacy Mechanism Design\\Under Matrix-Valued Query}

\author{Thee Chanyaswad\quad{}Alex Dytso\quad{}H. Vincent Poor\quad{}Prateek
Mittal}
\maketitle

\lyxaddress{\begin{center}
Princeton University\\Princeton, NJ, USA
\par\end{center}}
\begin{abstract}
Traditionally, differential privacy mechanism design has been tailored
for a scalar-valued query function. Although many mechanisms such
as the Laplace and Gaussian mechanisms can be extended to a matrix-valued
query function by adding i.i.d. noise to each element of the matrix,
this method is often sub-optimal as it forfeits an opportunity to
exploit the structural characteristics typically associated with matrix
analysis. In this work, we consider the design of differential privacy
mechanism specifically for a matrix-valued query function. The proposed
solution is to utilize a matrix-variate noise, as opposed to the traditional
scalar-valued noise. Particularly, we propose a novel differential
privacy mechanism called the \emph{Matrix-Variate Gaussian (MVG) mechanism},
which adds a \emph{matrix-valued} noise drawn from a matrix-variate
Gaussian distribution. We prove that the MVG mechanism preserves $(\epsilon,\delta)$-differential
privacy, and show that it allows the structural characteristics of
the matrix-valued query function to naturally be exploited. Furthermore,
due to the multi-dimensional nature of the MVG mechanism and the matrix-valued
query, we introduce the concept of \emph{directional noise}, which
can be utilized to mitigate the impact the noise has on the utility
of the query. Finally, we demonstrate the performance of the MVG mechanism
and the advantages of directional noise using three matrix-valued
queries on three privacy-sensitive datasets. We find that the MVG
mechanism notably outperforms four previous state-of-the-art approaches,
and provides comparable utility to the non-private baseline. Our work
thus presents a promising prospect for both future research and implementation
of differential privacy for matrix-valued query functions.
\end{abstract}
{\let\thefootnote\relax\footnotetext{This is the full version of the work, so parts of the material in this paper are derived from the conference version of this work \cite{chanyaswad2018mvg}.}}

\section{Introduction}

Differential privacy \cite{RefWorks:186,RefWorks:195,RefWorks:151}
has become the gold standard for a rigorous privacy guarantee, and
there has been the development of many differentially-private mechanisms.
Some popular mechanisms include the classical Laplace mechanism \cite{RefWorks:195}
and the Exponential mechanism \cite{RefWorks:192}. In addition, there
are other mechanisms that build upon these two classical ones such
as those based on data partition and aggregation \cite{RefWorks:224,RefWorks:241,RefWorks:242,RefWorks:243,RefWorks:244,RefWorks:245,RefWorks:246,RefWorks:247,RefWorks:248,RefWorks:219,RefWorks:193},
and those based on adaptive queries \cite{RefWorks:236,RefWorks:221,RefWorks:316,RefWorks:317,RefWorks:318,RefWorks:319,RefWorks:175}.
From this observation, differentially-private mechanisms may be categorized
into two groups: the basic mechanisms, and the derived mechanisms.
The basic mechanisms' privacy guarantee is self contained, whereas
the derived mechanisms' privacy guarantee is achieved through a combination
of basic mechanisms, composition theorems, and the post-processing
invariance property \cite{RefWorks:185}.

In this work, we consider the design of a \emph{basic mechanism} for
\emph{matrix-valued query functions}. Existing basic mechanisms for
differential privacy are designed usually for scalar-valued query
functions. However, in many practical settings, the query functions
are multi-dimensional and can be succinctly represented as matrix-valued
functions. Examples of matrix-valued query functions in the real-world
applications include the covariance matrix \cite{RefWorks:249,RefWorks:194,RefWorks:178},
the kernel matrix \cite{RefWorks:33}, the adjacency matrix \cite{RefWorks:329},
the incidence matrix \cite{RefWorks:329}, the rotation matrix \cite{RefWorks:346},
the Hessian matrix \cite{RefWorks:347}, the transition matrix \cite{RefWorks:348},
and the density matrix \cite{RefWorks:349}, which find applications
in statistics \cite{RefWorks:350}, machine learning \cite{RefWorks:376},
graph theory \cite{RefWorks:329}, differential equations \cite{RefWorks:347},
computer graphics \cite{RefWorks:346}, probability theory \cite{RefWorks:348},
quantum mechanics \cite{RefWorks:349}, and many other fields \cite{RefWorks:328}.

One property that distinguishes the matrix-valued query functions
from the scalar-valued query functions is the relationship and interconnection
among the elements of the matrix. One may naively treat these matrices
as merely a collection of scalar values, but that could prove sub-optimal
since the \emph{structure} and \emph{relationship} among these scalar
values are often informative and essential to the understanding and
analysis of the system. For example, in graph theory, the adjacency
matrix is symmetric for an undirected graph, but not for a directed
graph \cite{RefWorks:329} \textendash{} an observation which is implausible
to extract from simply looking at the collection of elements without
considering how they are arranged in the matrix.

In differential privacy, the traditional method for dealing with a
matrix-valued query function is to extend a scalar-valued mechanism
by adding \emph{independent and identically distributed} (i.i.d.)
noise to each element of the matrix \cite{RefWorks:195,RefWorks:186,RefWorks:220}.
However, this method fails to utilize the structural characteristics
of the matrix-valued noise and query function. Although some advanced
methods have explored this possibility in an iterative/procedural
manner \cite{RefWorks:236,RefWorks:221}, the structural characteristics
of the matrices are still largely under-investigated. This is partly
due to the lack of a basic mechanism that is directly designed for
matrix-valued query functions, making the utilization of matrix structures
and application of available tools in matrix analysis challenging.

In this work, we formalize the study of the matrix-valued differential
privacy, and present a new basic mechanism that can readily exploit
the structural characteristics of the matrices \textendash{} the \emph{Matrix-Variate
Gaussian (MVG) mechanism}. The high-level concept of the MVG mechanism
is simple \textendash{} it adds a matrix-variate Gaussian noise scaled
to the $L_{2}$-sensitivity of the matrix-valued query function (cf.
Figure \ref{fig:mvg_schematic}). We rigorously prove that the MVG
mechanism guarantees $(\epsilon,\delta)$-differential privacy, and
show that, with the MVG mechanism, the structural characteristics
of the matrix-valued query functions can readily be incorporated into
the mechanism design. Specifically, we present an example of how the
MVG mechanism can yield greater utility by exploiting the positive-semi
definiteness of the matrix-valued query function. Moreover, due to
the multi-dimensional nature of the noise and the query function,
the MVG mechanism allows flexibility in the design via the novel notion
of \emph{directional noise}. An important consequence of the concept
of directional noise is that the matrix-valued noise in the MVG mechanism
can be devised to affect certain parts of the matrix-valued query
function less than the others, while providing \emph{the same} privacy
guarantee. In practice, this property could be advantageous as the
noise can be tailored to have minimal impact on the intended utility.
We present simple algorithms to incorporate the directional noise
into the differential privacy mechanism design, and theoretically
present the optimal design for the MVG mechanism with directional
noise that maximizes the power-to-noise ratio of the mechanism output.

\begin{figure}
\begin{centering}
\includegraphics[scale=0.5]{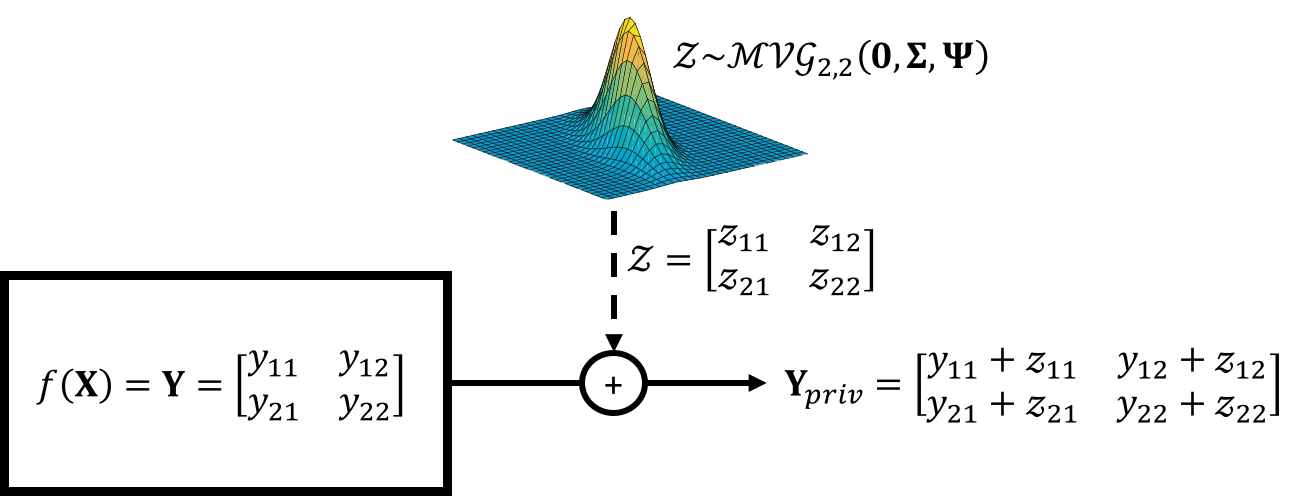} 
\par\end{centering}
\caption{Given a matrix-valued query function $f(\mathbf{X})\in\mathbb{R}^{m\times n}$,
the MVG mechanism adds a \emph{matrix-valued noise} drawn from the
\emph{matrix-variate Gaussian distribution} $\mathcal{MVG}_{m,n}(\mathbf{0},\boldsymbol{\Sigma},\boldsymbol{\Psi})$
to guarantee $(\epsilon,\delta)$-differential privacy. The schematic
shows an example when $m=n=2$. \label{fig:mvg_schematic}}
\end{figure}

Finally, to illustrate the effectiveness of the MVG mechanism, we
conduct experiments on three privacy-sensitive real-world datasets
\textendash{} Liver Disorders \cite{RefWorks:322,RefWorks:413}, Movement
Prediction \cite{RefWorks:396}, and Cardiotocography \cite{RefWorks:322,RefWorks:412}.
The experiments include three tasks involving matrix-valued query
functions \textendash{} regression, finding the first principal component,
and covariance estimation. The results show that the MVG mechanism
can evidently outperform four prior state-of-the-art mechanisms \textendash{}
the Laplace mechanism, the Gaussian mechanism, the Exponential mechanism,
and the JL transform \textendash{} in utility in all experiments,
and can provide the utility similar to that achieved with the non-private
methods, while guaranteeing differential privacy.

To summarize, the main contributions are as follows. 
\begin{itemize}
\item We formalize the study of matrix-valued query functions in differential
privacy and introduce the novel \emph{Matrix-Variate Gaussian (MVG)
mechanism}. 
\item We rigorously prove that the MVG mechanism guarantees $(\epsilon,\delta)$-differential
privacy.
\item We show that exploiting the structural characteristic of the matrix-valued
query function can improve the utility performance of the MVG mechanism.
\item We introduce a novel concept of \emph{directional noise}, and propose
two simple algorithms to implement this novel concept with the MVG
mechanism. 
\item We theoretically exhibit how the directional noise can be devised
to provide the maximum utility from the MVG mechanism.
\item We evaluate our approach on three real-world datasets and show that
our approach can outperform four prior state-of-the-art mechanisms
in all experiments, and yields utility performance close to the non-private
baseline. 
\end{itemize}

\section{Prior Works \label{sec:Prior-Works}}

Existing mechanisms for differential privacy may be categorized into
two types: the \emph{basic} \cite{RefWorks:195,RefWorks:186,RefWorks:313,RefWorks:192,RefWorks:220,RefWorks:277,RefWorks:397,RefWorks:398,RefWorks:399};
and the \emph{derived mechanisms} \cite{RefWorks:193,RefWorks:401,RefWorks:224,RefWorks:241,RefWorks:242,RefWorks:243,RefWorks:244,RefWorks:245,RefWorks:246,RefWorks:247,RefWorks:248,RefWorks:219,RefWorks:253,RefWorks:187,RefWorks:221,RefWorks:236,RefWorks:237,RefWorks:314,pca-gauss,RefWorks:400,RefWorks:337,RefWorks:339,RefWorks:316,RefWorks:317,RefWorks:236,RefWorks:175,RefWorks:319}.
Since our work concerns the basic mechanism design, we focus our discussion
on this type, and provide a general overview of the other.

\subsection{Basic Mechanisms \label{subsec:Basic-Mechanisms}}

Basic mechanisms are those whose privacy guarantee is self-contained,
i.e. it does not deduce the guarantee from another mechanism. Here,
we discuss four popular existing basic mechanisms: the Laplace mechanism,
the Gaussian mechanism, the Johnson-Lindenstrauss transform method,
and the Exponential mechanism.

\subsubsection{Laplace Mechanism}

The classical \emph{Laplace mechanism} \cite{RefWorks:195} adds noise
drawn from the Laplace distribution scaled to the $L_{1}$-sensitivity
of the query function. It was initially designed for a scalar-valued
query function, but can be extended to a matrix-valued query function
by adding i.i.d. Laplace noise to each element of the matrix. The
Laplace mechanism provides the strong $\epsilon$-differential privacy
guarantee and is relatively simple to implement. However, its generalization
to a matrix-valued query function does not automatically utilize the
structure of the matrices involved.

\subsubsection{Gaussian Mechanism}

The \emph{Gaussian mechanism} \cite{RefWorks:220,RefWorks:186,RefWorks:277}
uses i.i.d. additive noise drawn from the Gaussian distribution scaled
to the $L_{2}$-sensitivity of the query function. The Gaussian mechanism
guarantees $(\epsilon,\delta)$-differential privacy. It suffers from
the same limitation as the Laplace mechanism when extended to a matrix-valued
query function, i.e. it does not automatically consider the structure
of the matrices.

\subsubsection{Johnson-Lindenstrauss Transform}

The \emph{Johnson-Lindenstrauss (JL) transform method} \cite{RefWorks:313}
uses multiplicative noise to guarantee $(\epsilon,\delta)$-differential
privacy. It is, in fact, a rare basic mechanism designed for a matrix-valued
query function. Despite its promise, previous works show that the
JL transform method can be applied to queries with certain properties
only, as we discuss here. 
\begin{itemize}
\item Blocki et al. \cite{RefWorks:313} use a random matrix, whose entries
are drawn i.i.d. from a Gaussian distribution, and the method is applicable
to the Laplacian of a graph and the covariance matrix. 
\item Blum and Roth \cite{RefWorks:397} use a hash function that implicitly
represents the JL transform, and the method is suitable for a sparse
query.
\item Upadhyay \cite{RefWorks:398,RefWorks:399} uses a multiplicative combination
of random matrices to provide a JL transform that is applicable to
any matrix-valued query function whose singular values are all above
a minimum threshold. 
\end{itemize}
Among these methods, Upadhyay's works \cite{RefWorks:398,RefWorks:399}
stand out as possibly the most general. In our experiments, we show
that our approach can yield higher utility for the same privacy budget
than these methods.

\subsubsection{Exponential Mechanism}

In contrast to additive and multiplicative noise used in previous
approaches, the \emph{Exponential mechanism} uses noise introduced
via the sampling process \cite{RefWorks:192}. The Exponential mechanism
draws its query answers from a custom probability density function
designed to preserve $\epsilon$-differential privacy. To provide
reasonable utility, the Exponential mechanism designs its sampling
distribution based on the \emph{quality function}, which indicates
the utility score of each possible sample. Due to its generality,
the Exponential mechanism has been utilized for many types of query
functions, including the matrix-valued query functions. We experimentally
compare our approach to the Exponential mechanism, and show that,
with slightly weaker privacy guarantee, our method can yield significant
utility improvement.

Finally, we conclude that our method differs from the four existing
basic mechanisms as follows. In contrast with the i.i.d. noise in
the Laplace and Gaussian mechanisms, the MVG mechanism allows a non-i.i.d.
noise (cf. Section \ref{sec:Directional-Noise}). As opposed to the
multiplicative noise in the JL transform and the sampling noise in
the Exponential mechanism, the MVG mechanism uses an additive noise
for matrix-valued query functions.

\subsection{Derived Mechanisms}

Derived mechanisms are those whose privacy guarantee is deduced from
other basic mechanisms via the composition theorems and the post-processing
invariance property \cite{RefWorks:185}. Derived mechanisms are often
designed to provide better utility by exploiting some properties of
the query function or of the data. Blocki et al. \cite{RefWorks:313}
also define a similar categorization with the term ``revised algorithm''.

The general techniques used by derived mechanisms are often translatable
among basic mechanisms, including our MVG mechanism. Given our focus
on a novel basic mechanism, these techniques are less relevant to
our work, and we leave the investigation of integrating them into
the MVG framework in future work. Some of the popular techniques used
by derived mechanisms are summarized here.

\subsubsection{Sensitivity Control}

This technique avoids the worst-case sensitivity in basic mechanisms
by using variant concepts of sensitivity. Examples include the \emph{smooth
sensitivity framework} \cite{RefWorks:193} and \emph{elastic sensitivity}
\cite{RefWorks:401}.

\subsubsection{Data Partition and Aggregation}

This technique uses data partition and aggregation to produce more
accurate query answers \cite{RefWorks:224,RefWorks:241,RefWorks:242,RefWorks:243,RefWorks:244,RefWorks:245,RefWorks:246,RefWorks:247,RefWorks:248,RefWorks:219}.
The partition and aggregation processes are done in a differentially-private
manner either via the composition theorems and the post-processing
invariance property \cite{RefWorks:185}, or with a small extra privacy
cost. Hay et al. \cite{RefWorks:253} nicely summarize many works
that utilize this concept.

\subsubsection{Non-uniform Data Weighting}

This technique lowers the level of perturbation required for the privacy
protection by \emph{weighting each data sample or dataset differently}
\cite{RefWorks:187,RefWorks:221,RefWorks:236,RefWorks:237}. The rationale
is that each sample in a dataset, or each instance of the dataset
itself, has a heterogeneous contribution to the query output. Therefore,
these mechanisms place a higher weight on the critical samples or
instances of the database to provide better utility.

\subsubsection{Data Compression}

This approach reduces the level of perturbation required for differential
privacy via \emph{dimensionality reduction}. Various dimensionality
reduction methods have been proposed. For example, Kenthapadi et al.
\cite{RefWorks:314}, Xu et al. \cite{RefWorks:400}, and Li et al.
\cite{RefWorks:337} use random projection; Chanyaswad et al. \cite{pca-gauss}
and Jiang et al. \cite{RefWorks:339} use principal component analysis
(PCA); Xiao et al. \cite{RefWorks:224} use wavelet transform; and
Acs et al. \cite{RefWorks:219} use lossy Fourier transform.

\subsubsection{Adaptive Queries}

The derived mechanisms based on adaptive queries use prior and/or
auxiliary information to improve the utility of the query answers.
Examples include the matrix mechanism \cite{RefWorks:316,RefWorks:317},
the multiplicative weights mechanism \cite{RefWorks:236,RefWorks:221},
the low-rank mechanism \cite{RefWorks:318}, boosting \cite{RefWorks:175},
and the sparse vector technique \cite{RefWorks:220,RefWorks:319}.

Finally, we conclude with three main observations. First, the MVG
mechanism falls into the category of basic mechanism. Second, techniques
used in derived mechanisms are generally applicable to multiple basic
mechanisms, including our novel MVG mechanism. Third, therefore, for
fair comparison, we will compare the MVG mechanism with the four state-of-the-art
basic mechanisms presented in this section.

\section{Background}

We begin with a discussion of basic concepts pertaining to the MVG
mechanism for matrix-valued query.

\subsection{Matrix-Valued Query \label{subsec:Matrix-Valued-Query}}

In our analysis, we use the term \emph{dataset} interchangeably with
database, and represent it with the matrix $\mathbf{X}$. The matrix-valued
query function, $f(\mathbf{X})\in\mathbb{R}^{m\times n}$, has $m$
rows and $n$ columns. We define the notion of neighboring datasets
$\{\mathbf{X}_{1},\mathbf{X}_{2}\}$ as two datasets that differ by
a single record, and denote it as $d(\mathbf{X}_{1},\mathbf{X}_{2})=1$.
We note, however, that although the neighboring datasets differ by
only a single record, $f(\mathbf{X}_{1})$ and $f(\mathbf{X}_{2})$
may differ in every element.

We denote a matrix-valued random variable with the calligraphic font,
e.g. $\mathcal{Z}$, and its instance with the bold font, e.g. $\mathbf{Z}$.
Finally, as will become relevant later, we use the columns of $\mathbf{X}$
to denote the records (samples) in the dataset.

\subsection{$(\epsilon,\delta)$-Differential Privacy}

In the paradigm of data privacy, differential privacy \cite{RefWorks:151,RefWorks:186}
provides a rigorous privacy guarantee, and has been widely adopted
in the community \cite{RefWorks:185}. Differential privacy guarantees
that the involvement of any one particular record of the dataset would
not drastically change the query answer. 
\begin{defn}
\label{def:differential_privacy}A mechanism $\mathcal{A}$ on a query
function $f(\cdot)$ is $(\epsilon,\delta)$- differentially-private
if for all neighboring datasets $\{\mathbf{X}_{1},\mathbf{X}_{2}\}$,
and for all possible measurable matrix-valued outputs $\mathbf{Y}\subseteq\mathbb{R}^{m\times n}$,
\[
\Pr[\mathcal{A}(f(\mathbf{X}_{1}))\in\mathbf{Y}]\leq e^{\epsilon}\Pr[\mathcal{A}(f(\mathbf{X}_{2}))\in\mathbf{Y}]+\delta.
\]
\end{defn}

\subsection{Matrix-Variate Gaussian Distribution}

One of our main innovations is the use of the noise drawn from a matrix-variate
probability distribution. More specifically, in the MVG mechanism,
the additive noise is drawn from the matrix-variate Gaussian distribution,
defined as follows \cite{RefWorks:279,RefWorks:280,RefWorks:281,RefWorks:282,RefWorks:284,RefWorks:367}. 
\begin{defn}
\label{def:tmvg_dist}An $m\times n$ matrix-valued random variable
$\mathcal{X}$ has a matrix-variate Gaussian distribution, denoted
as $\mathcal{MVG}_{m,n}(\mathbf{M},\boldsymbol{\Sigma},\boldsymbol{\Psi})$,
if it has the probability density function: 
\[
p_{\mathcal{X}}(\mathbf{X})=\frac{\exp\{-\frac{1}{2}\mathrm{tr}[\boldsymbol{\Psi}^{-1}(\mathbf{X}-\mathbf{M})^{T}\boldsymbol{\Sigma}^{-1}(\mathbf{X}-\mathbf{M})]\}}{(2\pi)^{mn/2}\left|\boldsymbol{\Psi}\right|^{m/2}\left|\boldsymbol{\Sigma}\right|^{n/2}},
\]
where $\mathrm{tr}(\cdot)$ is the matrix trace \cite{RefWorks:208},
$\left|\cdot\right|$ is the matrix determinant \cite{RefWorks:208},
$\mathbf{M}\in\mathbb{R}^{m\times n}$ is the mean, $\boldsymbol{\Sigma}\in\mathbb{R}^{m\times m}$
is the row-wise covariance, and $\boldsymbol{\Psi}\in\mathbb{R}^{n\times n}$
is the column-wise covariance. 
\end{defn}
Noticeably, the probability density function (pdf) of $\mathcal{MVG}_{m,n}(\mathbf{M},\boldsymbol{\Sigma},\boldsymbol{\Psi})$
looks similar to that of the $m$-dimensional multivariate Gaussian
distribution, $\mathcal{N}_{m}(\boldsymbol{\mu},\boldsymbol{\Sigma})$.
Indeed, $\mathcal{MVG}_{m,n}(\mathbf{M},\boldsymbol{\Sigma},\boldsymbol{\Psi})$
is a generalization of $\mathcal{N}_{m}(\boldsymbol{\mu},\boldsymbol{\Sigma})$
to a matrix-valued random variable. This leads to a few notable additions.
First, the mean vector $\boldsymbol{\mu}$ now becomes the mean matrix
$\mathbf{M}$. Second, in addition to the traditional row-wise covariance
matrix $\boldsymbol{\Sigma}$, there is also the column-wise covariance
matrix $\boldsymbol{\Psi}$. The latter addition is due to the fact
that, not only could the rows of the matrix be distributed non-uniformly,
but also could its columns.

We may intuitively explain this addition as follows. If we draw $n$
i.i.d. samples from $\mathcal{N}_{m}(\mathbf{0},\boldsymbol{\Sigma})$
denoted as $\mathbf{x}_{1},\mathbf{x}_{2},\ldots,\mathbf{x}_{n}\in\mathbb{R}^{m}$,
and concatenate them into a matrix $\mathbf{X}=[\mathbf{x}_{1},\mathbf{x}_{2},\ldots,\mathbf{x}_{n}]\in\mathbb{R}^{m\times n}$,
then, it can be shown that $\mathbf{X}$ is drawn from $\mathcal{MVG}_{m,n}(\mathbf{0},\boldsymbol{\Sigma},\mathbf{I})$,
where $\mathbf{I}$ is the identity matrix \cite{RefWorks:279}. However,
if we consider the case when the columns of $\mathbf{X}$ are not
i.i.d., and are distributed with the covariance $\boldsymbol{\Psi}$
instead, then, it can be shown that this is distributed according
to $\mathcal{MVG}_{m,n}(\mathbf{0},\boldsymbol{\Sigma},\boldsymbol{\Psi})$
\cite{RefWorks:279}.

\subsection{Relevant Matrix Algebra Theorems}

We recite major theorems in matrix algebra that are essential to the
subsequent analysis and discussion as follows.
\begin{thm}[Singular value decomposition (SVD) \cite{RefWorks:208}]
\label{thm:svd}A matrix $\mathbf{A}\in\mathbb{R}^{m\times n}$ can
be decomposed into two unitary matrices $\mathbf{W}_{1}\in\mathbb{R}^{m\times m},\mathbf{W}_{2}\in\mathbb{R}^{n\times n}$,
and a diagonal matrix $\boldsymbol{\Lambda}$, whose diagonal elements
are ordered non-increasingly downward. These diagonal elements are
the \emph{singular values} of $\mathbf{A}$ denoted as $\sigma_{1}\geq\sigma_{2}\geq\cdots\geq0$,
and $\mathbf{A}=\mathbf{W}_{1}\boldsymbol{\Lambda}\mathbf{W}_{2}^{T}$. 
\end{thm}
\begin{lem}[Laurent-Massart \cite{RefWorks:369}]
\label{lem:laurent_massart} For a matrix-variate random variable
$\mathcal{N}\sim\mathcal{MVG}_{m,n}(\mathbf{0},\mathbf{I}_{m},\mathbf{I}_{n})$,
$\delta\in[0,1]$, and $\zeta(\delta)=2\sqrt{-mn\ln\delta}-2\ln\delta+mn$,
the following inequality holds: 
\[
\Pr[\left\Vert \mathcal{N}\right\Vert _{F}^{2}\leq\zeta(\delta)^{2}]\geq1-\delta,
\]
where $\left\Vert \cdot\right\Vert _{F}$ is the Frobenius norm of
a matrix.
\end{lem}
\begin{lem}[Merikoski-Sarria-Tarazaga \cite{RefWorks:288}]
\label{lem:singular_bound} The non-increasingly ordered singular
values of a matrix $\mathbf{A}\in\mathbb{R}^{m\times n}$ have the
values of 
\[
0\leq\sigma_{i}\leq\frac{\left\Vert \mathbf{A}\right\Vert _{F}}{\sqrt{i}},
\]
where $\left\Vert \cdot\right\Vert _{F}$ is the Frobenius norm of
a matrix.
\end{lem}
\begin{lem}[von Neumann \cite{RefWorks:402}]
\label{lem:v_neumann}Let $\mathbf{A},\mathbf{B}\in\mathbb{R}^{m\times n}$,
and let $\sigma_{i}(\mathbf{A})$ and $\sigma_{i}(\mathbf{B})$ be
the non-increasingly ordered singular values of $\mathbf{A}$ and
$\mathbf{B}$, respectively. Then
\[
\mathrm{tr}(\mathbf{A}\mathbf{B}^{T})\leq\Sigma_{i=1}^{r}\sigma_{i}(\mathbf{A})\sigma_{i}(\mathbf{B}),
\]
where $r=\min\{m,n\}$. 
\end{lem}
\begin{lem}[Trace magnitude bound \cite{RefWorks:278}]
\label{lem:abs_trace_bound}Let $\mathbf{A}\in\mathbb{R}^{m\times n}$,
and let $\sigma_{i}(\mathbf{A})$ be the non-increasingly ordered
singular values of $\mathbf{A}$. Then
\[
\left|\mathrm{tr}(\mathbf{A})\right|\leq\sum_{i=1}^{r}\sigma_{i}(\mathbf{A}),
\]
where $r=\min\{m,n\}$.
\end{lem}
\begin{lem}[Hadamard's inequality \cite{RefWorks:473}]
 \label{lem:hadamard_ineq} Let $\mathbf{A}\in\mathbb{R}^{m\times m}$
be non-singular and $\mathbf{a}_{i}$ be the $i^{th}$ column vector
of $\mathbf{A}$ . Then, 
\[
\left|\mathbf{A}\right|^{2}\leq\prod_{i=1}^{m}\left\Vert \mathbf{a}_{i}\right\Vert ^{2},
\]
where $\left|\cdot\right|$ is the matrix determinant \cite{RefWorks:208},
and $\left\Vert \cdot\right\Vert $ is the $L_{2}$-norm of a vector.
\end{lem}

\section{MVG Mechanism: Differential Privacy with Matrix-Valued Query}

Matrix-valued query functions are different from their scalar counterparts
in terms of the vital information contained in how the elements are
arranged in the matrix. To fully exploit these structural characteristics
of matrix-valued query functions, we present a novel mechanism for
matrix-valued query functions: the \emph{Matrix-Variate Gaussian (MVG)
mechanism}.

\subsection{Definitions}

First, let us introduce the sensitivity of the matrix-valued query
function used in the MVG mechanism. 
\begin{defn}[Sensitivity]
\label{def:sensitivity}Given a matrix-valued query function $f(\mathbf{X})\in\mathbb{R}^{m\times n}$,
define the $L_{2}$-sensitivity as, 
\[
s_{2}(f)=\sup_{d(\mathbf{X}_{1},\mathbf{X}_{2})=1}\left\Vert f(\mathbf{X}_{1})-f(\mathbf{X}_{2})\right\Vert _{F},
\]
where $\left\Vert \cdot\right\Vert _{F}$ is the Frobenius norm \cite{RefWorks:208}. 
\end{defn}
Then, we present the MVG mechanism as follows. 
\begin{defn}[MVG mechanism]
\label{def:mvg_mech}Given a matrix-valued query function $f(\mathbf{X})\in\mathbb{R}^{m\times n}$,
and a matrix-valued random variable $\mathcal{Z}\sim\mathcal{MVG}_{m,n}(\mathbf{0},\boldsymbol{\Sigma},\boldsymbol{\Psi})$,
the \emph{MVG mechanism} is defined as, 
\[
\mathcal{MVG}(f(\mathbf{X}))=f(\mathbf{X})+\mathcal{Z},
\]
where $\boldsymbol{\Sigma}$ is the row-wise covariance matrix, and
$\boldsymbol{\Psi}$ is the column-wise covariance matrix. 
\end{defn}
Note that so far, we have not specified how to pick $\boldsymbol{\Sigma}$
and $\boldsymbol{\Psi}$ according to the sensitivity $s_{2}(f)$
in the MVG mechanism. We discuss the explicit form of $\boldsymbol{\Sigma}$
and $\boldsymbol{\Psi}$ next.

As the additive matrix-valued noise of the MVG mechanism is drawn
from $\mathcal{MVG}_{m,n}(\mathbf{0},\boldsymbol{\Sigma},\boldsymbol{\Psi})$,
the parameters to be designed for the mechanism are the covariance
matrices $\boldsymbol{\Sigma}$ and $\boldsymbol{\Psi}$. In the following
discussion, we derive the sufficient conditions on $\boldsymbol{\Sigma}$
and $\boldsymbol{\Psi}$ such that the MVG mechanism preserves $(\epsilon,\delta)$-differential
privacy. Furthermore, since one of the motivations of the MVG mechanism
is to facilitate the exploitation of the structural characteristics
of the matrix-value query, we demonstrate how the structural knowledge
about the matrix-value query can improve the sufficient condition
for the MVG mechanism. The term \emph{improve} here certainly requires
further specification, and we provide such clarification when the
context is appropriate later in this section.

Hence, the subsequent discussion proceeds as follows. First, we present
a sufficient condition for the values of $\boldsymbol{\Sigma}$ and
$\boldsymbol{\Psi}$ to ensure that the MVG mechanism preserves $(\epsilon,\delta)$-differential
privacy \emph{without assuming structural knowledge} about the matrix-valued
query. Second, we present an alternative sufficient condition for
$\boldsymbol{\Sigma}$ and $\boldsymbol{\Psi}$ to ensure that the
MVG mechanism preserves $(\epsilon,\delta)$-differential privacy
with the assumption that the matrix-value query is \emph{symmetric
positive semi-definite (PSD)}. Finally, we rigorously prove that,
by incorporating the knowledge of positive semi-definiteness about
the matrix-value query, we improve the sufficient condition to guarantee
$(\epsilon,\delta)$-differential privacy with the MVG mechanism.

\subsection{Differential Privacy Analysis for General Matrix-Valued Query (No
Structural Assumption) \label{subsec:Differential-Privacy-Analysis_general}}

\begin{table}
\begin{centering}
\begin{tabular}{|>{\centering}m{3.5cm}|>{\centering}m{10cm}|}
\hline 
{}{}$\mathbf{X}$  & {}{}database/dataset whose columns are data records and rows are
attributes/features.\tabularnewline
\hline 
{}{}$\mathcal{MVG}_{m,n}(\mathbf{0},\boldsymbol{\Sigma},\boldsymbol{\Psi})$  & {}{} $m\times n$ matrix-variate Gaussian distribution with zero
mean, the row-wise covariance $\boldsymbol{\Sigma}$, and the column-wise
covariance $\boldsymbol{\Psi}$.\tabularnewline
\hline 
{}{}$f(\mathbf{X})\in\mathbb{R}^{m\times n}$  & {}{}matrix-valued query function\tabularnewline
\hline 
 {}{}$r$  & {}{}$\min\{m,n\}$\tabularnewline
\hline 
 {}{}$H_{r}$  & {}{}generalized harmonic numbers of order $r$\tabularnewline
\hline 
{}{}$H_{r,1/2}$  & {}{}generalized harmonic numbers of order $r$ of $1/2$\tabularnewline
\hline 
 {}{}$\gamma$  & {}{}$\sup_{\mathbf{X}}\left\Vert f(\mathbf{X})\right\Vert _{F}$\tabularnewline
\hline 
 {}{}$\zeta(\delta)$  & {}{}$2\sqrt{-mn\ln\delta}-2\ln\delta+mn$\tabularnewline
\hline 
{}{}$\boldsymbol{\sigma}(\boldsymbol{\Sigma}^{-1})$  & {}{}vector of non-increasing singular values of $\boldsymbol{\Sigma}^{-1}$ \tabularnewline
\hline 
{}{}$\boldsymbol{\sigma}(\boldsymbol{\Psi}^{-1})$  & {}{}vector of non-increasing singular values of $\boldsymbol{\Psi}^{-1}$ \tabularnewline
\hline 
\end{tabular}
\par\end{centering}
\caption{Notations for the differential privacy analysis. \label{tab:Notations}}
\end{table}

First, we consider the most general differential privacy analysis
of the MVG mechanism. More specifically, we do not make explicit structural
assumption about the matrix-query function in this analysis. The following
theorem presents the key result under this general setting.
\begin{thm}
\label{thm:design_general} Let $\boldsymbol{\sigma}(\boldsymbol{\Sigma}^{-1})=[\sigma_{1}(\boldsymbol{\Sigma}^{-1}),\ldots,\sigma_{m}(\boldsymbol{\Sigma}^{-1})]^{T}$
and $\boldsymbol{\sigma}(\boldsymbol{\Psi}^{-1})=[\sigma_{1}(\boldsymbol{\Psi}^{-1}),\ldots,\sigma_{n}(\mathbf{\boldsymbol{\Psi}}^{-1})]^{T}$
be the vectors of non-increasingly ordered singular values of $\boldsymbol{\Sigma}^{-1}$
and $\boldsymbol{\Psi}^{-1}$, respectively, and let the relevant
variables be defined according to Table \ref{tab:Notations}. Then,
the MVG mechanism guarantees $(\epsilon,\delta)$-differential privacy
if $\boldsymbol{\Sigma}$ and $\boldsymbol{\Psi}$ satisfy the following
condition, 
\begin{equation}
\left\Vert \boldsymbol{\sigma}(\boldsymbol{\Sigma}^{-1})\right\Vert _{2}\left\Vert \boldsymbol{\sigma}(\boldsymbol{\Psi}^{-1})\right\Vert _{2}\leq\frac{(-\beta+\sqrt{\beta^{2}+8\alpha\epsilon})^{2}}{4\alpha^{2}},\label{eq.sufficient_condition}
\end{equation}
where $\alpha=[H_{r}+H_{r,1/2}]\gamma^{2}+2H_{r}\gamma s_{2}(f)$,
and $\beta=2(mn)^{1/4}\zeta(\delta)H_{r}s_{2}(f)$. 
\end{thm}
\begin{proof}
The MVG mechanism guarantees differential privacy if for every pair
of neighboring datasets $\{\mathbf{X}_{1},\mathbf{X}_{2}\}$ and all
possible measurable sets $\mathbf{S}\subseteq\mathbb{R}^{m\times n}$,
\[
\Pr\left[f(\mathbf{X}_{1})+\mathcal{Z}\in\mathbf{S}\right]\leq\exp(\epsilon)\Pr\left[f(\mathbf{X}_{2})+\mathcal{Z}\in\mathbf{S}\right].
\]
The proof now follows by observing that (cf. Section \ref{subsec:affine_tx}),
\[
\mathcal{Z}=\mathbf{W}_{\boldsymbol{\Sigma}}\boldsymbol{\Lambda}_{\boldsymbol{\Sigma}}^{1/2}\mathcal{N}\boldsymbol{\Lambda}_{\boldsymbol{\Psi}}^{1/2}\mathbf{W}_{\boldsymbol{\Psi}}^{T},
\]
and defining the following events: 
\[
\mathbf{R}_{1}=\{\mathcal{N}:\|\mathcal{N}\|_{F}^{2}\le\zeta(\delta)^{2}\},\,\mathbf{R}_{2}=\{\mathcal{N}:\|\mathcal{N}\|_{F}^{2}>\zeta(\delta)^{2}\},
\]
where $\zeta(\delta)$ is defined in Theorem \ref{lem:laurent_massart}.
Next, observe that 
\begin{align*}
\Pr\left[f(\mathbf{X}_{1})+\mathcal{Z}\in\mathbf{S}\right] & =\Pr\left[\left(\{f(\mathbf{X}_{1})+\mathcal{Z}\in\mathbf{S}\}\cap\mathbf{R}_{1}\right)\cup\left(\{f(\mathbf{X}_{1})+\mathcal{Z}\in\mathbf{S}\}\cap\mathbf{R}_{2}\right)\right]\\
 & \le\Pr\left[\{f(\mathbf{X}_{1})+\mathcal{Z}\in\mathbf{S}\}\cap\mathbf{R}_{1}\right]+\Pr\left[\{f(\mathbf{X}_{1})+\mathcal{Z}\in\mathbf{S}\}\cap\mathbf{R}_{2}\right],
\end{align*}
where the last inequality follows from the union bound. By Theorem
\ref{lem:laurent_massart} and the definition of the set $\mathbf{R}_{2}$,
we have, 
\begin{align*}
\Pr\left[\{f(\mathbf{X}_{1})+\mathcal{Z}\in\mathbf{S}\}\cap\mathbf{R}_{2}\right]\le\Pr\left[\mathbf{R}_{2}\right]=1-\Pr\left[\mathbf{R}_{1}\right]\le\delta.
\end{align*}

In the rest of the proof, we find sufficient conditions for the following
inequality to hold: 
\begin{align*}
\Pr\left[f(\mathbf{X}_{1})+\mathcal{Z}\in(\mathbf{S}\cap\mathbf{R}_{1})\right]\le\exp(\epsilon)\Pr\left[f(\mathbf{X}_{2})+\mathcal{Z}\in\mathbf{S}\right].
\end{align*}
this would complete the proof of differential privacy guarantee.

Using the definition of $\mathcal{MVG}_{m,n}(\mathbf{0},\boldsymbol{\Sigma},\boldsymbol{\Psi})$
(Definition \ref{def:tmvg_dist}), this is satisfied if we have, 
\begin{align*}
\int_{\mathbf{S}\cap\mathbf{R}_{1}}\exp\{-\frac{1}{2}\mathrm{tr}[\boldsymbol{\Psi}^{-1}(\mathbf{Y}-f(\mathbf{X}_{1}))^{T}\boldsymbol{\Sigma}^{-1}(\mathbf{Y}-f(\mathbf{X}_{1}))]\}d\mathbf{Y} & \leq\\
\exp(\epsilon)\cdot\int_{\mathbf{S}\cap\mathbf{R}_{1}}\exp\{-\frac{1}{2}\mathrm{tr}[\boldsymbol{\Psi}^{-1}(\mathbf{Y}-f(\mathbf{X}_{2}))^{T}\boldsymbol{\Sigma}^{-1}(\mathbf{Y}-f(\mathbf{X}_{2}))]\}d\mathbf{Y}.
\end{align*}
By inserting $\frac{\exp\{-\frac{1}{2}\mathrm{tr}[\boldsymbol{\Psi}^{-1}(\mathbf{Y}-f(\mathbf{X}_{2}))^{T}\boldsymbol{\Sigma}^{-1}(\mathbf{Y}-f(\mathbf{X}_{2}))]\}}{\exp\{-\frac{1}{2}\mathrm{tr}[\boldsymbol{\Psi}^{-1}(\mathbf{Y}-f(\mathbf{X}_{2}))^{T}\boldsymbol{\Sigma}^{-1}(\mathbf{Y}-f(\mathbf{X}_{2}))]\}}$
inside the integral on the left side, it suffices to show that 
\[
\frac{\exp\{-\frac{1}{2}\mathrm{tr}[\boldsymbol{\Psi}^{-1}(\mathbf{Y}-f(\mathbf{X}_{1}))^{T}\boldsymbol{\Sigma}^{-1}(\mathbf{Y}-f(\mathbf{X}_{1}))]\}}{\exp\{-\frac{1}{2}\mathrm{tr}[\boldsymbol{\Psi}^{-1}(\mathbf{Y}-f(\mathbf{X}_{2}))^{T}\boldsymbol{\Sigma}^{-1}(\mathbf{Y}-f(\mathbf{X}_{2}))]\}}\leq\exp(\epsilon),
\]
for all $\mathbf{Y}\in\mathbf{S}\cap\mathbf{R}_{1}$. With some algebraic
manipulations, the left hand side of this condition can be expressed
as, 
\begin{align*}
= & \exp\{-\frac{1}{2}\mathrm{tr}[\boldsymbol{\Psi}^{-1}\mathbf{Y}{}^{T}\mathbf{\boldsymbol{\Sigma}}^{-1}(f(\mathbf{X}_{2})-f(\mathbf{X}_{1}))+\boldsymbol{\Psi}^{-1}(f(\mathbf{X}_{2})-f(\mathbf{X}_{1}))^{T}\boldsymbol{\Sigma}^{-1}\mathbf{Y}\\
 & -\boldsymbol{\Psi}^{-1}f(\mathbf{X}_{2})^{T}\boldsymbol{\Sigma}^{-1}f(\mathbf{X}_{2})+\boldsymbol{\Psi}^{-1}f(\mathbf{X}_{1})^{T}\boldsymbol{\Sigma}^{-1}f(\mathbf{X}_{1})]\}\\
= & \exp\{\frac{1}{2}\mathrm{tr}[\boldsymbol{\Psi}^{-1}\mathbf{Y}^{T}\mathbf{\boldsymbol{\Sigma}}^{-1}\boldsymbol{\Delta}+\boldsymbol{\Psi}^{-1}\boldsymbol{\Delta}^{T}\boldsymbol{\Sigma}^{-1}\mathbf{Y}\\
 & +\boldsymbol{\Psi}^{-1}f(\mathbf{X}_{2})^{T}\boldsymbol{\Sigma}^{-1}f(\mathbf{X}_{2})-\boldsymbol{\Psi}^{-1}f(\mathbf{X}_{1})^{T}\boldsymbol{\Sigma}^{-1}f(\mathbf{X}_{1})]\},
\end{align*}
where $\boldsymbol{\Delta}=f(\mathbf{X}_{1})-f(\mathbf{X}_{2})$.
This quantity has to be bounded by $\leq\exp(\epsilon)$, so we present
the following \emph{characteristic equation}, which has to be satisfied
for all possible neighboring $\{\mathbf{X}_{1},\mathbf{X}_{2}\}$
and all $\mathbf{Y}\in\mathbf{S}\cap\mathbf{R}_{1}$, for the MVG
mechanism to guarantee $(\epsilon,\delta)$-differential privacy:
\begin{equation}
\mathrm{tr}[\boldsymbol{\Psi}^{-1}\mathcal{Y}^{T}\boldsymbol{\Sigma}^{-1}\boldsymbol{\Delta}+\boldsymbol{\Psi}^{-1}\boldsymbol{\Delta}^{T}\boldsymbol{\Sigma}^{-1}\mathcal{Y}+\boldsymbol{\Psi}^{-1}f(\mathbf{X}_{2})^{T}\boldsymbol{\Sigma}^{-1}f(\mathbf{X}_{2})-\boldsymbol{\Psi}^{-1}f(\mathbf{X}_{1})^{T}\boldsymbol{\Sigma}^{-1}f(\mathbf{X}_{1})]\leq2\epsilon\label{eq:characteristic_eq}
\end{equation}
Specifically, we want to show that this inequality holds with probability
$1-\delta$.

From the characteristic equation, the proof analyzes the four terms
in the sum separately since the trace is additive.

\emph{The first term}: $\mathrm{tr}[\boldsymbol{\Psi}^{-1}\mathcal{Y}^{T}\boldsymbol{\Sigma}^{-1}\boldsymbol{\Delta}]$.
First, let us denote $\mathcal{Y}=f(\mathbf{X})+\mathcal{Z}$, where
$f(\mathbf{X})$ and $\mathcal{Z}$ are any possible instances of
the query and the noise, respectively. Then, we can rewrite the first
term as, $\mathrm{tr}[\boldsymbol{\Psi}^{-1}f(\mathbf{X})^{T}\boldsymbol{\Sigma}^{-1}\boldsymbol{\Delta}]+\mathrm{tr}[\boldsymbol{\Psi}^{-1}\mathcal{Z}^{T}\boldsymbol{\Sigma}^{-1}\boldsymbol{\Delta}]$.
The earlier part can be bounded from Lemma \ref{lem:v_neumann}: 
\[
\mathrm{tr}[\boldsymbol{\Psi}^{-1}f(\mathbf{X})^{T}\boldsymbol{\Sigma}^{-1}\boldsymbol{\Delta}]\leq\sum_{i=1}^{r}\sigma_{i}(\boldsymbol{\Psi}^{-1}f(\mathbf{X})^{T})\sigma_{i}(\boldsymbol{\Delta}^{T}\boldsymbol{\Sigma}^{-1}).
\]
Lemma \ref{lem:singular_bound} can then be used to bound each singular
value. In more detail, 
\[
\sigma_{i}(\boldsymbol{\Psi}^{-1}f(\mathbf{X})^{T})\leq\frac{\left\Vert \boldsymbol{\Psi}^{-1}f(\mathbf{X})^{T}\right\Vert _{F}}{\sqrt{i}}\leq\frac{\left\Vert \boldsymbol{\Psi}^{-1}\right\Vert _{F}\left\Vert f(\mathbf{X})\right\Vert _{F}}{\sqrt{i}},
\]
where the last inequality is via the sub-multiplicative property of
a matrix norm \cite{RefWorks:290}. It is well-known that $\bigl\Vert\boldsymbol{\Psi}^{-1}\bigr\Vert_{F}=\bigl\Vert\boldsymbol{\sigma}(\boldsymbol{\Psi}^{-1})\bigr\Vert_{2}$
(cf. \cite[p. 342]{RefWorks:208}), and since $\gamma=\sup_{\mathbf{X}}\bigl\Vert f(\mathbf{X})\bigr\Vert_{F}$,
\[
\sigma_{i}(\boldsymbol{\Psi}^{-1}f(\mathbf{X})^{T})\leq\gamma\left\Vert \boldsymbol{\sigma}(\boldsymbol{\Psi}^{-1})\right\Vert _{2}/\sqrt{i}.
\]
Applying the same steps to the other singular value, and using Definition
\ref{def:sensitivity}, we can write, 
\[
\sigma_{i}(\boldsymbol{\Delta}^{T}\boldsymbol{\Sigma}^{-1})\leq s_{2}(f)\left\Vert \boldsymbol{\sigma}(\boldsymbol{\Sigma}^{-1})\right\Vert _{2}/\sqrt{i}.
\]
Substituting the two singular value bounds, the earlier part of the
first term can then be bounded by,
\begin{equation}
\mathrm{tr}[\boldsymbol{\Psi}^{-1}f(\mathbf{X})^{T}\boldsymbol{\Sigma}^{-1}\boldsymbol{\Delta}]\leq\gamma s_{2}(f)H_{r}\left\Vert \boldsymbol{\sigma}(\boldsymbol{\Sigma}^{-1})\right\Vert _{2}\left\Vert \boldsymbol{\sigma}(\boldsymbol{\Psi}^{-1})\right\Vert _{2}.\label{eq:first_term_part1}
\end{equation}

The latter part of the first term is more complicated since it involves
$\mathcal{Z}$, so we will derive the bound in more detail. First,
let us define $\mathcal{N}$ to be drawn from $\mathcal{MVG}_{m,n}(\mathbf{0},\mathbf{I}_{m},\mathbf{I}_{n})$,
so we can write $\mathcal{Z}$ in terms of $\mathcal{N}$ using affine
transformation \cite{RefWorks:279}: $\mathcal{Z}=\mathbf{B}_{\boldsymbol{\Sigma}}\mathcal{N}\mathbf{B}_{\boldsymbol{\Psi}}^{T}$.
To specify $\mathbf{B}_{\boldsymbol{\Sigma}}$ and $\mathbf{B}_{\boldsymbol{\Psi}}$,
we solve the following linear equations, respectively, 
\begin{align*}
\mathbf{B}_{\boldsymbol{\Sigma}}\mathbf{B}_{\boldsymbol{\Sigma}}^{T} & =\boldsymbol{\Sigma};\\
\mathbf{B}_{\boldsymbol{\Psi}}\mathbf{B}_{\boldsymbol{\Psi}}^{T} & =\boldsymbol{\Psi}.
\end{align*}
This can be readily solved with SVD (cf. \cite[p. 440]{RefWorks:208});
hence, $\mathbf{B}_{\boldsymbol{\Sigma}}=\mathbf{W}_{\boldsymbol{\Sigma}}\boldsymbol{\Lambda}_{\boldsymbol{\Sigma}}^{\frac{1}{2}}$,
and $\mathbf{B}_{\boldsymbol{\Psi}}=\mathbf{W}_{\boldsymbol{\Psi}}\boldsymbol{\Lambda}_{\boldsymbol{\Psi}}^{\frac{1}{2}}$,
where $\boldsymbol{\Sigma}=\mathbf{W}_{\boldsymbol{\Sigma}}\boldsymbol{\Lambda}_{\boldsymbol{\Sigma}}\mathbf{W}_{\boldsymbol{\Sigma}}^{T}$,
and $\boldsymbol{\Psi}=\mathbf{W}_{\boldsymbol{\Psi}}\boldsymbol{\Lambda}_{\boldsymbol{\Psi}}\mathbf{W}_{\boldsymbol{\Psi}}^{T}$
from SVD. Therefore, $\mathcal{Z}$ can be written as, 
\[
\mathcal{Z}=\mathbf{W}_{\boldsymbol{\Sigma}}\boldsymbol{\Lambda}_{\boldsymbol{\Sigma}}^{1/2}\mathcal{N}\boldsymbol{\Lambda}_{\boldsymbol{\Psi}}^{1/2}\mathbf{W}_{\boldsymbol{\Psi}}^{T}.
\]
Substituting into the latter part of the first term yields, 
\[
\mathrm{tr}[\boldsymbol{\Psi}^{-1}\mathcal{Z}^{T}\boldsymbol{\Sigma}^{-1}\boldsymbol{\Delta}]=\mathrm{tr}[\mathbf{W}_{\boldsymbol{\Psi}}\boldsymbol{\Lambda}_{\boldsymbol{\Psi}}^{-1/2}\mathcal{N}\boldsymbol{\Lambda}_{\boldsymbol{\Sigma}}^{-1/2}\mathbf{W}_{\boldsymbol{\Sigma}}^{T}\boldsymbol{\Delta}].
\]
This can be bounded by Lemma \ref{lem:v_neumann} as, 
\[
\mathrm{tr}[\boldsymbol{\Psi}^{-1}\mathcal{Z}^{T}\boldsymbol{\Sigma}^{-1}\boldsymbol{\Delta}]\leq\sum_{i=1}^{r}\sigma_{i}(\mathbf{W}_{\boldsymbol{\Psi}}\boldsymbol{\Lambda}_{\boldsymbol{\Psi}}^{-1/2}\mathcal{N}\boldsymbol{\Lambda}_{\boldsymbol{\Sigma}}^{-1/2}\mathbf{W}_{\boldsymbol{\Sigma}}^{T})\sigma_{i}(\boldsymbol{\Delta}).
\]
The two singular values can then be bounded by Lemma \ref{lem:singular_bound}.
For the first singular value, 
\begin{align*}
\sigma_{i}(\mathbf{W}_{\boldsymbol{\Psi}}\boldsymbol{\Lambda}_{\boldsymbol{\Psi}}^{-1/2}\mathcal{N}\boldsymbol{\Lambda}_{\boldsymbol{\Sigma}}^{-1/2}\mathbf{W}_{\boldsymbol{\Sigma}}^{T}) & \leq\frac{\left\Vert \mathbf{W}_{\boldsymbol{\Psi}}\boldsymbol{\Lambda}_{\boldsymbol{\Psi}}^{-1/2}\mathcal{N}\boldsymbol{\Lambda}_{\boldsymbol{\Sigma}}^{-1/2}\mathbf{W}_{\boldsymbol{\Sigma}}^{T}\right\Vert _{F}}{\sqrt{i}}\\
 & \leq\frac{\left\Vert \boldsymbol{\Lambda}_{\boldsymbol{\Sigma}}^{-1/2}\right\Vert _{F}\left\Vert \boldsymbol{\Lambda}_{\boldsymbol{\Psi}}^{-1/2}\right\Vert _{F}\left\Vert \mathcal{N}\right\Vert _{F}}{\sqrt{i}}.
\end{align*}
By definition, $\left\Vert \boldsymbol{\Lambda}_{\boldsymbol{\Sigma}}^{-1/2}\right\Vert _{F}=\sqrt{\mathrm{tr}(\boldsymbol{\Lambda}_{\boldsymbol{\Sigma}}^{-1})}=\left\Vert \boldsymbol{\sigma}(\boldsymbol{\Sigma}^{-1})\right\Vert _{1}^{1/2}$,
where $\left\Vert \cdot\right\Vert _{1}$ is the 1-norm. By norm relation,
$\left\Vert \boldsymbol{\sigma}(\boldsymbol{\Sigma}^{-1})\right\Vert _{1}^{1/2}\leq m^{1/4}\left\Vert \boldsymbol{\sigma}(\boldsymbol{\Sigma}^{-1})\right\Vert _{2}^{1/2}$.
With similar derivation for $\bigl\Vert\boldsymbol{\Lambda}_{\boldsymbol{\Psi}}^{-1/2}\bigr\Vert_{F}$
and with Theorem \ref{lem:laurent_massart}, the singular value can
be bounded with probability $1-\delta$ as, 
\[
\sigma_{i}(\mathbf{W}_{\boldsymbol{\Psi}}\boldsymbol{\Lambda}_{\boldsymbol{\Psi}}^{-\frac{1}{2}}\mathcal{N}\boldsymbol{\Lambda}_{\boldsymbol{\Sigma}}^{-\frac{1}{2}}\mathbf{W}_{\boldsymbol{\Sigma}}^{T})\leq\frac{(mn)^{\frac{1}{4}}\zeta(\delta)\left\Vert \boldsymbol{\sigma}(\boldsymbol{\Sigma}^{-1})\right\Vert _{2}^{\frac{1}{2}}\left\Vert \boldsymbol{\sigma}(\boldsymbol{\Psi}^{-1})\right\Vert _{2}^{\frac{1}{2}}}{\sqrt{i}}.
\]
Meanwhile, the other singular value can be readily bounded with Lemma
\ref{lem:singular_bound} as $\sigma_{i}(\boldsymbol{\Delta})\leq s_{2}(f)/\sqrt{i}$.
Hence, the latter part of the first term is bounded with probability
$\geq1-\delta$ as,
\begin{equation}
\mathrm{tr}[\boldsymbol{\Psi}^{-1}\mathcal{Z}^{T}\boldsymbol{\Sigma}^{-1}\boldsymbol{\Delta}]\leq(mn)^{\frac{1}{4}}\zeta(\delta)H_{r}s_{2}(f)\left\Vert \boldsymbol{\sigma}(\boldsymbol{\Sigma}^{-1})\right\Vert _{2}^{\frac{1}{2}}\left\Vert \boldsymbol{\sigma}(\boldsymbol{\Psi}^{-1})\right\Vert _{2}^{\frac{1}{2}}.\label{eq:first_term_part2}
\end{equation}
 Since the parameter $(\left\Vert \boldsymbol{\sigma}(\boldsymbol{\Sigma}^{-1})\right\Vert _{2}\left\Vert \boldsymbol{\sigma}(\boldsymbol{\Psi}^{-1})\right\Vert _{2})^{1/2}$
appears a lot in the derivation, let us define 
\[
\phi=(\left\Vert \boldsymbol{\sigma}(\boldsymbol{\Sigma}^{-1})\right\Vert _{2}\left\Vert \boldsymbol{\sigma}(\boldsymbol{\Psi}^{-1})\right\Vert _{2})^{1/2}.
\]
Finally, combining Eq. (\ref{eq:first_term_part1}) and (\ref{eq:first_term_part2})
yields the bound for the first term, 
\begin{equation}
\mathrm{tr}[\boldsymbol{\Psi}^{-1}\mathcal{Y}^{T}\boldsymbol{\Sigma}^{-1}\boldsymbol{\Delta}]\leq\gamma H_{r}s_{2}(f)\phi^{2}+(mn)^{1/4}\zeta(\delta)H_{r}s_{2}(f)\phi.\label{eq:first_term_bound}
\end{equation}

\emph{The second term}: $\mathrm{tr}[\boldsymbol{\Psi}^{-1}\boldsymbol{\Delta}^{T}\boldsymbol{\Sigma}^{-1}\mathcal{Y}]$.
By following the same steps as in the first term, it can be shown
that the second term has the exact same bound as the first terms,
i.e. 
\begin{equation}
\mathrm{tr}[\boldsymbol{\Psi}^{-1}\boldsymbol{\Delta}^{T}\boldsymbol{\Sigma}^{-1}\mathcal{Y}]\leq\gamma H_{r}s_{2}(f)\phi^{2}+(mn)^{1/4}\zeta(\delta)H_{r}s_{2}(f)\phi.\label{eq:second_term_bound}
\end{equation}

\emph{The third term}: $\mathrm{tr}[\boldsymbol{\Psi}^{-1}f(\mathbf{X}_{2})^{T}\boldsymbol{\Sigma}^{-1}f(\mathbf{X}_{2})]$.
Applying Lemma \ref{lem:v_neumann} and \ref{lem:singular_bound},
we can readily bound it as, 
\[
\mathrm{tr}[\boldsymbol{\Psi}^{-1}f(\mathbf{X}_{2})^{T}\boldsymbol{\Sigma}^{-1}f(\mathbf{X}_{2})]\leq\gamma^{2}H_{r}\phi^{2}.
\]

\emph{The fourth term}: $-\mathrm{tr}[\boldsymbol{\Psi}^{-1}f(\mathbf{X}_{1})^{T}\boldsymbol{\Sigma}^{-1}f(\mathbf{X}_{1})]$.
Since this term has the negative sign, we consider the absolute value
instead. Using Lemma \ref{lem:abs_trace_bound}, 
\[
\left|\mathrm{tr}[\boldsymbol{\Psi}^{-1}f(\mathbf{X}_{1})^{T}\boldsymbol{\Sigma}^{-1}f(\mathbf{X}_{1})]\right|\leq\sum_{i=1}^{r}\sigma_{i}(\boldsymbol{\Psi}^{-1}f(\mathbf{X}_{1})^{T}\boldsymbol{\Sigma}^{-1}f(\mathbf{X}_{1})).
\]
Then, using the singular value bound in Lemma \ref{lem:singular_bound},
\[
\sigma_{i}(\boldsymbol{\Psi}^{-1}f(\mathbf{X}_{1})^{T}\boldsymbol{\Sigma}^{-1}f(\mathbf{X}_{1}))\leq\frac{\left\Vert \boldsymbol{\Psi}^{-1}\right\Vert _{F}\left\Vert f(\mathbf{X}_{1})\right\Vert _{F}^{2}\left\Vert \boldsymbol{\Sigma}^{-1}\right\Vert _{F}}{\sqrt{i}}.
\]
Hence, the fourth term can be bounded by, 
\[
\left|\mathrm{tr}[\boldsymbol{\Psi}^{-1}f(\mathbf{X}_{1})^{T}\boldsymbol{\Sigma}^{-1}f(\mathbf{X}_{1})]\right|\leq\gamma^{2}H_{r,1/2}\phi^{2}.
\]

\emph{Four terms combined:} by combining the four terms and rearranging
them, the characteristic equation becomes, 
\begin{equation}
\alpha\phi^{2}+\beta\phi\leq2\epsilon.\label{eq:quad_eq_general}
\end{equation}
This is a quadratic equation, of which the solution is $\phi\in[\frac{-\beta-\sqrt{\beta^{2}+8\alpha\epsilon}}{2\alpha},\frac{-\beta+\sqrt{\beta^{2}+8\alpha\epsilon}}{2\alpha}]$.
Since we know $\phi\geq0$, due to the axiom of the norm, we only
have the one-sided solution, 
\[
\phi\leq\frac{-\beta+\sqrt{\beta^{2}+8\alpha\epsilon}}{2\alpha},
\]
which immediately implies the criterion in Theorem \ref{thm:design_general}. 
\end{proof}
\begin{figure}
\begin{centering}
\includegraphics[scale=0.5]{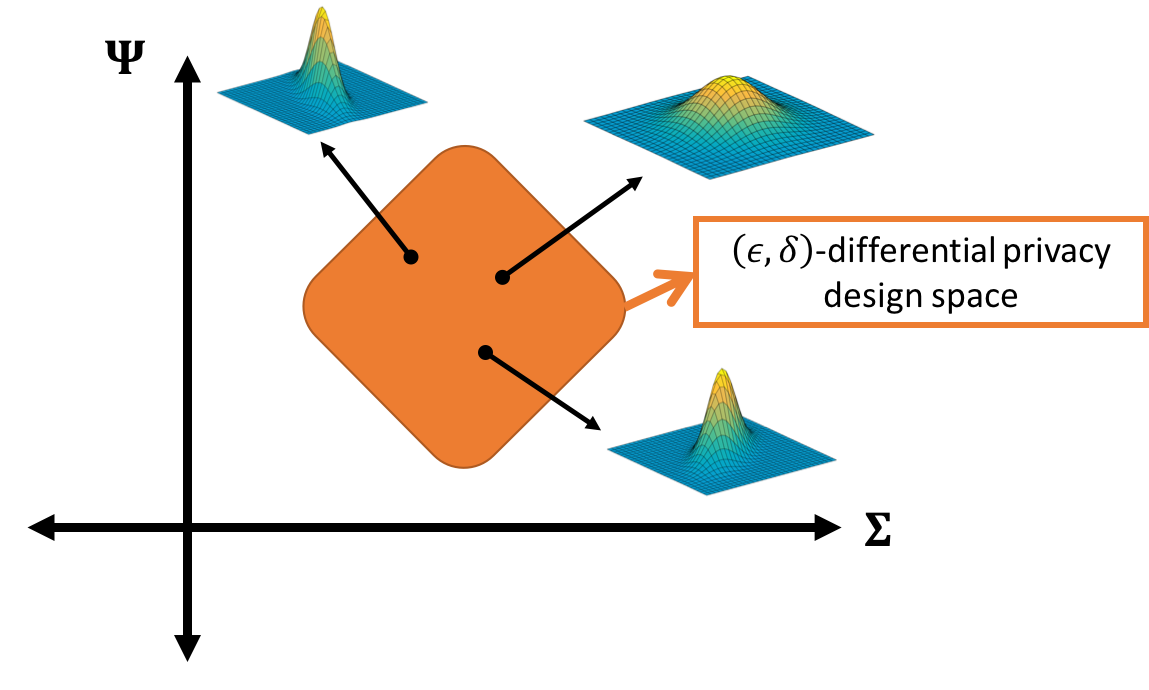} 
\par\end{centering}
\caption{A conceptual display of the MVG design space. The illustration visualizes
the design space coordinated by the two design parameters of $\mathcal{MVG}_{m,n}(\mathbf{0},\boldsymbol{\Sigma},\boldsymbol{\Psi})$.
Each point on the space corresponds to an instance of $\mathcal{MVG}_{m,n}(\mathbf{0},\boldsymbol{\Sigma},\boldsymbol{\Psi})$.
From this perspective, Theorem \ref{thm:design_general} suggests
that any instance of $\mathcal{MVG}_{m,n}(\mathbf{0},\boldsymbol{\Sigma},\boldsymbol{\Psi})$
in the (conceptual) shaded area would preserve $(\epsilon,\delta)$-differential
privacy. \label{fig:A-conceptual-display}}
\end{figure}
\begin{rem}
In Theorem \ref{thm:design_general}, we assume that the Frobenius
norm of the query function is bounded for all possible datasets by
$\gamma$. This assumption is valid in practice because real-world
data are rarely unbounded (cf. \cite{RefWorks:372}), and it is a
common assumption in the analysis of differential privacy for multi-dimensional
query functions (cf. \cite{RefWorks:195,RefWorks:178,RefWorks:338,RefWorks:249}). 
\end{rem}
\begin{rem}
The values of the generalized harmonic numbers \textendash{} $H_{r}$,
and $H_{r,1/2}$ \textendash{} can be obtained from the table lookup
for a given value of $r$, or can easily be computed recursively \cite{RefWorks:289}. 
\end{rem}
The sufficient condition in Theorem \ref{thm:design_general} yields
an important observation: the privacy guarantee by the MVG mechanism
depends \emph{only on the singular values} of $\boldsymbol{\Sigma}$
and $\boldsymbol{\Psi}$ through their norm. In other words, we may
have multiple instances of $\mathcal{MVG}_{m,n}(\mathbf{0},\boldsymbol{\Sigma},\boldsymbol{\Psi})$
that yield the exact same privacy guarantee (cf. Figure \ref{fig:A-conceptual-display}).
This phenomenon gives rise to an interesting novel concept of \emph{directional
noise}, which will be discussed in Section \ref{sec:Directional-Noise}.

We emphasize again that, in Theorem \ref{thm:design_general}, we
derive the sufficient condition for the MVG mechanism to guarantee
$(\epsilon,\delta)$-differential privacy \emph{without} making structural
assumption about the matrix-valued query function. In the next two
sections, we illustrate how incorporating the knowledge about the
intrinsic structural characteristic of the matrix-valued query function
of interest can yield an alternative sufficient condition. Then, we
prove that such alternative sufficient condition can provide the better
utility, when compared to the analysis without using the structural
knowledge.

\subsection{Differential Privacy Analysis for Symmetric Positive Semi-Definite
(PSD) Matrix-Valued Query \label{subsec:Differential-Privacy-Analysis_psd}}

To provide a concrete example of how the structural characteristics
of the matrix-valued query function can be exploited via the MVG mechanism,
we consider a matrix-valued query function that is \emph{symmetric
positive semi-definite (PSD)}. To avoid being cumbersome, we will
drop the explicit 'symmetric' in the subsequent references, but the
readers should keep in mind that we work with symmetric matrices here.
First, let us define a positive semi-definite matrix in our context.
\begin{defn}
A symmetric matrix $\mathbf{X}\in\mathbb{R}^{n\times n}$ is \emph{positive
semi-definite (PSD)} if $\mathbf{v}^{T}\mathbf{X}\mathbf{v}\geq0$
for all non-zero $\mathbf{v}\in\mathbb{R}^{n}$.
\end{defn}
Conceptually, we can think of a positive semi-definite matrix in matrix
analysis as the similar notion to a non-negative number in scalar-valued
analysis. More importantly, positive semi-definite matrices occur
regularly in practical settings. Examples of positive semi-definite
matrices in practice include the maximum likelihood estimate of the
covariance matrix \cite[chapter 7]{RefWorks:208}, the Hessian matrix
\cite[chapter 7]{RefWorks:208}, the kernel matrix in machine learning
\cite{RefWorks:33,RefWorks:469}, and the Laplacian matrix of a graph
\cite{RefWorks:329}.

With respect to differential privacy analysis, the assumption of positive
semi-definiteness on the query function is only applicable if it holds
for every possible instance of the datasets. Fortunately, this is
true for all of the aforementioned matrix-valued query functions because
the positive semi-definiteness is the \emph{intrinsic} nature of such
functions. In other words, if a user queries the maximum likelihood
estimate of the covariance matrix of the dataset, the (non-private)
matrix-valued query answer would always be positive semi-definite
regardless of the dataset from which it is computed. The same property
applies to other examples given. We refer to this type of property
as \emph{intrinsic} to the matrix-valued query function since it holds
due only to the nature of the query function regardless of the nature
of the dataset. This is clearly crucial in differential privacy analysis
as differential privacy considers the worst-case scenario, so any
assumption made would only be valid if it applies even in such scenario.

Before presenting the main result, we emphasize that the intrinsic
nature phenomenon is not unique to positive semi-definite matrix.
In other words, there are many other structural properties of the
matrix-valued query function that are also intrinsic. For example,
the adjacency matrix for an undirected graph is always symmetric \cite{RefWorks:329},
and the bi-stochastic matrix always has all non-negative entries with
each row and column sums up to one \cite{RefWorks:348}. Therefore,
the idea of exploiting structural characteristics of the matrix-valued
query function is very applicable in practice under the setting of
privacy-aware analysis.

Returning to the main result, we consider the MVG mechanism on a query
function that is symmetric positive semi-definite. Due to the definitive
symmetry of the query function output, it is reasonable to impose
the design choice $\boldsymbol{\Sigma}=\boldsymbol{\Psi}$ on the
MVG mechanism. The rationale is that, since the query function is
symmetric, its row-wise covariance and column-wise covariance are
necessarily equal if we view the query output as a random variable.
Hence, it is reasonable to employ the matrix-valued noise with the
same symmetry. As a result, this helps restrict our design space to
that with $\boldsymbol{\Sigma}=\boldsymbol{\Psi}$. With this setting,
we present the following theorem which states the sufficient condition
for the MVG mechanism to preserve $(\epsilon,\delta)$-differential
privacy when the query function is \emph{symmetric positive semi-definite}.
\begin{thm}
\label{thm:mvg_for_psd} Given a symmetric positive semi-definite
(PSD) matrix-valued query function $f(\mathbf{X})\in\mathbb{R}^{r\times r}$,
let $\boldsymbol{\sigma}(\boldsymbol{\Sigma}^{-1})=[\sigma_{1}(\boldsymbol{\Sigma}^{-1}),\ldots,\sigma_{r}(\boldsymbol{\Sigma}^{-1})]^{T}$
be the vectors of non-increasingly ordered singular values of $\boldsymbol{\Sigma}^{-1}$,
let $\boldsymbol{\Psi}=\boldsymbol{\Sigma}$, and let the relevant
variables be defined according to Table \ref{tab:Notations}. Then,
the MVG mechanism guarantees $(\epsilon,\delta)$-differential privacy
if $\boldsymbol{\Sigma}$ satisfy the following condition,
\begin{equation}
\left\Vert \boldsymbol{\sigma}(\boldsymbol{\Sigma}^{-1})\right\Vert _{2}^{2}\leq\frac{(-\beta+\sqrt{\beta^{2}+8\omega\epsilon})^{2}}{4\omega^{2}},\label{eq.sufficient_condition_psd}
\end{equation}
where $\omega=4H_{r}\gamma s_{2}(f)$, and $\beta=2r{}^{1/2}\zeta(\delta)H_{r}s_{2}(f)$. 
\end{thm}
\begin{proof}
The proof starts from the same characteristic equation (Eq. \eqref{eq:characteristic_eq})
as in Theorem \ref{thm:design_general}. However, since the query
function is symmetric, $f(\mathbf{X})=f(\mathbf{X})^{T}$. Furthermore,
since we impose $\boldsymbol{\Psi}=\boldsymbol{\Sigma}$, the characteristic
equation can be simplified as, 
\begin{equation}
\mathrm{tr}[\boldsymbol{\Sigma}^{-1}\mathcal{Y}^{T}\boldsymbol{\Sigma}^{-1}\boldsymbol{\Delta}+\boldsymbol{\Sigma}^{-1}\boldsymbol{\Delta}^{T}\boldsymbol{\Sigma}^{-1}\mathcal{Y}+\boldsymbol{\Sigma}^{-1}f(\mathbf{X}_{2})\boldsymbol{\Sigma}^{-1}f(\mathbf{X}_{2})-\boldsymbol{\Sigma}^{-1}f(\mathbf{X}_{1})\boldsymbol{\Sigma}^{-1}f(\mathbf{X}_{1})]\leq2\epsilon.\label{eq:characteristic_eq_psd}
\end{equation}
Again, this condition needs to be met with probability $\geq1-\delta$
for the MVG mechanism to preserve $(\epsilon,\delta)$-differential
privacy.

First, consider the \emph{first term}: $\mathrm{tr}[\boldsymbol{\Sigma}^{-1}\mathcal{Y}^{T}\boldsymbol{\Sigma}^{-1}\boldsymbol{\Delta}]$.
This term is exactly the same as the first term in the proof of Theorem
\ref{thm:design_general}, i.e. the first term of Eq. \eqref{eq:characteristic_eq}.
Hence, we readily have the upper bound on this term as 
\begin{equation}
\mathrm{tr}[\boldsymbol{\Sigma}^{-1}\mathcal{Y}^{T}\boldsymbol{\Sigma}^{-1}\boldsymbol{\Delta}]\leq\gamma s_{2}(f)H_{r}\left\Vert \boldsymbol{\sigma}(\boldsymbol{\Sigma}^{-1})\right\Vert _{2}^{2}+m^{1/2}\zeta(\delta)H_{r}s_{2}(f)\left\Vert \boldsymbol{\sigma}(\boldsymbol{\Sigma}^{-1})\right\Vert _{2},\label{eq:first_term_bound_psd}
\end{equation}
with probability $\geq1-\delta$. We note two minor differences between
Eq. \eqref{eq:first_term_bound} and Eq. \eqref{eq:first_term_bound_psd}.
First, the factor of $(mn)^{1/4}$ becomes $m^{1/2}$. This is simply
due to the fact that $n=m$ in the current setup with a PSD query
function. Second, the variable $\phi=(\left\Vert \boldsymbol{\sigma}(\boldsymbol{\Sigma}^{-1})\right\Vert _{2}\left\Vert \boldsymbol{\sigma}(\boldsymbol{\Psi}^{-1})\right\Vert _{2})^{1/2}$
in Eq. \eqref{eq:first_term_bound} becomes simply $\left\Vert \boldsymbol{\sigma}(\boldsymbol{\Sigma}^{-1})\right\Vert _{2}$
in Eq. \eqref{eq:first_term_bound_psd}. This is due to the fact that
$\boldsymbol{\Psi}=\boldsymbol{\Sigma}$ with the current PSD setting.
Apart from these, Eq. \eqref{eq:first_term_bound} and Eq. \eqref{eq:first_term_bound_psd}
are equivalent.

Second, consider the \emph{second term}: $\mathrm{tr}[\boldsymbol{\Sigma}^{-1}\boldsymbol{\Delta}^{T}\boldsymbol{\Sigma}^{-1}\mathcal{Y}]$.
Again, the second term of Eq. \eqref{eq:characteristic_eq_psd} is
the same as that of Eq. \eqref{eq:characteristic_eq}. Hence, we can
readily write,
\begin{equation}
\mathrm{tr}[\boldsymbol{\Sigma}^{-1}\boldsymbol{\Delta}^{T}\boldsymbol{\Sigma}^{-1}\mathcal{Y}]\leq\gamma H_{r}s_{2}(f)\left\Vert \boldsymbol{\sigma}(\boldsymbol{\Sigma}^{-1})\right\Vert _{2}^{2}+m{}^{1/2}\zeta(\delta)H_{r}s_{2}(f)\left\Vert \boldsymbol{\sigma}(\boldsymbol{\Sigma}^{-1})\right\Vert _{2},\label{eq:second_term_bound_psd}
\end{equation}
with probability $\geq1-\delta$. Again, we note the same two differences
between Eq. \eqref{eq:second_term_bound} and Eq. \eqref{eq:second_term_bound_psd}
as those between Eq. \eqref{eq:first_term_bound} and Eq. \eqref{eq:first_term_bound_psd}.

Next, consider the \emph{third term} and \emph{fourth term} combined:
$\mathrm{tr}[\boldsymbol{\Sigma}^{-1}f(\mathbf{X}_{2})\boldsymbol{\Sigma}^{-1}f(\mathbf{X}_{2})-\boldsymbol{\Sigma}^{-1}f(\mathbf{X}_{1})\boldsymbol{\Sigma}^{-1}f(\mathbf{X}_{1})]$.
Let us denote for a moment $\mathbf{A}=\boldsymbol{\Sigma}^{-1}f(\mathbf{X}_{2})$
and $\mathbf{B}=\boldsymbol{\Sigma}^{-1}f(\mathbf{X}_{1})$. Then,
this combined term can be re-written as, $\mathrm{tr}[\mathbf{A}^{2}-\mathbf{B}^{2}]$.
Next, we show that 
\[
\mathrm{tr}[\mathbf{A}^{2}-\mathbf{B}^{2}]=\mathrm{tr}[(\mathbf{A}-\mathbf{B})(\mathbf{A}+\mathbf{B})]
\]
by starting from the right hand side and proceeding to equate it to
the left hand side.
\begin{align*}
\mathrm{tr}([\mathbf{A}-\mathbf{B}][\mathbf{A}+\mathbf{B}])= & \mathrm{tr}(\mathbf{A}\mathbf{A}+\mathbf{A}\mathbf{B}-\mathbf{B}\mathbf{A}-\mathbf{B}\mathbf{B})\\
= & \mathrm{tr}(\mathbf{A}^{2})+\mathrm{tr}(\mathbf{A}\mathbf{B})-\mathrm{tr}(\mathbf{B}\mathbf{A})-\mathrm{tr}(\mathbf{B}^{2})\\
= & \mathrm{tr}(\mathbf{A}^{2})+\mathrm{tr}(\mathbf{A}\mathbf{B})-\mathrm{tr}(\mathbf{A}\mathbf{B})-\mathrm{tr}(\mathbf{B}^{2})\\
= & \mathrm{tr}(\mathbf{A}^{2})-\mathrm{tr}(\mathbf{B}^{2})\\
= & \mathrm{tr}(\mathbf{A}^{2}-\mathbf{B}^{2}),
\end{align*}
whereas the first-to-second line uses the additive property of the
trace, and the second-to-third line uses the commutative property
of the trace. Therefore, from this equation, we can write 
\begin{align*}
 & \mathrm{tr}[\boldsymbol{\Sigma}^{-1}f(\mathbf{X}_{2})\boldsymbol{\Sigma}^{-1}f(\mathbf{X}_{2})-\boldsymbol{\Sigma}^{-1}f(\mathbf{X}_{1})\boldsymbol{\Sigma}^{-1}f(\mathbf{X}_{1})]\\
= & \mathrm{tr}[(\boldsymbol{\Sigma}^{-1}f(\mathbf{X}_{2})-\boldsymbol{\Sigma}^{-1}f(\mathbf{X}_{1}))(\boldsymbol{\Sigma}^{-1}f(\mathbf{X}_{2})+\boldsymbol{\Sigma}^{-1}f(\mathbf{X}_{1}))]\\
= & \mathrm{tr}[\boldsymbol{\Sigma}^{-1}(f(\mathbf{X}_{2})-f(\mathbf{X}_{1}))\boldsymbol{\Sigma}^{-1}(f(\mathbf{X}_{2})+f(\mathbf{X}_{1}))]\\
= & \mathrm{tr}[\boldsymbol{\Sigma}^{-1}\tilde{\boldsymbol{\Delta}}\boldsymbol{\Sigma}^{-1}(f(\mathbf{X}_{2})+f(\mathbf{X}_{1}))],
\end{align*}
where $\tilde{\boldsymbol{\Delta}}=-\boldsymbol{\Delta}$. Then, we
can use Lemma \ref{lem:v_neumann} to write,
\[
\mathrm{tr}[\boldsymbol{\Sigma}^{-1}\tilde{\boldsymbol{\Delta}}\boldsymbol{\Sigma}^{-1}(f(\mathbf{X}_{2})+f(\mathbf{X}_{1}))]\leq\sum_{i=1}^{r}\sigma_{i}(\boldsymbol{\Sigma}^{-1}\tilde{\boldsymbol{\Delta}})\sigma_{i}(\boldsymbol{\Sigma}^{-1}(f(\mathbf{X}_{2})+f(\mathbf{X}_{1}))).
\]
Next, we use Lemma \ref{lem:singular_bound} to bound the two sets
of singular values as follows. For the first set of singular values,
\[
\sigma_{i}(\boldsymbol{\Sigma}^{-1}\tilde{\boldsymbol{\Delta}})\leq\frac{\left\Vert \boldsymbol{\Sigma}^{-1}\tilde{\boldsymbol{\Delta}}\right\Vert _{F}}{\sqrt{i}}\leq\frac{\left\Vert \boldsymbol{\Sigma}^{-1}\right\Vert _{F}\left\Vert \tilde{\boldsymbol{\Delta}}\right\Vert _{F}}{\sqrt{i}}\leq\frac{\left\Vert \boldsymbol{\sigma}(\boldsymbol{\Sigma}^{-1})\right\Vert _{2}s_{2}(f)}{\sqrt{i}},
\]
whereas the last step follows from the fact that $\left\Vert \tilde{\boldsymbol{\Delta}}\right\Vert _{F}=\left\Vert -\boldsymbol{\Delta}\right\Vert _{F}=\left\Vert \boldsymbol{\Delta}\right\Vert _{F}\leq s_{2}(f)$.
For the second set of singular values,
\begin{align*}
\sigma_{i}(\boldsymbol{\Sigma}^{-1}(f(\mathbf{X}_{2})+f(\mathbf{X}_{1}))) & \leq\frac{\left\Vert \boldsymbol{\Sigma}^{-1}(f(\mathbf{X}_{2})+f(\mathbf{X}_{1}))\right\Vert _{F}}{\sqrt{i}}\\
 & \leq\frac{\left\Vert \boldsymbol{\Sigma}^{-1}\right\Vert _{F}\left\Vert (f(\mathbf{X}_{2})+f(\mathbf{X}_{1}))\right\Vert _{F}}{\sqrt{i}}\\
 & \leq\frac{\left\Vert \boldsymbol{\Sigma}^{-1}\right\Vert _{F}(\left\Vert f(\mathbf{X}_{2})\right\Vert _{F}+\left\Vert f(\mathbf{X}_{1})\right\Vert _{F})}{\sqrt{i}}\\
 & \leq\frac{\left\Vert \boldsymbol{\sigma}(\boldsymbol{\Sigma}^{-1})\right\Vert _{2}2\gamma}{\sqrt{i}},
\end{align*}
whereas the second-to-third line follows from the triangular inequality.
Then, we combine the two bounds on the two sets of singular values
to get a bound for the third term and fourth term combined as,
\[
\mathrm{tr}[\boldsymbol{\Sigma}^{-1}f(\mathbf{X}_{2})\boldsymbol{\Sigma}^{-1}f(\mathbf{X}_{2})-\boldsymbol{\Sigma}^{-1}f(\mathbf{X}_{1})\boldsymbol{\Sigma}^{-1}f(\mathbf{X}_{1})]\leq2\gamma s_{2}(f)H_{r}\left\Vert \boldsymbol{\sigma}(\boldsymbol{\Sigma}^{-1})\right\Vert _{2}.
\]

\emph{Four terms combined}: by combining the four terms and rearranging
them, the characteristic equation becomes, 
\begin{equation}
\omega\left\Vert \boldsymbol{\sigma}(\boldsymbol{\Sigma}^{-1})\right\Vert _{2}^{2}+\beta\left\Vert \boldsymbol{\sigma}(\boldsymbol{\Sigma}^{-1})\right\Vert _{2}\leq2\epsilon.\label{eq:quad_equation_psd}
\end{equation}
This is a quadratic equation, of which the solution is $\left\Vert \boldsymbol{\sigma}(\boldsymbol{\Sigma}^{-1})\right\Vert _{2}\in[\frac{-\beta-\sqrt{\beta^{2}+8\omega\epsilon}}{2\omega},\frac{-\beta+\sqrt{\beta^{2}+8\omega\epsilon}}{2\omega}]$.
Since we know that $\left\Vert \boldsymbol{\sigma}(\boldsymbol{\Sigma}^{-1})\right\Vert _{2}\geq0$
due to the axiom of the norm, we only have the one-sided solution,
\[
\left\Vert \boldsymbol{\sigma}(\boldsymbol{\Sigma}^{-1})\right\Vert _{2}\leq\frac{-\beta+\sqrt{\beta^{2}+8\omega\epsilon}}{2\omega},
\]
which immediately implies the criterion in Theorem \ref{thm:mvg_for_psd}. 
\end{proof}
The sufficient condition in Theorem \ref{thm:mvg_for_psd} shares
the similar observation to that in Theorem \ref{thm:design_general},
i.e. the privacy guarantee by the MVG mechanism with a positive semi-definite
query function depends \emph{only on the singular values} of $\boldsymbol{\Sigma}$
(and $\boldsymbol{\Psi}$, effectively). However, the two theorems
differ slightly in that one is a function of $\alpha$, while the
other is a function of $\omega$. This is clearly the result of the
PSD assumption made by Theorem \ref{thm:mvg_for_psd}, but not by
Theorem \ref{thm:design_general}. In the next section, we claim that,
if the matrix-valued query function of interest is positive semi-definite,
it is more beneficial to apply Theorem \ref{thm:mvg_for_psd} than
Theorem \ref{thm:design_general} in most practical cases. This fosters
the notion that exploiting the structure of the matrix-valued query
function is attractive.

\subsection{Comparative Analysis on the Benefit of Exploiting Structural Characteristics
of Matrix-Valued Query Functions \label{subsec:Comparative-Analysis}}

In Section \ref{subsec:Differential-Privacy-Analysis_general} and
Section \ref{subsec:Differential-Privacy-Analysis_psd}, we discuss
two sufficient conditions for the MVG mechanism to guarantee $(\epsilon,\delta)$-differential
privacy in Theorem \ref{thm:design_general} and Theorem \ref{thm:mvg_for_psd},
respectively. The two theorems differ in one significant way \textendash{}
Theorem \ref{thm:design_general} does not utilize the structural
characteristic of the matrix-valued query function, whereas Theorem
\ref{thm:mvg_for_psd} does. More specifically, Theorem \ref{thm:mvg_for_psd}
utilizes the positive semi-definiteness of the matrix-valued query
function.

In this section, we establish the claim that such utilization can
be beneficial to the MVG mechanism for privacy-aware data analysis.
To establish the benefit notion, we first observe that both Theorem
\ref{thm:design_general} and Theorem \ref{thm:mvg_for_psd} put an
upper-bound on the singular values $\boldsymbol{\sigma}(\boldsymbol{\Sigma}^{-1})$
and $\boldsymbol{\sigma}(\boldsymbol{\Psi}^{-1})$. Let us consider
only $\boldsymbol{\sigma}(\boldsymbol{\Sigma}^{-1})$ for the moment.
It is well-known that the singular values of an inverted matrix is
the inverse of the singular values, i.e. $\sigma_{i}(\boldsymbol{\Sigma}^{-1})=1/\sigma_{i}(\boldsymbol{\Sigma})$.
Hence, we can write
\begin{equation}
\left\Vert \boldsymbol{\sigma}(\boldsymbol{\Sigma}^{-1})\right\Vert _{2}=\sqrt{\sum_{i=1}^{m}\frac{1}{\sigma_{i}^{2}(\boldsymbol{\Sigma})}}.\label{eq:singular_vector_decomposed}
\end{equation}
This representation of $\left\Vert \boldsymbol{\sigma}(\boldsymbol{\Sigma}^{-1})\right\Vert _{2}$
provides a very intuitive view of the sufficient conditions in Theorem
\ref{thm:design_general} and Theorem \ref{thm:mvg_for_psd} as follows.

In the MVG mechanism, the privacy preservation is achieved via the
addition of the matrix noise. Higher level of privacy guarantee requires
more perturbation. For additive noise, the notion of ``more'' here
can be quantified by the \emph{variance} of the noise. Specifically,
the noise with higher variance provides more perturbation. The last
connection we need to make is that between the notion of the noise
variance and the singular values of the covariance $\sigma_{i}(\boldsymbol{\Sigma})$.
In Section \ref{subsec:dir_noise_as_noniid}, we provide the detail
of this connection. Here, it suffices to say that each singular value
of the covariance matrix, $\sigma_{i}(\boldsymbol{\Sigma})$, corresponds
to a component of the overall variance of the matrix-variate Gaussian
noise. Therefore, intuitively, \emph{the larger the singular values
are, the larger the overall variance becomes, i.e. the higher the
perturbation is.} However, there is clearly a tradeoff. Although higher
variance can provide better privacy protection, it can also inadvertently
hurt the utility of the MVG mechanism. In signal-processing terminology,
this degradation of the utility can be described by the reduction
in the \emph{signal-to-noise ratio (SNR)} due to the increase in the
perturbation level, i.e. the increase in noise variance. Hence, our
notion of \emph{better} MVG mechanism, and, more broadly, \emph{better}
mechanism for differential privacy, corresponds to that of \emph{higher
SNR} as follows.
\begin{ax}
\label{axm:low_variance_high_snr}A differential-private mechanism
has the higher signal-to-noise ratio (SNR) if it employs the perturbation
with the lower variance.
\end{ax}
With this axiom and the given intuition about $\sigma_{i}(\boldsymbol{\Sigma})$,
let us revisit Eq. \eqref{eq:singular_vector_decomposed}. To achieve
low noise variance, $\sigma_{i}(\boldsymbol{\Sigma})$ needs to be
as low as possible. From Eq. \eqref{eq:singular_vector_decomposed},
it is clear that low $\sigma_{i}(\boldsymbol{\Sigma})$ corresponds
to large $\left\Vert \boldsymbol{\sigma}(\boldsymbol{\Sigma}^{-1})\right\Vert _{2}$.
Hence, from Axiom \ref{axm:low_variance_high_snr}, it is desirable
to have $\left\Vert \boldsymbol{\sigma}(\boldsymbol{\Sigma}^{-1})\right\Vert _{2}$
as large as possible, while still preserving differential privacy.
However, \emph{the sufficient conditions in Theorem \ref{thm:design_general}
and Theorem \ref{thm:mvg_for_psd} put the (different) constraints
on how large $\left\Vert \boldsymbol{\sigma}(\boldsymbol{\Sigma}^{-1})\right\Vert _{2}$
can be in order to preserve $(\epsilon,\delta)$-differential privacy}.
Then, clearly, for fixed $(\epsilon,\delta)$, the better mechanism
according to Axiom \ref{axm:low_variance_high_snr} is the one with
the higher upper-bound on $\left\Vert \boldsymbol{\sigma}(\boldsymbol{\Sigma}^{-1})\right\Vert _{2}$.
Therefore, we establish the following axiom for our comparative analysis.
\begin{ax}
\label{axm:better_mvg}For the MVG mechanism, the ($\epsilon,\delta$)-differential
privacy sufficient condition with the higher SNR is the one with the
larger upper-bound on $\left\Vert \boldsymbol{\sigma}(\boldsymbol{\Sigma}^{-1})\right\Vert _{2}$
and $\left\Vert \boldsymbol{\sigma}(\boldsymbol{\Psi}^{-1})\right\Vert _{2}$
for any fixed $(\epsilon,\delta)$.
\end{ax}
Then, we can now return to the main objective of this section, i.e.
\emph{to establish the benefit of exploiting structural characteristics
of matrix-valued query functions}. Recall that the main difference
in the setup of Theorem \ref{thm:design_general} and Theorem \ref{thm:mvg_for_psd}
is that the former does not utilize the structural characteristic
of the matrix-valued query function, while the latter utilizes the
positive semi-definiteness characteristic. As a result, Theorem \ref{thm:design_general}
and Theorem \ref{thm:mvg_for_psd} have different sufficient conditions.
More specifically, taken into the fact that $\boldsymbol{\Psi}=\boldsymbol{\Sigma}$
in the setting of Theorem \ref{thm:mvg_for_psd}, the two sufficient
conditions only differ by the upper-bound they constrain on how large
$\left\Vert \boldsymbol{\sigma}(\boldsymbol{\Sigma}^{-1})\right\Vert _{2}$
and $\left\Vert \boldsymbol{\sigma}(\boldsymbol{\Psi}^{-1})\right\Vert _{2}$
can be. Therefore, to compare the two sufficient conditions, we only
need to compare their upper-bounds. The following theorem presents
the main result of this comparison.
\begin{thm}
\label{thm:psd_is_tighter_than_general}Let $f(\mathbf{X})\in\mathbb{R}^{r\times r}$
be a symmetric positive semi-definite matrix-valued query function.
Then, the MVG mechanism on $f(\mathbf{X})$ implemented by Theorem
\ref{thm:mvg_for_psd} \textendash{} which utilizes the PSD characteristic
\textendash{} has higher SNR than that implemented by Theorem \ref{thm:design_general}
\textendash{} which does not utilize the PSD characteristic, if one
of the following conditions is met:
\begin{equation}
s_{2}(f)\leq\gamma\label{eq:psd_better_condition_1}
\end{equation}
or
\begin{equation}
r>12.\label{eq:psd_better_condition_2}
\end{equation}
\end{thm}
\begin{proof}
Recall that the two upper-bounds are $(-\beta+\sqrt{\beta^{2}+8\alpha\epsilon})^{2}/4\alpha^{2}$
and $(-\beta+\sqrt{\beta^{2}+8\omega\epsilon})^{2}/4\omega^{2}$ for
Theorem \ref{thm:design_general} and Theorem \ref{thm:mvg_for_psd},
respectively. Recall also that both upper-bounds are derived from
their respective quadratic equations (Eq. \eqref{eq:quad_eq_general}
and Eq. \eqref{eq:quad_equation_psd}). The proof is considerably
simpler if we start from these two quadratic equations. 

We first restate the two quadratic equations again here for Theorem
\ref{thm:design_general} and Theorem \ref{thm:mvg_for_psd}, respectively:
\begin{align}
\alpha\phi^{2}+\beta\phi-2\epsilon & \leq0\label{eq:quad_general_2}\\
\omega\phi^{2}+\beta\phi-2\epsilon & \leq0.\label{eq:quad_psd_2}
\end{align}
Note that in the current setup, $f(\mathbf{X})\in\mathbb{R}^{r\times r}$
and $\boldsymbol{\Psi}=\boldsymbol{\Sigma}$, so $\phi=\left\Vert \boldsymbol{\sigma}(\boldsymbol{\Sigma}^{-1})\right\Vert _{2}$
in both theorems. Then, we use a lesser-known formula for finding
roots of a quadratic equation, which has appeared in the \emph{Muller's
method} for root-finding algorithm \cite{RefWorks:470} and \emph{Vieta's
formulas} for polynomial coefficient relations \cite{RefWorks:471,RefWorks:472}.
\[
\phi\in\left[\frac{4\epsilon}{\beta+\sqrt{\beta^{2}+8\alpha\epsilon}},\frac{4\epsilon}{\beta-\sqrt{\beta^{2}+8\alpha\epsilon}}\right],
\]
and 
\[
\phi\in\left[\frac{4\epsilon}{\beta+\sqrt{\beta^{2}+8\omega\epsilon}},\frac{4\epsilon}{\beta-\sqrt{\beta^{2}+8\omega\epsilon}}\right],
\]
for Eq. \eqref{eq:quad_general_2} and Eq. \eqref{eq:quad_psd_2},
respectively. Since we know that $\phi\geq0$ in both cases due to
the axiom of the norm, we have a one-sided solution for both as,
\begin{equation}
\phi\leq\frac{4\epsilon}{\beta-\sqrt{\beta^{2}+8\alpha\epsilon}},\label{eq:sol2_general}
\end{equation}
and 
\begin{equation}
\phi\leq\frac{4\epsilon}{\beta-\sqrt{\beta^{2}+8\omega\epsilon}},\label{eq:sol2_psd}
\end{equation}
for Eq. \eqref{eq:quad_general_2} and Eq. \eqref{eq:quad_psd_2},
respectively. These two solutions are particularly simple to compare
since they are similar except for the term $\alpha$ and $\omega$
in the denominator. Recall from Axiom \ref{axm:better_mvg} that the
better MVG mechanism is the one with the larger upper-bound. Hence,
from Eq. \eqref{eq:sol2_general} and Eq. \eqref{eq:sol2_psd}, we
want to show that,
\[
\frac{4\epsilon}{\beta-\sqrt{\beta^{2}+8\omega\epsilon}}\geq\frac{4\epsilon}{\beta-\sqrt{\beta^{2}+8\alpha\epsilon}}.
\]
The denominator is always positive since $\epsilon,\theta,\alpha>0$,
so we can simplify this condition as follows.
\begin{align*}
\frac{4\epsilon}{\beta-\sqrt{\beta^{2}+8\omega\epsilon}} & \geq\frac{4\epsilon}{\beta-\sqrt{\beta^{2}+8\alpha\epsilon}}\\
\beta-\sqrt{\beta^{2}+8\alpha\epsilon} & \geq\beta-\sqrt{\beta^{2}+8\omega\epsilon}\\
\beta^{2}+8\alpha\epsilon & \geq\beta^{2}+8\omega\epsilon\\
\alpha & \geq\omega.
\end{align*}
Substitute in the definition of $\alpha$ and $\omega$ and we have,
\begin{align}
[H_{r}+H_{r,1/2}]\gamma^{2}+2H_{r}\gamma s_{2}(f) & \geq4H_{r}\gamma s_{2}(f)\nonumber \\{}
[H_{r}+H_{r,1/2}]\gamma^{2} & \geq2H_{r}\gamma s_{2}(f)\nonumber \\{}
[H_{r}+H_{r,1/2}]\gamma & \geq2H_{r}s_{2}(f).\label{eq:condition_for_better}
\end{align}
From the inequality in Eq. \eqref{eq:condition_for_better}, we can
take two routes to ensure that this inequality is satisfied, and that
gives rise to the two either-or conditions in Theorem \ref{thm:psd_is_tighter_than_general}.
We explore each route separately next.

\emph{Route I.} First, let us consider the relationship between $\gamma$
and $s_{2}(f)$ for PSD query functions, and show that, for PSD matrix-valued
query functions, $s_{2}(f)\leq\sqrt{2}\gamma$. From the definition
of the Frobenius norm,
\begin{align*}
s_{2}(f)= & \sup_{d(\mathbf{X}_{1},\mathbf{X}_{2})=1}\left\Vert \mathbf{X}_{1}-\mathbf{X}_{2}\right\Vert _{F}\\
= & \sup_{d(\mathbf{X}_{1},\mathbf{X}_{2})=1}\sqrt{\mathrm{tr}[(\mathbf{X}_{1}-\mathbf{X}_{2})(\mathbf{X}_{1}-\mathbf{X}_{2})^{T}]}\\
= & \sup_{d(\mathbf{X}_{1},\mathbf{X}_{2})=1}\sqrt{\mathrm{tr}[\mathbf{X}_{1}\mathbf{X}_{1}^{T}-\mathbf{X}_{1}\mathbf{X}_{2}^{T}-\mathbf{X}_{2}\mathbf{X}_{1}^{T}+\mathbf{X}_{2}\mathbf{X}_{2}^{T}]}\\
= & \sup_{d(\mathbf{X}_{1},\mathbf{X}_{2})=1}\sqrt{\mathrm{tr}[\mathbf{X}_{1}\mathbf{X}_{1}^{T}]-2\mathrm{tr}[\mathbf{X}_{1}\mathbf{X}_{2}^{T}]+\mathrm{tr}[\mathbf{X}_{2}\mathbf{X}_{2}^{T}]}
\end{align*}
Since both $\mathbf{X}_{1}$ and $\mathbf{X}_{2}$ are PSD, $\mathbf{X}_{1}\mathbf{X}_{2}^{T}$
is also PSD \cite{RefWorks:468}. Hence, it follows from a property
of the trace and the PSD matrix that, 
\[
\mathrm{tr}[\mathbf{X}_{1}\mathbf{X}_{2}^{T}]=\sum\lambda_{i}(\mathbf{X}_{1}\mathbf{X}_{2}^{T})\geq0.
\]
With this fact, along with the fact that $\left\Vert \mathbf{X}\right\Vert _{F}=\sqrt{\mathrm{tr}(\mathbf{X}\mathbf{X}^{T})}\leq\gamma$,
it immediately follows that,
\[
s_{2}(f)\leq\sqrt{\mathrm{tr}[\mathbf{X}_{1}\mathbf{X}_{1}^{T}]+\mathrm{tr}[\mathbf{X}_{2}\mathbf{X}_{2}^{T}]}\leq\sqrt{2\gamma^{2}}=\sqrt{2}\gamma.
\]
Next, let us return to the inequality in Eq. \eqref{eq:condition_for_better}
and substitute in $s_{2}(f)\leq\sqrt{2}\gamma$.
\begin{align}
[H_{r}+H_{r,1/2}]\gamma & \geq2\sqrt{2}H_{r}\gamma\nonumber \\
H_{r}+H_{r,1/2} & \geq2\sqrt{2}H_{r}\nonumber \\
H_{r,1/2} & \geq(2\sqrt{2}-1)H_{r}.\label{eq:Hr_inequality_condition}
\end{align}
Using the definition of the harmonic number, we can write
\[
\sum_{i=1}^{r}\frac{1}{\sqrt{r}}\geq(2\sqrt{2}-1)\sum_{i=1}^{r}\frac{1}{r}.
\]
Since $\frac{1}{\sqrt{r}}>\frac{1}{r}$ for all $r>1$, it is clear
that this condition is met for all $r$ greater than a certain threshold.
The threshold can easily be acquired numerically or analytically to
be $r>12$. This completes the proof of the first route.

\emph{Route II}. Second, we use the fact that $\frac{1}{\sqrt{r}}\geq\frac{1}{r}$
for $r>0$. Then, it is clear from the definition of the harmonic
number that $H_{r,1/2}\geq H_{r}$. Hence, we can write the inequality
in Eq. \eqref{eq:condition_for_better} as,
\begin{align*}
[H_{r}+H_{r,1/2}]\gamma & \geq2H_{r}s_{2}(f)\\
2H_{r}\gamma & \geq2H_{r}s_{2}(f)\\
\gamma & \geq s_{2}(f),
\end{align*}
which immediately completes the proof of the second route.
\end{proof}
One final necessary remark regarding this comparative analysis is
to justify the practicality of the conditions in Theorem \ref{thm:psd_is_tighter_than_general}.
We reiterate that \emph{only one of the two conditions} needs to be
met and discuss separately the practicality of each. 

\emph{The first condition}, i.e $s_{2}(f)\leq\gamma$, can be interpret
conceptually as that the sensitivity should not be greater than the
largest value of the query output. Two observations support the pragmatism
of this condition. First, the sensitivity is derived from the change
in the query output when \emph{only one input record} changes. Intuitively,
this suggests that such change should be confined to the largest possible
value of the query output, especially for a positive semi-definite
query function whose output is always non-negative. Second, one of
the premises of differential privacy for statistical query is that
the query can be answered accurately if the sensitivity of such query
is small \cite{RefWorks:195}. Hence, this condition coincides properly
with the primary premise of differential privacy. To concretely illustrate
its practicality, we consider two examples of matrix-valued query
functions as follows.
\begin{example}
\label{exa:cov_exp}Let us consider a dataset consisting of $n$ records
drawn i.i.d. from an unknown $\mathbb{R}^{m}$ distribution with zero
mean, and consider the covariance estimation as the query function
\cite[chapter 7]{RefWorks:208}, i.e. $f(\mathbf{X})=\frac{\mathbf{X}\mathbf{X}^{T}}{n}$.
Let us assume without the loss of generality that $\mathbf{X}\in[-c,c]^{m\times n}$.
Then, the largest query output value can be computed as,
\[
\gamma=\sup_{\mathbf{X}}\left\Vert \frac{\mathbf{X}\mathbf{X}^{T}}{n}\right\Vert _{F}=\sup_{\mathbf{X}}\left\Vert \frac{1}{n}\sum_{i=1}^{n}\mathbf{x}_{i}\mathbf{x}_{i}^{T}\right\Vert _{F}=mc^{2}.
\]
Next, the sensitivity can be computed by considering that $f(\mathbf{X})=\frac{1}{n}\mathbf{X}\mathbf{X}^{T}=\frac{1}{n}\sum_{i=1}^{n}\mathbf{x}_{i}\mathbf{x}_{i}^{T}$.
Hence, for neighboring datasets $\mathbf{X}$ and $\mathbf{X}'$,
the query outputs differ by only one summand and the sensitivity can
be computed as,
\[
s_{2}(f)=\sup_{\mathbf{X},\mathbf{X}'}\frac{\left\Vert \mathbf{x}_{j}\mathbf{x}_{j}^{T}-\mathbf{x}_{j}'\mathbf{x}_{j}'^{T}\right\Vert _{F}}{n}\leq\frac{2\sqrt{\sum_{j=1}^{m^{2}}x_{j}(i)^{4}}}{n}=\frac{2mc^{2}}{n}.
\]
Clearly, $s_{2}(f)\ll\gamma$ for any reasonably-sized dataset.
\end{example}
\begin{example}
\label{exa:kernel_exp}Let us consider the kernel matrix often used
in machine learning \cite{RefWorks:33,RefWorks:469}. Given a kernel
function $k(\mathbf{x}_{i},\mathbf{x}_{j})\in\mathbb{R}$, the kernel
matrix query is $f(\mathbf{X})=\mathbf{G}\in\mathbb{R}^{n\times n}$,
where $[\mathbf{G}]_{ij}=k(\mathbf{x}_{i},\mathbf{x}_{j})$. We note
that the notation $[\cdot]_{ij}$ indicates the $[i^{th},j^{th}]$-element
of the matrix. Let us assume further without the loss of generality
that $k(\mathbf{x}_{i},\mathbf{x}_{j})\leq c$. Then, the largest
query output value can be computed as,
\[
\gamma=\sup_{\mathbf{X}}\left\Vert \mathbf{G}\right\Vert _{F}=\sup\sqrt{\sum_{i,j}k(\mathbf{x}_{i},\mathbf{x}_{j})^{2}}=nc.
\]
Next, the sensitivity can be computed by considering that the $[i^{th},j^{th}]$-element
of the matrix $\mathbf{G}$ only depends on $\mathbf{x}_{i},\mathbf{x}_{j}\in\mathbf{X}$,
but not on any other record. Hence, for neighboring datasets $\mathbf{X}$
and $\mathbf{X}'$, the query outputs differ by only one row and column.
Then, the sensitivity can be computed as,
\[
s_{2}(f)=\sup_{\mathbf{X},\mathbf{X}'}\left\Vert \mathbf{G}-\mathbf{G}'\right\Vert _{F}=\sup\sqrt{2\sum_{i=1}^{n}4k(\mathbf{x}_{k},\mathbf{x}_{i})^{2}-4k(\mathbf{x}_{k},\mathbf{x}_{k})^{2}}=c\sqrt{8n-4}.
\]
From the values of $\gamma$ and $s_{2}(f)$, it can easily be shown
that $s_{2}(f)<\gamma$ for $n>7$. For the kernel matrix used in
machine learning, $n$ corresponds to the number of samples in the
dataset. Clearly, in practice, the dataset has considerably more than
7 samples, so this condition is always met for the kernel matrix query
function.
\end{example}
Next, let us consider \emph{the second condition}, i.e. $r>12$, which
simply states that the dimension of the positive semi-definite matrix-valued
query function should be larger than 12. This can be justified simply
by noting that many real-world positive semi-definite matrix-valued
query functions have dimensions significantly larger than 12. For
example, as seen in Example \ref{exa:cov_exp} and Example \ref{exa:kernel_exp},
the dimensions of the covariance matrix and the kernel matrix depend
on the number of features and data samples, respectively. Both quantities
are generally considerably larger than 12 in practice. Similarly,
the Laplacian matrix of a graph has its dimension depending upon the
number of nodes in the graph, which is usually considerably larger
than 12 \cite{RefWorks:329}.

Finally, we emphasize that \emph{only one of the two conditions},
i.e. $s_{2}(f)\leq\gamma$ or $r>12$, needs to be satisfied in order
for the MVG mechanism to fully exploit the positive semi-definiteness
characteristic of the matrix-valued query function. As we have shown,
this poses little restriction in practice.

As a segue to the next section, we revisit the sufficient conditions
for the MVG mechanism to guarantee ($\epsilon,\delta$)-differential
privacy in Theorem \ref{thm:design_general} and Theorem \ref{thm:mvg_for_psd}.
In both theorems, the sufficient conditions put an upper-bound only
on the singular values of the two covariance matrices $\boldsymbol{\Sigma}$
and $\boldsymbol{\Psi}$. We briefly describe in this section that
each singular value can be interpret as a variance of the noise. In
the next section, we provide a detailed analysis on this interpretation
and how it can be exploited for the application of privacy-aware data
analysis with the MVG mechanism.

\section{Directional Noise for Matrix-Valued Query \label{sec:Directional-Noise}}

Recall from the differential privacy analysis of the MVG mechanism
(Theorem \ref{thm:design_general} and Theorem \ref{thm:mvg_for_psd})
that the sufficient conditions for the MVG mechanism to guarantee
$(\epsilon,\delta)$-differential privacy in both theorems only apply
to the singular values of the two covariance matrices $\boldsymbol{\Sigma}$
and $\boldsymbol{\Psi}$. Here, we investigate the ramification of
this result via the novel notion of \emph{directional noise}.

\subsection{Motivation for Non-i.i.d. Noise for Matrix-Valued Query}

For a matrix-valued query function, the standard method for basic
mechanisms that use additive noise is by adding the \emph{independent
and identically distributed} (i.i.d.) noise to each element of the
matrix query. However, as common in matrix analysis \cite{RefWorks:208},
the matrices involved often have some geometric and algebraic characteristics
that can be exploited. As a result, it is usually the case that only
certain ``parts'' \textendash{} the term which will be defined more
precisely shortly \textendash{} of the matrices contain useful information.
In fact, this is one of the rationales behind many compression techniques
such as the popular principal component analysis (PCA) \cite{RefWorks:33,RefWorks:51,RefWorks:225}.
Due to this reason, adding the same amount of noise to every ``part''
of the matrix query may be highly suboptimal.

\subsection{Directional Noise as a Non-i.i.d. Noise}

\label{subsec:dir_noise_as_noniid}

\begin{figure}
\begin{centering}
\includegraphics[scale=0.37]{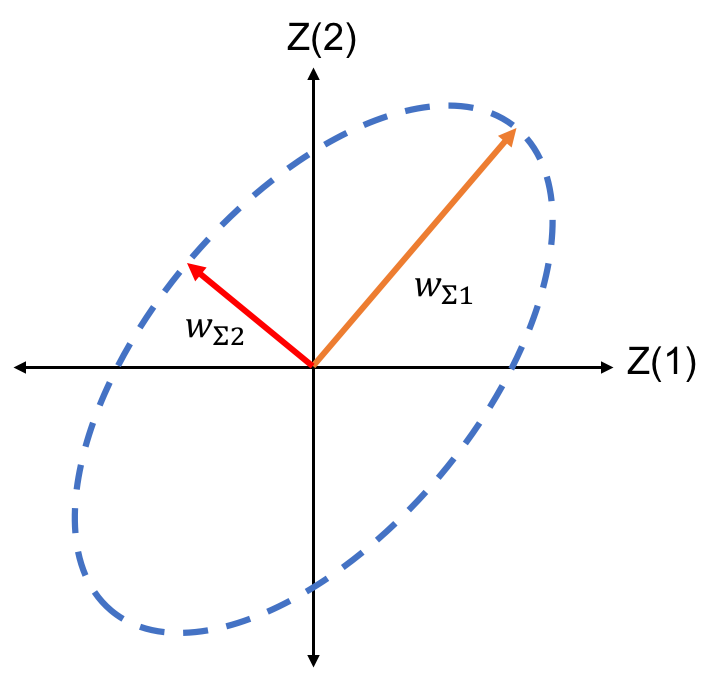}\includegraphics[scale=0.32]{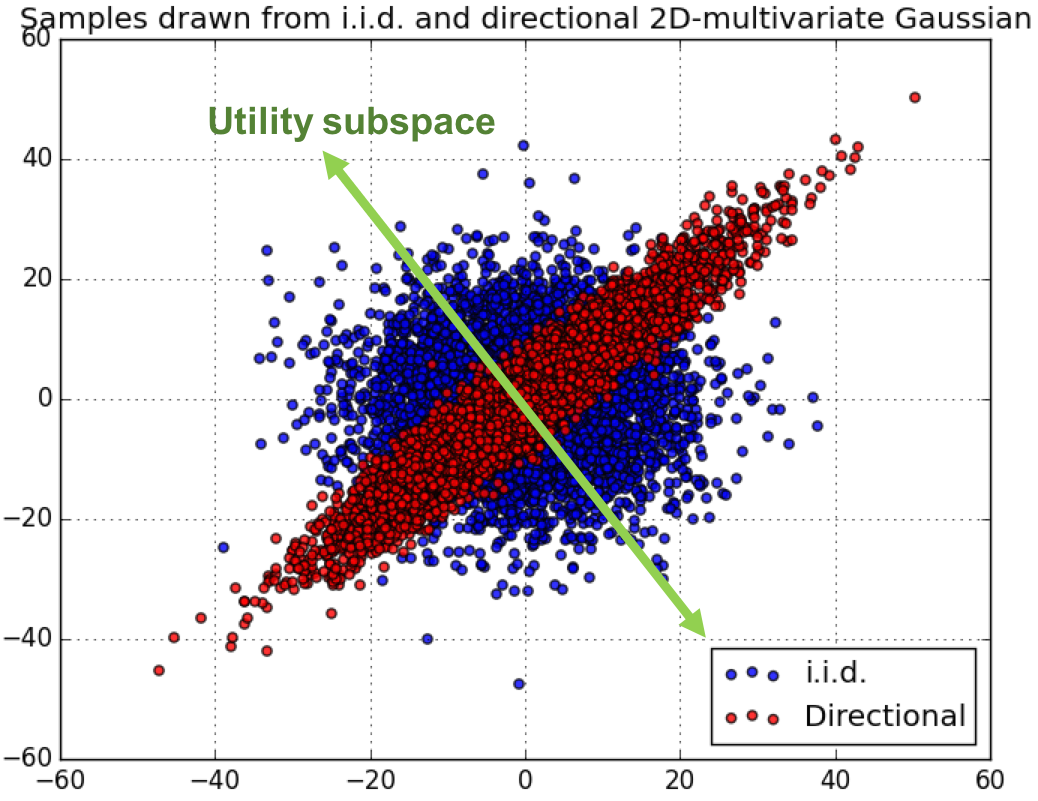}\includegraphics[scale=0.37]{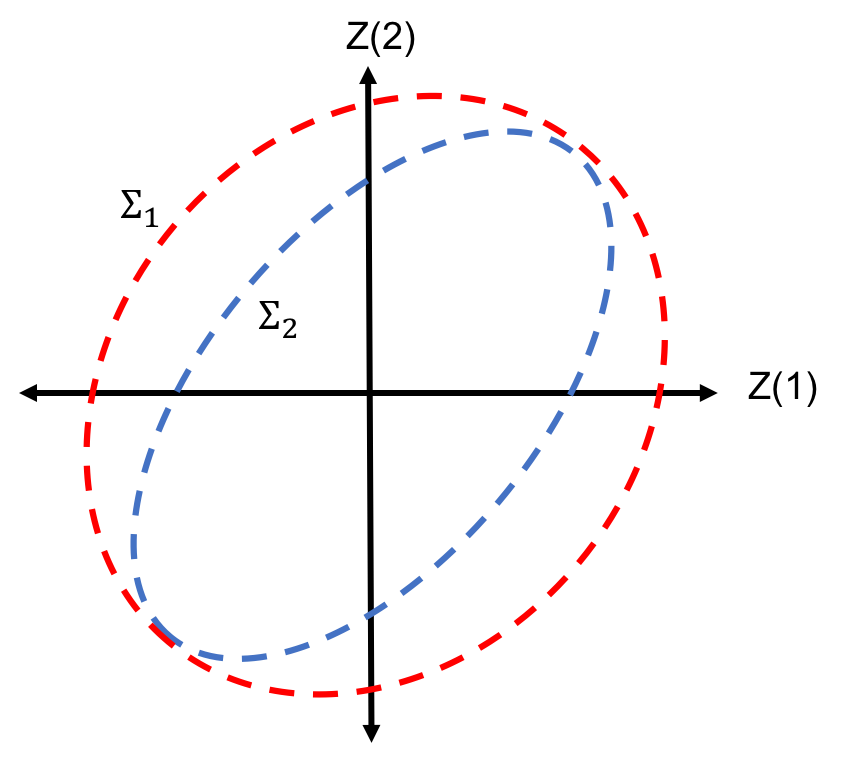}
\par\end{centering}
\caption{(Left) An ellipsoid of equi-density contour of a 2D multivariate Gaussian
distribution. The blue contour represents every point in this 2D space
which has the same probability density under this distribution. The
two arrows indicate the principal axes of this ellipsoid. (Middle)
Directional noise (red) and i.i.d. noise (blue) drawn from a 2D-multivariate
Gaussian distribution. The green line represents a possible utility
subspace that can benefit from this instance of directional noise.
(Right) Graphical comparison of the noise generated from two different
values of $\boldsymbol{\Sigma}$. The dash ellipsoids represent the
equi-density contours corresponding to the noise generated by the
two different values of $\boldsymbol{\Sigma}$. Clearly, from the
signal-to-noise ratio (SNR) point of view, the blue contour ($\boldsymbol{\Sigma}_{2}$)
is more preferable to the red contour ($\boldsymbol{\Sigma}_{1}$)
since its noise has smaller overall variance.\label{fig:dir_noise_samples}}
\end{figure}

Let us elaborate further on the ``parts'' of a matrix. In matrix
analysis, the prevalent paradigm to extract underlying properties
of a matrix is via \emph{matrix factorization} \cite{RefWorks:208}.
This is a family of algorithms and the specific choice depends upon
the application and types of insights it requires. Particularly, in
our application, the factorization that is informative is the singular
value decomposition (SVD) (Theorem \ref{thm:svd}) of the two covariance
matrices of \emph{$\mathcal{MVG}_{m,n}(\mathbf{0},\boldsymbol{\Sigma},\boldsymbol{\Psi})$}.

Consider first the covariance matrix $\boldsymbol{\Sigma}\in\mathbb{R}^{m\times m}$,
and write its SVD as, $\boldsymbol{\Sigma}=\mathbf{W}_{1}\boldsymbol{\Lambda}\mathbf{W}_{2}^{T}$.
It is well-known that, for the covariance matrix, we have the equality
$\mathbf{W}_{1}=\mathbf{W}_{2}$ since it is positive definite (cf.
\cite{RefWorks:291,RefWorks:208}). Hence, let us more concisely write
the SVD of $\boldsymbol{\Sigma}$ as, 
\[
\boldsymbol{\Sigma}=\mathbf{W}_{\boldsymbol{\Sigma}}\boldsymbol{\Lambda}_{\boldsymbol{\Sigma}}\mathbf{W}_{\boldsymbol{\Sigma}}^{T}.
\]
This representation gives us a very useful insight to the noise generated
from \emph{$\mathcal{MVG}_{m,n}(\mathbf{0},\boldsymbol{\Sigma},\boldsymbol{\Psi})$}:
it tells us the \emph{directions} of the noise via the column vectors
of $\mathbf{W}_{\boldsymbol{\Sigma}}$, and \emph{variance} of the
noise in each direction via the singular values in $\boldsymbol{\Lambda}_{\boldsymbol{\Sigma}}$.

For simplicity, let us consider a two-dimensional multivariate Gaussian
distribution, i.e. $m=2$, so there are two column vectors of $\mathbf{W}_{\boldsymbol{\Sigma}}=[\mathbf{w}_{\boldsymbol{\Sigma}1},\mathbf{w}_{\boldsymbol{\Sigma}2}]$.
The geometry of this distribution can be depicted by an ellipsoid,
e.g. the dash contour in Figure \ref{fig:dir_noise_samples}, Left
(cf. \cite[ch. 4]{RefWorks:51}, \cite[ch. 2]{RefWorks:225}). This
ellipsoid is characterized by its two principal axes \textendash{}
the major and the minor axes. It is well-known that the two column
vectors from SVD, i.e. $\mathbf{w}_{\boldsymbol{\Sigma}1}$ and $\mathbf{w}_{\boldsymbol{\Sigma}2}$,
are unit vectors pointing in the directions of the major and minor
axes of this ellipsoid, and more importantly, the length of each axis
is characterized by its corresponding singular value, i.e. $\sigma_{\boldsymbol{\Sigma}1}$
and $\sigma_{\boldsymbol{\Sigma}2}$, respectively (cf. \cite[ch. 4]{RefWorks:51})
(recall from Theorem \ref{thm:svd} that $diag(\boldsymbol{\Lambda}_{\boldsymbol{\Sigma}})=[\sigma_{\boldsymbol{\Sigma}1},\sigma_{\boldsymbol{\Sigma}2}]$).
This is illustrated by Figure \ref{fig:dir_noise_samples}, Left.
Therefore, when we consider the noise generated from this 2D multivariate
Gaussian distribution, we arrive at the following interpretation of
the SVD of its covariance matrix: the noise is distributed toward
the two principal \emph{directions} specified by $\mathbf{w}_{\boldsymbol{\Sigma}1}$
and $\mathbf{w}_{\boldsymbol{\Sigma}2}$, with the \emph{variance}
scaled by the corresponding singular values, $\sigma_{\boldsymbol{\Sigma}1}$
and $\sigma_{\boldsymbol{\Sigma}2}$.

We can extend this interpretation to a more general case with $m>2$,
and also to the other covariance matrix $\boldsymbol{\Psi}$. Then,
we have a full interpretation of \emph{$\mathcal{MVG}_{m,n}(\mathbf{0},\boldsymbol{\Sigma},\boldsymbol{\Psi})$}
as follows. The matrix-valued noise distributed according to \emph{$\mathcal{MVG}_{m,n}(\mathbf{0},\boldsymbol{\Sigma},\boldsymbol{\Psi})$}
has two components: the\emph{ row-wise noise}, and the \emph{column-wise
noise}. The row-wise noise and the column-wise noise are characterized
by the two covariance matrices, $\boldsymbol{\Sigma}$ and $\boldsymbol{\Psi}$,
respectively, as follows.

\subsubsection{For the row-wise noise}
\begin{itemize}
\item The row-wise noise is characterized by $\boldsymbol{\Sigma}$. 
\item SVD of $\boldsymbol{\Sigma}=\mathbf{W}_{\boldsymbol{\Sigma}}\boldsymbol{\Lambda}_{\boldsymbol{\Sigma}}\mathbf{W}_{\boldsymbol{\Sigma}}^{T}$
decomposes the row-wise noise into two components \textendash{} the
directions and the variances of the noise in those directions. 
\item The \emph{directions} of the row-wise noise are specified by the column
vectors of $\mathbf{W}_{\boldsymbol{\Sigma}}$. 
\item The \emph{variance} of each row-wise-noise direction is indicated
by its corresponding singular value in $\boldsymbol{\Lambda}_{\boldsymbol{\Sigma}}$. 
\end{itemize}

\subsubsection{For the column-wise noise}
\begin{itemize}
\item The column-wise noise is characterized by $\boldsymbol{\Psi}$. 
\item SVD of $\boldsymbol{\Psi}=\mathbf{W}_{\boldsymbol{\Psi}}\boldsymbol{\Lambda}_{\boldsymbol{\Psi}}\mathbf{W}_{\boldsymbol{\Psi}}^{T}$
decomposes the column-wise noise into two components \textendash{}
the directions and the variances of the noise in those directions. 
\item The \emph{directions} of the column-wise noise are specified by the
column vectors of $\mathbf{W}_{\boldsymbol{\Psi}}$. 
\item The \emph{variance} of each column-wise-noise direction is indicated
by its corresponding singular value in $\boldsymbol{\Lambda}_{\boldsymbol{\Psi}}$. 
\end{itemize}
Since \emph{$\mathcal{MVG}_{m,n}(\mathbf{0},\boldsymbol{\Sigma},\boldsymbol{\Psi})$}
is fully characterized by its covariances, these two components of
the matrix-valued noise drawn from \emph{$\mathcal{MVG}_{m,n}(\mathbf{0},\boldsymbol{\Sigma},\boldsymbol{\Psi})$}
provide a complete interpretation of the matrix-variate Gaussian noise.

\subsection{Directional Noise via the MVG Mechanism \label{subsec:Dir_noise_via_mvg}}

With the notion of directional noise, we now revisit Theorem \ref{thm:design_general}
and Theorem \ref{thm:mvg_for_psd}. Recall that the sufficient conditions
for the MVG mechanism to preserve $(\epsilon,\delta)$-differential
privacy according to both theorems put the constraint only on the
singular values of $\boldsymbol{\Sigma}$ and $\boldsymbol{\Psi}$.
However, as we discuss in the previous section, the singular values
of $\boldsymbol{\Sigma}$ and $\boldsymbol{\Psi}$ only indicate the
\emph{variance} of the noise in each direction, but \emph{not the
directions they are attributed to}. In other words, Theorem \ref{thm:design_general}
and Theorem \ref{thm:mvg_for_psd} suggest that the MVG mechanism
preserves $(\epsilon,\delta)$-differential privacy \emph{as long
as the overall variances of the noise satisfy a certain threshold,
but these variances can be attributed non-uniformly in any direction.}

This major claim certainly warrants further discussion, and we will
defer it to Section \ref{sec:Practical-Implementation}, where we
present the technical detail on how to actually implement this concept
of directional noise in practical settings. It is important to only
note here that this claim \emph{does not mean} that we can avoid adding
noise in any particular direction altogether. On the contrary, there
is still a minimum amount of noise \emph{required in every direction}
for the MVG mechanism to guarantee differential privacy, but the noise
simply can be attributed unevenly in different directions (cf. Fig.
\ref{fig:dir_noise_samples}, Right, for an example).

Finally, with the understanding of directional noise via the MVG mechanism,
we can revisit the discussion in Sec. \ref{subsec:Comparative-Analysis}.
Conceptually, the sufficient conditions in Theorem \ref{thm:design_general}
and Theorem \ref{thm:mvg_for_psd} impose a threshold on the minimum
overall variance of the noise in every direction. Therefore, intuitively,
to achieve the most utility out of the noisy query output, we want
the threshold to be as low as possible (cf. Figure \ref{fig:dir_noise_samples},
Right). As discussed in Section \ref{subsec:Comparative-Analysis},
since both sufficient conditions put the constraint upon the \emph{inverse}
of the variance, this means that we want the upper-bound of the sufficient
conditions to be as large as possible for maximum utility, and, hence,
the derivation of Theorem \ref{thm:psd_is_tighter_than_general}.

\subsection{Utility Gain via Directional Noise \label{subsec:Utility-Gain-via-dir-noise}}

As the concept of the directional noise is established, one important
remaining question is how we can exploit this to enhance the utility
of differential privacy. The answer to this question depends on the
query function. Here, we present examples for two popular matrix-valued
query functions: the identity query and the covariance matrix.

\subsubsection{Identity Query}

Formally, the identity query is $f(\mathbf{X})=\mathbf{X}$. We illustrate
how to exploit directional noise for enhancing utility via the following
two examples. 
\begin{example}
Consider the personalized warfarin dosing problem \cite{RefWorks:404},
which can be considered as the regression problem with the identity
query. In the i.i.d. noise scheme, every feature used in the warfarin
dosing prediction is equally perturbed. However, domain experts may
have prior knowledge that some features are more critical than the
others, so adding directional noise designed such that the more critical
features are perturbed less can potentially yield better prediction
performance. 
\end{example}
We note from this example that the directions chosen here are among
the \emph{standard basis}, e.g. $\mathbf{e}_{1}=[1,0,\ldots,0]^{T},\mathbf{e}_{2}=[0,1,\ldots,0]^{T}$,
which is one of the simplest forms of directions. Moreover, the directions
in this example are decided based on the domain expertise. When domain
knowledge is unavailable, we may still identify the useful directions
by spending a fraction of the privacy budget on deriving the directions. 
\begin{example}
Consider again the warfarin dosing problem \cite{RefWorks:404}, and
assume that we do not possess any prior knowledge about the predictive
features. We can still learn this information from the data by spending
a small privacy budget on deriving differentially-private principal
components (P.C.) from available differentially-private PCA algorithms
\cite{pca-gauss,RefWorks:249,RefWorks:178,RefWorks:313,RefWorks:194,RefWorks:405,RefWorks:406}.
Each P.C. can then serve as a direction and, with directional noise,
we can selectively add less noise in the highly informative directions
as indicated by PCA. 
\end{example}
As opposed to the previous example, the directions in this example
are not necessary among the standard basis, but can be \emph{any unit
vector}. Nevertheless, this example illustrates how directional noise
can provide additional utility benefit even with no assumption on
the prior knowledge.

Figure \ref{fig:dir_noise_samples}, Middle, illustrates an example
of how directional noise can provide the utility gain over i.i.d.
noise. In the illustration, we assume the data with two features,
and assume that we have obtained the utility direction, e.g. from
PCA, represented by the green line. This can be considered as the
utility subspace we desire to be least perturbed. The many small circles
in the illustration represent how the i.i.d. noise and directional
noise are distributed under the 2D multivariate Gaussian distribution.
Clearly, directional noise can reduce the perturbation experienced
on the utility subspace when compared to the i.i.d. noise.

\subsubsection{Covariance Matrix}

Let us consider a slightly more involved matrix-valued query similar
to Example \ref{exa:cov_exp}: the covariance matrix, i.e. $f(\mathbf{X})=\frac{1}{n}\mathbf{X}\mathbf{X}^{T}$.
Here, we consider the dataset with $n$ samples and $m$ features.
The following example illustrates how we can utilize directional noise
for this query function. 
\begin{example}
Consider the Netflix prize dataset \cite{RefWorks:377,RefWorks:252}.
A popular method for solving the Netflix challenge is via the low-rank
approximation \cite{RefWorks:300}. One way to perform this method
is to query the covariance matrix of the dataset \cite{RefWorks:194,RefWorks:408,RefWorks:178}.
Suppose we use output perturbation to preserve differential privacy,
i.e. $\mathcal{A}(f(\mathbf{X}))=f(\mathbf{X})+\mathcal{Z}$, and
suppose we have prior knowledge from domain experts that some features
are more informative than the others. Since the covariance matrix
has the underlying property that each row and column correspond to
a single feature \cite{RefWorks:51}, we can use this prior knowledge
with directional noise by adding less noise to the rows and columns
corresponding to the informative features. 
\end{example}
This example emphasizes the usefulness of the concept of directional
noise even when the query function is more complex than the simple
identity query.

In the recent three examples, we derive the directions for the noise
either from PCA or prior knowledge. PCA, however, is probably only
suitable to the data-release type of query, e.g. the identity query.
It is, therefore, still unclear how to derive the directions of the
noise for a general matrix-valued query function when prior knowledge
may not be available. We discuss a possible solution to this problem
in the next section.

\subsubsection{Directional Noise for General Matrix-Valued Query}

\label{subsec:Dir._noise_general}

For a general matrix-valued query function, one general possible approach
to derive the directions of the noise is to use the SVD. As discussed
in Section \ref{sec:Directional-Noise}, SVD can decompose a matrix
into its directions and variances. Hence, for a general matrix-valued
query function, we can cast aside a small portion of privacy budget
to derive the directions from the SVD of the query function. Clearly,
the direction-derivation process via SVD needs to be private. Fortunately,
there have been many works on differentially-private SVD \cite{RefWorks:405,RefWorks:406,RefWorks:194}.
We experimentally demonstrate the feasibility of the approach in Section
\ref{sec:Experimental-Results}.

In the next section, we discuss how we implement directional noise
with the MVG mechanism in practice and propose two simple algorithms
for two types of directional noise.

\section{Practical Implementation of MVG Mechanism \label{sec:Practical-Implementation}}

The differential privacy condition in Theorem \ref{thm:design_general}
and Theorem \ref{thm:mvg_for_psd}, even along with the notion of
directional noise in the previous section, still leads to a large
design space for the MVG mechanism. In this section, we present two
simple algorithms to implement the MVG mechanism with two types of
directional noise that can be appropriate for a wide range of real-world
applications. Then, we conclude the section with a discussion on the
sampling algorithms for \emph{$\mathcal{MVG}_{m,n}(\mathbf{0},\boldsymbol{\Sigma},\boldsymbol{\Psi})$}
necessary for the practical implementation of our mechanism.

As discussed in Section \ref{subsec:Dir_noise_via_mvg}, Theorem \ref{thm:design_general}
and Theorem \ref{thm:mvg_for_psd} state that the MVG mechanism satisfies
$(\epsilon,\delta)$-differential privacy as long as the singular
values of the row-wise and column-wise covariance matrices satisfy
the sufficient condition. This provides tremendous flexibility in
the choice of the \emph{directions of the noise}. First, we notice
from the sufficient condition in Theorem \ref{thm:design_general}
that the singular values for $\boldsymbol{\Sigma}$ and $\boldsymbol{\Psi}$
are decoupled, i.e. they can be designed independently, whereas Theorem
\ref{thm:mvg_for_psd} explicitly imposes $\boldsymbol{\Psi}=\boldsymbol{\Sigma}$.
Hence, the \emph{row-wise noise} and \emph{column-wise noise} can
be considered as the \emph{two modes of noise} in the MVG mechanism.
By this terminology, we discuss two types of directional noise: the
\emph{unimodal} and \emph{equi-modal} directional noise.

\subsection{Unimodal Directional Noise \label{subsec:Unimodal-Directional-Noise}}

For the unimodal directional noise, we select \emph{one mode} of the
noise to be directional noise, whereas the other mode of the noise
is set to be i.i.d. For this discussion, we assume that the row-wise
noise is directional noise, while the column-wise noise is i.i.d.
However, the opposite case can be readily analyzed with the similar
analysis. Furthermore, for simplicity, we only consider the differential
privacy guarantee provided by Theorem \ref{thm:design_general} for
the unimodal directional noise.

We note that, apart from simplifying the practical implementation
that we will discuss shortly, this type of directional noise can be
appropriate for many applications. For example, for the identity query,
we may not possess any prior knowledge on the quality of each sample,
so the best strategy would be to consider the i.i.d. column-wise noise
(recall that in our notation, samples are the column vectors).

Formally, the unimodal directional noise sets $\boldsymbol{\Psi}=\mathbf{I}_{n}$,
where $\mathbf{I}_{n}$ is the $n\times n$ identity matrix. This,
consequently, reduces the design space for the MVG mechanism with
directional noise to only the design of $\boldsymbol{\Sigma}$. Next,
consider the left side of the sufficient condition in Theorem \ref{thm:design_general},
and we have 
\begin{equation}
\left\Vert \boldsymbol{\sigma}(\boldsymbol{\Sigma}^{-1})\right\Vert _{2}=\sqrt{\sum_{i=1}^{m}\frac{1}{\sigma_{i}^{2}(\boldsymbol{\Sigma})}}\ \textrm{, and}\ \left\Vert \boldsymbol{\sigma}(\boldsymbol{\Psi}^{-1})\right\Vert _{2}=\sqrt{n}.\label{eq:singular_u}
\end{equation}
If we square both sides of the sufficient condition and re-arrange
it, we get a form of the condition such that the row-wise noise in
each direction is completely decoupled: 
\begin{equation}
\sum_{i=1}^{m}\frac{1}{\sigma_{i}^{2}(\boldsymbol{\Sigma})}\leq\frac{1}{n}\frac{(-\beta+\sqrt{\beta^{2}+8\alpha\epsilon})^{4}}{16\alpha^{4}}.\label{eq:decoupled_condition}
\end{equation}

This form gives a very intuitive interpretation of the directional
noise. First, we note that, to have small noise in the $i^{th}$ direction,
$\sigma_{i}(\boldsymbol{\Sigma})$ has to be small (cf. Section \ref{subsec:dir_noise_as_noniid}).
However, the sum of $1/\sigma_{i}^{2}(\boldsymbol{\Sigma})$ of the
noise in all directions, which should hence be large, is limited by
the quantity on the right side of Eq. (\ref{eq:decoupled_condition}).
This, in fact, explains why even with directional noise, we still
need to add noise in \emph{every direction} to guarantee differential
privacy. Consider the case when we set the noise in one direction
to be zero, and we have $\underset{\sigma\rightarrow0}{\lim}\frac{1}{\sigma}=\infty$,
which immediately violates the sufficient condition in Eq. (\ref{eq:decoupled_condition}).

From Eq. (\ref{eq:decoupled_condition}), the quantity $1/\sigma_{i}^{2}(\boldsymbol{\Sigma})$
is the inverse of the variance of the noise in the $i^{th}$ direction,
so we may think of it as the \emph{precision} measure of the query
answer in that direction. The intuition is that the higher this value
is, the lower the noise added in that direction, and, hence, the more
precise the query value in that direction is. From this description,
the constraint in Eq. (\ref{eq:decoupled_condition}) can be aptly
named as the \emph{precision budget}, and we immediately have the
following theorem. 
\begin{thm}
For the MVG mechanism with $\boldsymbol{\Psi}=\mathbf{I}_{n}$, the
precision budget is $(-\beta+\sqrt{\beta^{2}+8\alpha\epsilon})^{4}/(16\alpha^{4}n)$. 
\end{thm}
Finally, the remaining task is to determine the directions of the
noise and form $\boldsymbol{\Sigma}$ accordingly. To do so systematically,
we first decompose $\boldsymbol{\Sigma}$ by SVD as, 
\[
\boldsymbol{\Sigma}=\mathbf{W}_{\boldsymbol{\Sigma}}\boldsymbol{\Lambda}_{\boldsymbol{\Sigma}}\mathbf{W}_{\boldsymbol{\Sigma}}^{T}.
\]
This decomposition represents $\boldsymbol{\Sigma}$ by two components
\textendash{} the directions of the row-wise noise indicated by $\mathbf{W}_{\boldsymbol{\Sigma}}$,
and the variance of the noise indicated by $\boldsymbol{\Lambda}_{\boldsymbol{\Sigma}}$.
Since the precision budget only puts constraint upon $\boldsymbol{\Lambda}_{\boldsymbol{\Sigma}}$,
this decomposition allows us to freely chose any unitary matrix for
$\mathbf{W}_{\boldsymbol{\Sigma}}$ such that each column of $\mathbf{W}_{\boldsymbol{\Sigma}}$
indicates each independent direction of the noise.

Therefore, we present the following simple approach to design the
MVG mechanism with the unimodal directional noise: under a given precision
budget, allocate more precision to the directions of more importance.

Algorithm \ref{alg:mvg_design_unimodal} formalizes this procedure.
It takes as inputs, among other parameters, the \emph{precision allocation
strategy} $\boldsymbol{\theta}\in(0,1)^{m}$, and the \emph{directions}
$\mathbf{W}_{\boldsymbol{\Sigma}}\in\mathbb{R}^{m\times m}$. The
precision allocation strategy is a vector of size $m$, whose elements,
$\theta_{i}\in(0,1)$, corresponds to the importance of the $i^{th}$
direction indicated by the $i^{th}$ orthonormal column vector of
$\mathbf{W}_{\boldsymbol{\Sigma}}$. The higher the value of $\theta_{i}$,
the more important the $i^{th}$ direction is. Moreover, the algorithm
enforces that $\sum_{i=1}^{m}\theta_{i}\leq1$ to ensure that we do
not overspend the precision budget. The algorithm, then, proceeds
as follows. First, compute $\alpha$ and $\beta$ and, then, the precision
budget $P$. Second, assign precision to each direction based on the
precision allocation strategy. Third, derive the variance of the noise
in each direction accordingly. Then, compute $\boldsymbol{\Sigma}$
from the noise variance and directions, and draw a matrix-valued noise
from $\mathcal{MVG}_{m,n}(\mathbf{0},\boldsymbol{\Sigma},\mathbf{I})$.
Finally, output the query answer with additive matrix noise.

We make a remark here about choosing directions of the noise. As discussed
in Section \ref{sec:Directional-Noise}, any orthonormal set of vectors
can be used as the directions. The simplest instance of such set is
the the standard basis vectors, e.g. $\mathbf{e}_{1}=[1,0,0]^{T},\mathbf{e}_{2}=[0,1,0]^{T},\mathbf{e}_{3}=[0,0,1]^{T}$
for $\mathbb{R}^{3}$.

\begin{algorithm}
\textbf{{}{}{}Input:}{}{}{} (a) privacy parameters: $\epsilon,\delta$;
(b) the query function and its sensitivity: $f(\mathbf{X})\in\mathbb{R}^{m\times n},s_{2}(f)$;
(c) the precision allocation strategy $\boldsymbol{\theta}\in(0,1)^{m}:\left|\boldsymbol{\theta}\right|_{1}=1$;
and (d) the $m$ directions of the row-wise noise $\mathbf{W}_{\boldsymbol{\Sigma}}\in\mathbb{R}^{m\times m}$.

\begin{enumerate}
\item {}{}Compute $\alpha$ and $\beta$ (cf. Theorem \ref{thm:design_general}).
\item {}{}{}Compute the precision budget $P=\frac{(-\beta+\sqrt{\beta^{2}+8\alpha\epsilon})^{4}}{16\alpha^{4}n}$.

\item {}{}{}\textbf{{}for}{} $i=1,\ldots,m$: 
\begin{enumerate}
\item {}{}Set $p_{i}=\theta_{i}P$.
\item {}{}Compute the $i^{th}$ direction's variance, $\sigma_{i}(\boldsymbol{\Sigma})=1/\sqrt{p_{i}}$.
\end{enumerate}
\item {}{}{}Form the diagonal matrix $\boldsymbol{\Lambda}_{\boldsymbol{\Sigma}}=diag([\sigma_{1}(\boldsymbol{\Sigma}),\ldots,\sigma_{m}(\boldsymbol{\Sigma})])$.

\item {}{}{}Derive the covariance matrix: $\boldsymbol{\Sigma}=\mathbf{W}_{\boldsymbol{\Sigma}}\boldsymbol{\Lambda}_{\boldsymbol{\Sigma}}\mathbf{W}_{\boldsymbol{\Sigma}}^{T}$.
\item {}{}{}Draw a matrix-valued noise $\mathbf{Z}$ from $\mathcal{MVG}_{m,n}(\mathbf{0},\boldsymbol{\Sigma},\mathbf{I})$.

\end{enumerate}
\textbf{{}{}{}Output:}{}{}{} $f(\mathbf{X})+\mathbf{Z}$.

\caption{MVG mechanism with unimodal directional noise.\label{alg:mvg_design_unimodal}}
\end{algorithm}

\subsection{Equi-Modal Directional Noise \label{subsec:Equi-Modal-Directional-Noise}}

Next, we consider the type of directional noise of which the row-wise
noise and column-wise noise are distributed identically, which we
call the equi-modal directional noise. We recommend this type of directional
noise for a \emph{symmetric query function}, i.e. $f(\mathbf{X})=f(\mathbf{X})^{T}\in\mathbb{R}^{m\times m}$.
This covers a wide-range of query functions including the covariance
matrix \cite{RefWorks:249,RefWorks:194,RefWorks:178}, the kernel
matrix \cite{RefWorks:33}, the adjacency matrix of an undirected
graph \cite{RefWorks:329}, and the Laplacian matrix \cite{RefWorks:329}.
The motivation for this recommendation is that, for symmetric query
functions, any prior information about the rows would similarly apply
to the columns, so it is reasonable to use identical row-wise and
column-wise noise. Since Theorem \ref{thm:mvg_for_psd} considers
the \emph{symmetric} positive semi-definite matrix-valued query function,
the equi-modal directional noise can be used with both Theorem \ref{thm:design_general}
and Theorem \ref{thm:mvg_for_psd}. However, we note that to use the
equi-modal directional noise with Theorem \ref{thm:mvg_for_psd},
the matrix-valued query function also has to be positive semi-definite.

Formally, this type of directional noise imposes that $\boldsymbol{\Psi}=\boldsymbol{\Sigma}$.
Following a similar derivation to the unimodal type, we have the following
precision budget. 
\begin{thm}
\label{thm:equi_budget} For the MVG mechanism with $\boldsymbol{\Psi}=\boldsymbol{\Sigma}$,
the precision budget is $(-\beta+\sqrt{\beta^{2}+8\alpha\epsilon})^{2}/(4\alpha^{2})$
with Theorem \ref{thm:design_general}, or $(-\beta+\sqrt{\beta^{2}+8\omega\epsilon})^{2}/(4\omega^{2})$
with Theorem \ref{thm:mvg_for_psd}.
\end{thm}
We emphasize again that the choice between applying Theorem \ref{thm:design_general}
or Theorem \ref{thm:mvg_for_psd} for the precision budget depends
on the particular query function, as discussed in Sec. \ref{subsec:Comparative-Analysis}.
Then, following a similar procedure to the unimodal type, we present
Algorithm \ref{alg:mvg_design_equimodal} for the MVG mechanism with
the equi-modal directional noise. The algorithm follows the same steps
as Algorithm \ref{alg:mvg_design_unimodal}, except it derives the
precision budget from Theorem \ref{thm:equi_budget}, and draws the
noise from $\mathcal{MVG}_{m,m}(\mathbf{0},\boldsymbol{\Sigma},\boldsymbol{\Sigma})$.

\begin{algorithm}
\textbf{{}{}{}Input:}{}{}{} (a) privacy parameters: $\epsilon,\delta$;
(b) the query function and its sensitivity: $f(\mathbf{X})\in\mathbb{R}^{m\times m},s_{2}(f)$;
(c) the precision allocation strategy $\boldsymbol{\theta}\in(0,1)^{m}:\left|\boldsymbol{\theta}\right|_{1}=1$;
and (d) the $m$ noise directions $\mathbf{W}_{\boldsymbol{\Sigma}}\in\mathbb{R}^{m\times m}$.
\begin{enumerate}
\item {}{}{}Compute $\alpha$ and $\beta$ (cf. Theorem \ref{thm:design_general}),
or $\omega$ and $\beta$ (cf. Theorem \ref{thm:mvg_for_psd}) depending
on the query function (cf. Section \ref{subsec:Comparative-Analysis}).
\item {}{}{}Compute the precision budget $P=\frac{(-\beta+\sqrt{\beta^{2}+8\alpha\epsilon})^{2}}{4\alpha^{2}}$
or $P=\frac{(-\beta+\sqrt{\beta^{2}+8\omega\epsilon})^{2}}{4\omega^{2}}$,
according to the choice in step 1.
\item {}{}{}\textbf{{}for}{} $i=1,\ldots,m$: 
\begin{enumerate}
\item {}{}Set $p_{i}=\theta_{i}P$.
\item {}{}Compute the the $i^{th}$ direction's variance, $\sigma_{i}(\boldsymbol{\Sigma})=1/\sqrt{p_{i}}$.
\end{enumerate}
\item {}{}{}Form the diagonal matrix $\boldsymbol{\Lambda}_{\boldsymbol{\Sigma}}=diag([\sigma_{1}(\boldsymbol{\Sigma}),\ldots,\sigma_{m}(\boldsymbol{\Sigma})])$.

\item {}{}{}Derive the covariance matrix: $\boldsymbol{\Sigma}=\mathbf{W}_{\boldsymbol{\Sigma}}\boldsymbol{\Lambda}_{\boldsymbol{\Sigma}}\mathbf{W}_{\boldsymbol{\Sigma}}^{T}$.

\item {}{}{}Draw a matrix-valued noise $\mathbf{Z}$ from $\mathcal{MVG}_{m,m}(\mathbf{0},\boldsymbol{\Sigma},\boldsymbol{\Sigma})$.

\end{enumerate}
\textbf{{}{}{}Output:}{}{}{} $f(\mathbf{X})+\mathbf{Z}$.

\caption{MVG mechanism with equi-modal directional noise.\label{alg:mvg_design_equimodal}}
\end{algorithm}

\subsection{Sampling from $\mathcal{MVG}_{m,n}(\mathbf{0},\boldsymbol{\Sigma},\boldsymbol{\Psi})$
\label{subsec:Sampling-from-mvg}}

One remaining question on the practical implementation of the MVG
mechanism is how to efficiently draw the noise from $\mathcal{MVG}_{m,n}(\mathbf{0},\boldsymbol{\Sigma},\boldsymbol{\Psi})$.
Here, we present two methods to implement the $\mathcal{MVG}_{m,n}(\mathbf{0},\boldsymbol{\Sigma},\boldsymbol{\Psi})$
samplers from currently available samplers of other distributions.
The first is based on the equivalence between $\mathcal{MVG}_{m,n}(\mathbf{0},\boldsymbol{\Sigma},\boldsymbol{\Psi})$
and the multivariate Gaussian distribution, and the second is based
on the affine transformation of the i.i.d. normal distribution.

\subsubsection{Sampling via the Multivariate Gaussian}

This method uses the equivalence between $\mathcal{MVG}_{m,n}(\mathbf{0},\boldsymbol{\Sigma},\boldsymbol{\Psi})$
and the multivariate Gaussian distribution via the vectorization operator
$vec(\cdot)$, and the Kronecker product $\otimes$ \cite{RefWorks:208}.
The relationship is described by the following lemma \cite{RefWorks:282,RefWorks:279,RefWorks:367}. 
\begin{lem}
\label{lem:mvg_mn_equivalence}$\mathcal{X}\sim\mathcal{MVG}_{m,n}(\mathbf{M},\boldsymbol{\Sigma},\boldsymbol{\Psi})$
if and only if $vec(\mathcal{X})\sim\mathcal{N}_{mn}(vec(\mathbf{M}),\boldsymbol{\Psi}\otimes\boldsymbol{\Sigma})$,
where $\mathcal{N}_{mn}(vec(\mathbf{M}),\boldsymbol{\Psi}\otimes\boldsymbol{\Sigma})$
denotes the $mn$-dimensional multivariate Gaussian distribution with
mean $vec(\mathbf{M})$ and covariance $\boldsymbol{\Psi}\otimes\boldsymbol{\Sigma}$. 
\end{lem}
There are many available packages that implement the $\mathcal{N}_{mn}(vec(\mathbf{M}),\boldsymbol{\Psi}\otimes\boldsymbol{\Sigma})$
samplers \cite{RefWorks:227,RefWorks:228,RefWorks:229}, so this relationship
allows us to use them to build a sampler for $\mathcal{MVG}_{m,n}(\mathbf{0},\boldsymbol{\Sigma},\boldsymbol{\Psi})$.
To do so, we take the following steps: 
\begin{enumerate}
\item Convert the desired $\mathcal{MVG}_{m,n}(\mathbf{0},\boldsymbol{\Sigma},\boldsymbol{\Psi})$
into its equivalent $\mathcal{N}_{mn}(vec(\mathbf{M}),\boldsymbol{\Psi}\otimes\boldsymbol{\Sigma})$. 
\item Draw a sample from $\mathcal{N}_{mn}(vec(\mathbf{M}),\boldsymbol{\Psi}\otimes\boldsymbol{\Sigma})$. 
\item Convert the vectorized sample into its matrix form. 
\end{enumerate}
The computational complexity of this sampling method depends on the
multivariate Gaussian sampler used. Plus, the Kronecker product has
an extra complexity of $\mathcal{O}(m^{2}n^{2})$ \cite{RefWorks:451}.

\subsubsection{Sampling via the Affine Transformation of the i.i.d. Normal Noise}

\label{subsec:affine_tx} The second method to implement a sampler
for $\mathcal{MVG}_{m,n}(\mathbf{0},\boldsymbol{\Sigma},\boldsymbol{\Psi})$
is via the affine transformation of samples drawn i.i.d. from the
standard normal distribution, i.e. $\mathcal{N}(0,1)$. The transformation
is described by the following lemma \cite{RefWorks:279}. 
\begin{lem}
Let $\mathcal{N}\in\mathbb{R}^{m\times n}$ be a matrix-valued random
variable whose entries are drawn i.i.d. from the standard normal distribution
$\mathcal{N}(0,1)$. Then, the matrix $\mathcal{Z}=\mathbf{B}_{\boldsymbol{\Sigma}}\mathcal{N}\mathbf{B}_{\boldsymbol{\Psi}}^{T}$
is distributed according to $\mathcal{Z}\sim\mathcal{MVG}_{m,n}(\mathbf{0},\mathbf{B}_{\boldsymbol{\Sigma}}\mathbf{B}_{\boldsymbol{\Sigma}}^{T},\mathbf{B}_{\boldsymbol{\Psi}}\mathbf{B}_{\boldsymbol{\Psi}}^{T})$. 
\end{lem}
This transformation consequently allows the conversion between $mn$
samples drawn i.i.d. from $\mathcal{N}(0,1)$ and a sample drawn from
$\mathcal{MVG}_{m,n}(\mathbf{0},\mathbf{B}_{\boldsymbol{\Sigma}}\mathbf{B}_{\boldsymbol{\Sigma}}^{T},\mathbf{B}_{\boldsymbol{\Psi}}\mathbf{B}_{\boldsymbol{\Psi}}^{T})$.
To derive $\mathbf{B}_{\boldsymbol{\Sigma}}$ and $\mathbf{B}_{\boldsymbol{\Psi}}$
from given $\boldsymbol{\Sigma}$ and $\boldsymbol{\Psi}$ for $\mathcal{MVG}_{m,n}(\mathbf{0},\boldsymbol{\Sigma},\boldsymbol{\Psi})$,
we solve the two linear equations: $\mathbf{B}_{\boldsymbol{\Sigma}}\mathbf{B}_{\boldsymbol{\Sigma}}^{T}=\boldsymbol{\Sigma}$,
and $\mathbf{B}_{\boldsymbol{\Psi}}\mathbf{B}_{\boldsymbol{\Psi}}^{T}=\boldsymbol{\Psi}$,
and the solutions of these two equations can be acquired readily via
the Cholesky decomposition or SVD (cf. \cite{RefWorks:208}). We summarize
the steps for this implementation here using SVD: 
\begin{enumerate}
\item Draw $mn$ samples from $\mathcal{N}(0,1)$, and form a matrix $\mathcal{N}$. 
\item Let $\mathbf{B}_{\boldsymbol{\Sigma}}=\mathbf{W}_{\boldsymbol{\Sigma}}\boldsymbol{\Lambda}_{\boldsymbol{\Sigma}}^{1/2}$
and $\mathbf{B}_{\boldsymbol{\Psi}}=\mathbf{W}_{\boldsymbol{\Psi}}\boldsymbol{\Lambda}_{\boldsymbol{\Psi}}^{1/2}$,
where $\mathbf{W}_{\boldsymbol{\Sigma}},\boldsymbol{\Lambda}_{\boldsymbol{\Sigma}}$
and $\mathbf{W}_{\boldsymbol{\Psi}},\boldsymbol{\Lambda}_{\boldsymbol{\Psi}}$
are derived from SVD of $\boldsymbol{\Sigma}$ and $\boldsymbol{\Psi}$,
respectively. 
\item Compute the sample $\mathcal{Z}=\mathbf{B}_{\boldsymbol{\Sigma}}\mathcal{N}\mathbf{B}_{\boldsymbol{\Psi}}^{T}$. 
\end{enumerate}
The complexity of this method depends on that of the $\mathcal{N}(0,1)$
sampler used. Plus, there is an additional $\mathcal{O}(\max\{m^{3},n^{3}\})$
complexity from SVD \cite{RefWorks:451}.

\subsection{Complexity Analysis}

Comparing the complexity of the two sampling methods presented in
Section \ref{subsec:Sampling-from-mvg}, it is apparent that the choice
between the two methods depends primarily on the values of $m$ and
$n$. From this, we can show that Algorithm \ref{alg:mvg_design_unimodal}
is $\mathcal{O}(\min\{\max\{m^{3},n^{3}\},(mn)^{2}\}\})$, while Algorithm
\ref{alg:mvg_design_equimodal} is $\mathcal{O}(m^{3})$ as follows.

For Algorithm \ref{alg:mvg_design_unimodal}, the two most expensive
steps are the matrix multiplication in step 5 and the sampling in
step 6. The multiplication in step 5 is $\mathcal{O}(m^{3})$, while
the sampling in step 6 is $\mathcal{O}(\min\{\max\{m^{3},n^{3}\},(mn)^{2}\}\})$
if we pick the faster sampling methods among the two presented in
Section \ref{subsec:Sampling-from-mvg}. Hence, the complexity of
Algorithm \ref{alg:mvg_design_unimodal} is $\mathcal{O}(\min\{\max\{m^{3},n^{3}\},(mn)^{2}\}\})$.

For Algorithm \ref{alg:mvg_design_equimodal}, the two most expensive
steps are similarly the matrix multiplication in step 5 and the sampling
in step 6. The multiplication in step 5 is $\mathcal{O}(m^{3})$,
while the sampling in step 6 is $\mathcal{O}(m^{3})$ if we use the
affine transformation method (cf. Section \ref{subsec:Sampling-from-mvg}).
Hence, the complexity of Algorithm \ref{alg:mvg_design_equimodal}
is $\mathcal{O}(m^{3})$.

\section{The Design of the Precision Allocation Strategy}

Algorithm \ref{alg:mvg_design_unimodal} and Algorithm \ref{alg:mvg_design_equimodal}
take as an input the precision allocation strategy $\boldsymbol{\theta}\in(0,1)^{m}$:$\bigl|\boldsymbol{\theta}\bigr|_{1}=1$.
As discussed in Section \ref{subsec:Unimodal-Directional-Noise},
elements of $\boldsymbol{\theta}$ are chosen to emphasize how informative
or useful each direction is. The design of $\boldsymbol{\theta}$
to optimize the utility gain via the directional noise is an interesting
topic for future research. For example, one strategy we use in our
experiments is the \emph{binary allocation} strategy, i.e. give most
precision budget to the useful directions in equal amount, and give
the rest of the budget to the other directions in equal amount. This
strategy follows from the intuition that our prior knowledge only
tells us whether the directions are informative or not, but we \emph{do
not know} the granularity of the level of usefulness of these directions.

\subsection{Power-to-Noise Ratio (PNR)}

However, in other situations when we have more granular knowledge
about the directions, the problem of designing the best strategy can
be more challenging. Here, we propose a method to optimize the utility
with the MVG mechanism and directional noise based on maximizing the
\emph{power-to-noise ratio (PNR)} \cite{RefWorks:251,RefWorks:350,RefWorks:434}.
The formulation treats the matrix-valued query function output as
a random variable. Let us denote this matrix-valued random variable
of the query function as $\mathcal{S}\in\mathbb{R}^{m\times n}$.
Then, the output of the MVG mechanism, denoted by $\mathcal{Y}\in\mathbb{R}^{m\times n}$,
can be written as $\mathcal{Y}=\mathcal{S}+\mathcal{Z}$, where $\mathcal{Z}\sim\mathcal{MVG}_{m,n}(\mathbf{0},\boldsymbol{\Sigma},\boldsymbol{\Psi})$.
From this description, the PNR can be defined in statistical sense
based on the covariance of each random variable as (cf. \cite{RefWorks:251,RefWorks:350,RefWorks:434}),
\begin{equation}
PNR=\frac{signal+noise}{noise}=\left|\frac{cov(\mathcal{S})+cov(\mathcal{Z})}{cov(\mathcal{Z})}\right|,\label{eq:pnr_def}
\end{equation}
where the $cov(\cdot)$ operator indicates the covariance of the random
variable, and $\left|\cdot\right|$ is the matrix determinant. We
note that the determinant operation is necessary here since the covariance
of each random variable is a matrix, but the PNR is generally interpret
as a scalar value. We may use either the trace or determinant operator
for this purpose, but for mathematical simplicity which will be clear
later, we adopt the determinant here.

\subsection{Maximizing PNR under ($\epsilon,\delta$)-Differential Privacy Constraint}

\label{subsec:maximizing_PNR}

With the definition of PNR in Eq. \eqref{eq:pnr_def}, we need to
make a connection to the sufficient conditions in Theorem \ref{thm:design_general}
and Theorem \ref{thm:mvg_for_psd} to ensure that the MVG mechanism
guarantees ($\epsilon,\delta$)-differential privacy. For brevity,
throughout the subsequent discussion and unless otherwise stated,
we assume that we implement the MVG mechanism according to Theorem
\ref{thm:design_general}. This presents no limitation to our approach
since the same technique can be readily applied to Theorem \ref{thm:mvg_for_psd}
simply by changing the ($\epsilon,\delta$)-differential privacy constraint
based on its sufficient condition.

First, we can utilize the equivalence in Lemma \ref{lem:mvg_mn_equivalence}
and write the noise term as $cov(\mathcal{Z})=\boldsymbol{\Psi}\otimes\boldsymbol{\Sigma}\in\mathbb{R}^{mn\times mn}$.
Then, we formulate the constrained optimization problem of maximizing
PNR with ($\epsilon,\delta$)-differential privacy constraint as follows.
\begin{problem}
\label{prob:pnr_constrained_opt} Given a matrix-valued query function
with the covariance $cov(\mathcal{S})=\mathbf{K}_{f}\in\mathbb{R}^{mn\times mn}$,
find $\boldsymbol{\Sigma}\in\mathbb{R}^{m\times m}$ and $\boldsymbol{\Psi}\in\mathbb{R}^{n\times n}$
that optimize
\begin{align*}
\max_{\boldsymbol{\Sigma},\boldsymbol{\Psi}}\:\frac{|{\bf K}_{f}+\boldsymbol{\Psi}\otimes\boldsymbol{\Sigma}|}{|\boldsymbol{\Psi}\otimes\boldsymbol{\Sigma}|}\\
\mathrm{s.t.}\ \left\Vert \boldsymbol{\sigma}(\boldsymbol{\Sigma}^{-1})\right\Vert _{2}\left\Vert \boldsymbol{\sigma}(\boldsymbol{\Psi}^{-1})\right\Vert _{2}\leq & \frac{(-\beta+\sqrt{\beta^{2}+8\alpha\epsilon})^{2}}{4\alpha^{2}}.
\end{align*}
\end{problem}
This problem is difficult to solve, but the following relaxation allows
the problem to be solved analytically. Hence, we consider the following
relaxed problem.
\begin{problem}
\label{prob:relaxed_pnr_constrained_opt} Given the matrix-valued
query function with the covariance $cov(\mathcal{S})=\mathbf{K}_{f}\in\mathbb{R}^{mn\times mn}$,
find $\boldsymbol{\Sigma}\in\mathbb{R}^{m\times m}$ and $\boldsymbol{\Psi}\in\mathbb{R}^{n\times n}$
that optimize
\begin{align*}
\max_{\boldsymbol{\Sigma},\boldsymbol{\Psi}}\:\frac{|{\bf K}_{f}+\boldsymbol{\Psi}\otimes\boldsymbol{\Sigma}|}{|\boldsymbol{\Psi}\otimes\boldsymbol{\Sigma}|}\\
\mathrm{s.t.}\ \left\Vert \boldsymbol{\sigma}(\boldsymbol{\Sigma}^{-1})\right\Vert _{1}\left\Vert \boldsymbol{\sigma}(\boldsymbol{\Psi}^{-1})\right\Vert _{1}= & \frac{(-\beta+\sqrt{\beta^{2}+8\alpha\epsilon})^{2}}{4\alpha^{2}}.
\end{align*}
\end{problem}
The fact that the relaxation still satisfies the same differential
privacy as Problem \ref{prob:pnr_constrained_opt} can be readily
verified by the the norm inequality $\|\mathbf{x}\|_{2}\le\|\mathbf{x}\|_{1}$
\cite{RefWorks:208}. 

Before we delve into its solution, we first clarify the connection
between this optimization and the precision allocation strategy. Recall
from Section \ref{sec:Directional-Noise} that the covariances $\boldsymbol{\Sigma}$
and $\boldsymbol{\Psi}$ can be decomposed via the SVD into the directions
and variances. The inverses of the variances then correspond to the
singular values $\boldsymbol{\sigma}(\boldsymbol{\Sigma}^{-1})$ and
$\boldsymbol{\sigma}(\boldsymbol{\Psi}^{-1})$. This, hence, relates
the objective function in Problem \ref{prob:pnr_constrained_opt}
to its inequality constraint. Recall further that, for a given set
of directions of the noise, the precision allocation strategy $\boldsymbol{\theta}$
indicates the importance of each direction. More precisely, high value
of $\theta(i)$ for the $i^{th}$ direction means that the variance
of the noise in that direction would be small, i.e. large $\sigma_{i}(\boldsymbol{\Sigma}^{-1})$
or $\sigma_{i}(\boldsymbol{\Psi}^{-1})$. In other words, for a given
set of noise directions, the design of $\boldsymbol{\theta}$ is equivalently
done via the design of $\boldsymbol{\sigma}(\boldsymbol{\Sigma}^{-1})$
and $\boldsymbol{\sigma}(\boldsymbol{\Psi}^{-1})$. Since the optimization
in Problem \ref{prob:pnr_constrained_opt} or Problem \ref{prob:relaxed_pnr_constrained_opt}
directly concerns with the design of $\boldsymbol{\sigma}(\boldsymbol{\Sigma}^{-1})$
and $\boldsymbol{\sigma}(\boldsymbol{\Psi}^{-1})$, this completes
the connection between the optimization in Problem \ref{prob:pnr_constrained_opt}
or Problem \ref{prob:relaxed_pnr_constrained_opt} and the design
of the precision allocation strategy $\boldsymbol{\theta}$.

Next, we seek for a solution to the optimization in Problem \ref{prob:relaxed_pnr_constrained_opt}
in order to guide us to the optimal design of the precision allocation
strategy. The following theorem summarizes the solution.
\begin{thm}
\label{thm:pnr_max_sol} Consider the optimization in Problem \ref{prob:relaxed_pnr_constrained_opt}.
Let the SVD of ${\bf K}_{f}$ be given by ${\bf K}_{f}={\bf Q}{\bf \Lambda}_{f}{\bf Q}^{T}\in\mathbb{R}^{mn\times mn}$,
and let $\mathbf{\Lambda}_{\mathcal{Z}}\in\mathbb{R}^{mn\times mn}$
be a a diagonal matrix whose diagonal entries are positive and are
given by, 
\begin{equation}
[\mathbf{\Lambda}_{\mathcal{Z}}]_{ii}^{-1}=c-[{\bf \Lambda}_{f}]_{ii}^{-1},\label{eq:max_sigma_design}
\end{equation}
where $[\cdot]_{ii}$ indicates the $[i^{th},i^{th}]$-element of
the matrix, and the scalar $c$ is given by
\[
c=\frac{(-\beta+\sqrt{\beta^{2}+8\alpha\epsilon})^{2}+4\alpha^{2}\mathrm{tr}(\boldsymbol{\Lambda}_{f}^{-1})}{4mn\alpha^{2}}.
\]
Then, the solution to the optimization in Problem \ref{prob:relaxed_pnr_constrained_opt}
is given by,
\begin{equation}
\boldsymbol{\Psi}\otimes\boldsymbol{\Sigma}={\bf Q}^{T}\mathbf{\Lambda}_{\mathcal{Z}}{\bf Q}.\label{eq:opt_solution}
\end{equation}
\end{thm}
\begin{proof}
We observe that the optimization in Problem \ref{prob:relaxed_pnr_constrained_opt}
is in the form close to the \emph{water filling problem }in communication
systems design \cite[chapter 9]{RefWorks:344}. Hence, the goal is
to modify our optimization problem into the form solvable by the water
filling algorithm. We seek to find the optimal solution to the following
optimization problem. 
\[
\max_{\left\Vert \boldsymbol{\sigma}(\boldsymbol{\Sigma}^{-1})\right\Vert _{1}\left\Vert \boldsymbol{\sigma}(\boldsymbol{\Psi}^{-1})\right\Vert _{1}=d}\frac{|{\bf K}_{f}+\boldsymbol{\Psi}\otimes\boldsymbol{\Sigma}|}{|\boldsymbol{\Psi}\otimes\boldsymbol{\Sigma}|},
\]
where $d=\frac{(-\beta+\sqrt{\beta^{2}+8\alpha\epsilon})^{2}}{4\alpha^{2}}$.
By using multiplicative property of the determinant and the SVD of
${\bf K}_{f}={\bf Q}{\bf \Lambda}_{f}{\bf Q}^{T}$ we have that, 
\begin{align*}
\frac{\left|{\bf K}_{f}+\boldsymbol{\Psi}\otimes\boldsymbol{\Sigma}\right|}{\left|\boldsymbol{\Psi}\otimes\boldsymbol{\Sigma}\right|} & =\frac{\left|{\bf K}_{f}+\boldsymbol{\Psi}\otimes\boldsymbol{\Sigma}\right|}{\left|\boldsymbol{\Psi}\otimes\boldsymbol{\Sigma}\right|}\\
 & =\left|{\bf K}_{f}(\boldsymbol{\Psi}\otimes\boldsymbol{\Sigma})^{-1}+{\bf I}\right|\\
 & =\left|{\bf K}_{f}\right|\left|(\boldsymbol{\Psi}\otimes\boldsymbol{\Sigma})^{-1}+{\bf K}_{f}^{-1}\right|\\
 & =\left|{\bf K}_{f}\right|\left|(\boldsymbol{\Psi}\otimes\boldsymbol{\Sigma})^{-1}+{\bf Q}^{T}{\bf \Lambda}_{f}^{-1}{\bf Q}\right|\\
 & =\left|{\bf K}_{f}\right|\left|{\bf Q}^{T}\right|\left|{\bf Q}(\boldsymbol{\Psi}\otimes\boldsymbol{\Sigma})^{-1}{\bf Q}^{T}+{\bf Q}^{T}{\bf \Lambda}_{f}^{-1}\right|\left|{\bf Q}\right|\\
 & =\left|{\bf K}_{f}\right|\left|{\bf Q}(\boldsymbol{\Psi}\otimes\boldsymbol{\Sigma})^{-1}{\bf Q}^{T}+{\bf \Lambda}_{f}^{-1}\right|.
\end{align*}

Therefore, find the optimizer of Problem \ref{prob:relaxed_pnr_constrained_opt}
is equivalent to finding the optimizer of 
\begin{align}
 & \max_{\left\Vert \boldsymbol{\sigma}(\boldsymbol{\Psi}^{-1})\right\Vert _{1}\left\Vert \boldsymbol{\sigma}(\boldsymbol{\Sigma}^{-1})\right\Vert _{1}=d}\left|{\bf Q}(\boldsymbol{\Psi}\otimes\boldsymbol{\Sigma})^{-1}{\bf Q}^{T}+{\bf \Lambda}_{f}^{-1}\right|\nonumber \\
\stackrel{(a)}{=} & \max_{\boldsymbol{\Sigma},\boldsymbol{\Psi}:{\rm tr}\left(\boldsymbol{\Psi}^{-1}\otimes\boldsymbol{\Sigma}^{-1}\right)=d}\left|{\bf Q}(\boldsymbol{\Psi}\otimes\boldsymbol{\Sigma})^{-1}{\bf Q}^{T}+{\bf \Lambda}_{f}^{-1}\right|\nonumber \\
\stackrel{(b)}{=} & \max_{{\rm tr}\left({\bf Q}(\boldsymbol{\Psi}\otimes\boldsymbol{\Sigma})^{-1}{\bf Q}^{T}\right)=d}\left|{\bf Q}(\boldsymbol{\Psi}\otimes\boldsymbol{\Sigma})^{-1}{\bf Q}^{T}+{\bf \Lambda}_{f}^{-1}\right|\nonumber \\
\stackrel{(c)}{=} & \max_{\mathbf{\Lambda}_{\mathcal{Z}}:{\rm tr}\left(\mathbf{\Lambda}_{\mathcal{Z}}^{-1}\right)=d}\left|\mathbf{\Lambda}_{\mathcal{Z}}^{-1}+{\bf \Lambda}_{f}^{-1}\right|\nonumber \\
\stackrel{(d)}{\le} & \max_{\mathbf{\Lambda}_{\mathcal{Z}}:{\rm tr}\left(\mathbf{\Lambda}_{\mathcal{Z}}^{-1}\right)=d}\prod_{i=1}^{mn}\left([\mathbf{\Lambda}_{\mathcal{Z}}]_{ii}^{-1}+[{\bf \Lambda}_{f}]_{ii}^{-1}\right),\label{eq:final_opt_eq}
\end{align}
where (in)-equalities follow from the following properties.

\begin{itemize}

\item Step (a) uses the property of the Kronecker product that ${\rm tr}\left(\boldsymbol{\Psi}^{-1}\otimes\boldsymbol{\Sigma}^{-1}\right)={\rm tr}\left(\boldsymbol{\Psi}^{-1}\right)\cdot{\rm tr}\left(\boldsymbol{\Sigma}^{-1}\right)$
\cite{RefWorks:466}, and then uses the property of the trace that
$\mathrm{tr}(\boldsymbol{\Sigma}^{-1})=\left\Vert \boldsymbol{\sigma}(\boldsymbol{\Sigma}^{-1})\right\Vert _{1}$
and $\mathrm{tr}(\boldsymbol{\Psi}^{-1})=\left\Vert \boldsymbol{\sigma}(\boldsymbol{\Psi}^{-1})\right\Vert _{1}$
since $\boldsymbol{\Sigma}$ and $\boldsymbol{\Psi}$ are positive
definite \cite[p. 108]{RefWorks:208}, \cite{RefWorks:454,RefWorks:455}.

\item Step (b) follows from the identity of the Kronecker product
that $(\boldsymbol{\Psi}^{-1}\otimes\boldsymbol{\Sigma}^{-1})=(\boldsymbol{\Psi}\otimes\boldsymbol{\Sigma})^{-1}$
\cite[chapter 4]{RefWorks:278}, and the unitary invariance property
of the trace, i.e. $\mathrm{tr}(\mathbf{A})=\mathrm{tr}(\mathbf{A}\mathbf{Q}^{T}\mathbf{Q})=\mathrm{tr}(\mathbf{Q}\mathbf{A}\mathbf{Q}^{T})$
since ${\bf Q}^{T}{\bf Q}={\bf I}$. 

\item Step (c) follows from the unitary transformation ${\bf Q}(\boldsymbol{\Psi}\otimes\boldsymbol{\Sigma})^{-1}{\bf Q}^{T}=\mathbf{Q}(\mathbf{Q}^{T}\mathbf{\Lambda}_{\mathcal{Z}}\mathbf{Q})^{-1}\mathbf{Q}^{T}=\mathbf{Q}(\mathbf{Q}^{T}\mathbf{\Lambda}_{\mathcal{Z}}^{-1}\mathbf{Q})\mathbf{Q}^{T}=\mathbf{\Lambda}_{\mathcal{Z}}^{-1}$. 

\item Finally, step (d) uses the Hadamard's inequality (Lemma \ref{lem:hadamard_ineq}),
which holds with equality if and only if $\mathbf{\Lambda}_{\mathcal{Z}}^{-1}$
is a diagonal matrix \cite{RefWorks:344}.

\end{itemize}

The solution to the problem in Eq. \eqref{eq:final_opt_eq} is given
by the water filling solution \cite[chapter 9]{RefWorks:344} to be,
\[
[\mathbf{\Lambda}_{\mathcal{Z}}]_{ii}^{-1}=\left(c-[{\bf \Lambda}_{f}]_{ii}^{-1}\right)^{+},
\]
where $c$ is chosen such that 
\[
\sum_{i=1}^{mn}\left(c-[{\bf \Lambda}_{f}]_{ii}^{-1}\right)^{+}=d.
\]
Then, we can solve for $c$ by some elementary algebraic modifications
as
\[
c=\frac{d+\sum_{i=1}^{mn}[\boldsymbol{\Lambda}_{f}]_{ii}^{-1}}{mn}=\frac{d+\mathrm{tr}(\boldsymbol{\Lambda}_{f}^{-1})}{mn}.
\]
Substitute in the definition of $d$ and this concludes the proof.
\end{proof}
The solution in Theorem \ref{thm:pnr_max_sol} has a very intuitive
interpretation. Consider first the decomposition of the solution $\boldsymbol{\Psi}\otimes\boldsymbol{\Sigma}={\bf Q}^{T}\mathbf{\Lambda}_{\mathcal{Z}}{\bf Q}.$
This is a decomposition of the total covariance of the noise $cov(\mathcal{Z})$
into its directions specified by $\mathbf{Q}$ and the corresponding
variances specified by $\boldsymbol{\Lambda}_{\mathcal{Z}}$. Noticeably,
the directions $\mathbf{Q}$ are directly taken from the SVD of the
covariance of the query function, i.e. ${\bf K}_{f}={\bf Q}{\bf \Lambda}_{f}{\bf Q}^{T}$,
and the directional variances $\boldsymbol{\Lambda}_{\mathcal{Z}}$
are a function of $\boldsymbol{\Lambda}_{f}$. This matches our intuition
in Section \ref{subsec:Utility-Gain-via-dir-noise} \textendash{}
given the SVD of the query function, we know which directions of the
matrix-valued query function are more informative than the others,
so we should design the directional noise of the MVG mechanism to
minimize its impact on the informative directions.

In more detail, the solution in Theorem \ref{thm:pnr_max_sol} suggests
the following procedure for designing the MVG mechanism with directional
noise to maximize the utility.
\begin{enumerate}
\item Pick the noise directions from the SVD of the matrix-valued query
function.
\item To determine the noise variance for each direction, consider the singular
values of the matrix-valued query function as follows.
\begin{enumerate}
\item Suppose for the direction $\mathbf{q}_{i}$, the singular value of
the query function \textendash{} which indicates how informative that
direction is \textendash{} is $\sigma_{i}(\mathbf{K}_{f})$. Then,
we design the noise variance in that direction, $\sigma_{i}(\mathcal{Z})$
to be proportional to the inverse of the singular value of the query
function, i.e. $\sigma_{i}(\mathcal{Z})\propto\sigma_{i}^{-1}(\mathbf{K}_{f})$.
\item Furthermore, since the overall noise variance is constrained by the
($\epsilon,\delta$)-differential privacy condition, we design the
scalar $c$ to ensure that $\boldsymbol{\Lambda}_{\mathcal{Z}}$ satisfies
the constraint.
\item Finally, the proportional $\sigma_{i}(\mathcal{Z})\propto\sigma_{i}^{-1}(\mathbf{K}_{f})$
is achieved by Eq. \eqref{eq:max_sigma_design}.
\end{enumerate}
\item Compose the overall covariance of the noise via Eq. \eqref{eq:opt_solution}.
\end{enumerate}

\section{Experimental Setups}

\label{sec:Experiments}

\begin{table}
\begin{centering}
\begin{tabular}{|>{\centering}p{3cm}|>{\centering}p{3.7cm}|>{\centering}p{3.5cm}|>{\centering}p{3.7cm}|}
\hline 
 & Exp. I  & Exp. II  & Exp. III\tabularnewline
\hline 
\hline 
Task  & Regression  & $1^{st}$ P.C.  & Covariance est.\tabularnewline
\hline 
Dataset  & Liver \cite{RefWorks:322,RefWorks:413}  & Movement \cite{RefWorks:396}  & CTG \cite{RefWorks:322,RefWorks:412}\tabularnewline
\hline 
\# samples  & 248  & 10,176  & 2,126\tabularnewline
\hline 
\# features  & 6  & 4  & 21\tabularnewline
\hline 
Query $f(\mathbf{X})$  & $\mathbf{X}$  & $\mathbf{X}\mathbf{X}^{T}/n$  & $\mathbf{X}$\tabularnewline
\hline 
Query size  & $6\times248$  & $4\times4$  & $21\times2126$\tabularnewline
\hline 
Evaluation metric  & RMSE  & $\Delta\rho$ (Eq. (\ref{eq:delta_rho}))  & RSS (Eq. (\ref{eq:rss}))\tabularnewline
\hline 
MVG algorithm & 1  & 2  & 1\tabularnewline
\hline 
Source of directions  & Domain knowledge \cite{RefWorks:395} /SVD  & Inspection on data collection setup \cite{RefWorks:396}  & Domain knowledge \cite{RefWorks:416}\tabularnewline
\hline 
\end{tabular}
\par\end{centering}
\caption{The three experimental setups. \label{tab:exp_setups}}
\end{table}

We evaluate the proposed MVG mechanism on three experimental setups
and datasets. Table \ref{tab:exp_setups} summarizes our setups. In
all experiments, 100 trials are carried out and the average and 95\%
confidence interval are reported. These experimental setups are discussed
in detail here.

\subsection{Experiment I: Regression}

\subsubsection{Task and Dataset}

The first experiment considers the regression application on the Liver
Disorders dataset \cite{RefWorks:322,RefWorks:410}. The dataset derives
5 features from the blood sample of 345 patients. We leave out the
samples from 97 patients for testing, so the private dataset contains
248 patients. We follow the suggestion of Forsyth and Rada \cite{RefWorks:413}
by using these features to predict the average daily alcohol consumption
of the patients. All features and the teacher values are centered-adjusted
and are $\in[-1,1]$.

\subsubsection{Query Function and Evaluation Metric}

We perform the regression task in a differentially-private manner
via the identity query, i.e. $f(\mathbf{X})=\mathbf{X}$. Since regression
involves the teacher values, we treat them as a feature, so the query
size becomes $6\times248$. We use the kernel ridge regression (KRR)
\cite{RefWorks:33,RefWorks:231} as the regressor, and the root-mean-square
error (RMSE) \cite{RefWorks:51,RefWorks:33} as the evaluation metric.

\subsubsection{MVG Mechanism Design}

As discussed in Section \ref{subsec:Unimodal-Directional-Noise},
Algorithm \ref{alg:mvg_design_unimodal} is appropriate for the identity
query, so we employ it for this experiment. The $L_{2}$-sensitivity
of this query can be derived as follows. The query function is $f(\mathbf{X})=\mathbf{X}\in[-1,1]^{6\times248}$.
For neighboring datasets $\{\mathbf{X},\mathbf{X}'\}$, the $L_{2}$-sensitivity
is 
\[
s_{2}(f)=\sup_{\mathbf{X},\mathbf{X}'}\left\Vert \mathbf{X}-\mathbf{X}'\right\Vert _{F}=\sup_{\mathbf{X},\mathbf{X}'}\sqrt{\sum_{i=1}^{6}(x(i)-x'(i))^{2}}=2\sqrt{6}.
\]
To design the noise directions, we employ two methods as follows.

\paragraph{Via Domain Knowledge with Binary Allocation Strategy}

For this method, to identify the informative directions to allocate
the precision budget, we refer to the domain knowledge by Alatalo
et al. \cite{RefWorks:395}, which indicates that alanine aminotransferase
(ALT) is the most indicative feature for predicting the alcohol consumption
behavior. Additionally, from our prior experience working with regression
problems, we recognize that the teacher value (Y) is another important
feature to allocate more precision budget to. With this setup, we
use the standard basis vectors as the directions (cf. Section \ref{sec:Practical-Implementation})
and employ the following \emph{binary allocation} strategy.
\begin{itemize}
\item Allocate $\tau$\% of the precision budget to the two important features
(ALT and Y) by equal amount. 
\item Allocate the rest of the precision budget equally to the rest of the
features. 
\end{itemize}
We vary $\tau\in\{55,65,\ldots,95\}$ and report the best results\footnote{Note that, in the real-world deployment, this parameter selection
process should also be made private, as suggested by Chaudhuri and
Vinterbo \cite{RefWorks:417}.}. The rationale for this strategy is from the intuition that our prior
knowledge only tells us whether the directions are highly informative
or not, but we \emph{do not know} the granularity of the level of
usefulness of these directions. Hence, this strategy give most precision
budget to the useful directions in equal amount, and give the rest
of the budget to the other directions in equal amount.

\paragraph{Via Differentially-Private SVD with Max-PNR Allocation Strategy}

\label{par:DP-SVD+Max-PNR}

To illustrate the feasibility of the directional noise when prior
knowledge may not be available, we design another implementation of
the MVG mechanism with directional noise according to the method described
in Section \ref{subsec:Dir._noise_general}. This method derives the
directions from the SVD of the query function. However, since the
SVD itself can leak sensitive information about the dataset, the algorithm
for SVD also needs to be differentially-private. This means that we
need to allocate some privacy budget into this process. In this experiment,
we employ the differentially-private SVD algorithm in \cite{RefWorks:194},
and allocate 20\% of the total $\epsilon$ and $\delta$ privacy budget
to it. We pick this fraction based on an observation on our results
from other experiments, which show that the MVG mechanism can yield
good performance even when we pick the wrong informative directions
(cf. Section \ref{subsec:effect_of_noise_dir}). Hence, it is advantageous
to allocate more budget to the MVG mechanism than to the noise direction
derivation.

In addition, to illustrate the full capability of the MVG mechanism,
we employ the optimal precision allocation strategy presented in Section
\ref{subsec:maximizing_PNR}. However, since we consider the unimodal
directional noise, we only need to design $\boldsymbol{\Sigma}$ since
we set $\boldsymbol{\Psi}=\mathbf{I}$. Consequently, the term $\boldsymbol{\Psi}\otimes\boldsymbol{\Sigma}$
in Theorem \ref{thm:pnr_max_sol} can be replaced by simply $\boldsymbol{\Sigma}$,
and we can estimate $\mathbf{K}_{f}=\mathbf{X}\mathbf{X}^{T}\in\mathbb{R}^{6\times6}$,
which is the covariance among the features. This takes into account
the fact that the covariance among the samples are not considered
in the unimodal directional noise for this task (cf. Section \ref{subsec:Unimodal-Directional-Noise}).

\subsection{Experiment II: 1$^{st}$ Principal Component (1$^{st}$ P.C.)}

\subsubsection{Task and Dataset}

The second experiment considers the problem of determining the first
principal component ($1^{st}$ P.C.) from the principal component
analysis (PCA). This is one of the most popular problem both in machine
learning and differential privacy. Note that we only consider the
first principal component here for two reasons. First, many prior
works in differentially-private PCA algorithm consider this problem
or the similar problem of deriving a few major P.C. (cf. \cite{RefWorks:178,RefWorks:313,RefWorks:249}),
so this allows us to compare our approach to the state-of-the-art
approaches of a well-studied problem. Second, in practice, this method
for deriving the $1^{st}$ P.C. can be used iteratively along with
an orthogonal projection method to derive the rest of the principal
components (cf. \cite{RefWorks:414})\footnote{The iterative algorithm as a whole has to be differentially-private
as well, so this can be an interesting topic for future research.}.

We use the Movement Prediction via RSS (Movement) dataset \cite{RefWorks:396},
which consists of the radio signal strength measurement from 4 sensor
anchors (ANC\{0-3\}) \textendash{} corresponding to the 4 features
\textendash{} from 10,176 movement samples. The features are center-adjusted
and are $\in[-100,100]$.

\subsubsection{Query Function and Evaluation Metric}

We consider the setting presented in Example \ref{exa:cov_exp}, so
the query is $f(\mathbf{X})=\frac{1}{n}\mathbf{X}\mathbf{X}^{T}$.
Then, we use SVD to derive the $1^{st}$ P.C. from it. Hence, the
query size is $4\times4$.

We adopt the common quality metric traditionally used for P.C. \cite{RefWorks:33}
and also used by Dwork et al. \cite{RefWorks:249} \textendash{} the
\emph{captured variance} $\rho$. For a given P.C. $\mathbf{v}$,
the capture variance by $\mathbf{v}$ on the covariance matrix $\bar{\mathbf{S}}$
is defined as $\rho=\mathbf{v}^{T}\bar{\mathbf{S}}\mathbf{v}$. To
be consistent with other experiments, we use the absolute error in
$\rho$ as deviated from the maximum $\rho$ for a given $\bar{\mathbf{S}}$.
It is well-established that the maximum $\rho$ is equal to the largest
eigenvalue of $\bar{\mathbf{S}}$ (cf. \cite[Theorem 4.2.2]{RefWorks:208},
\cite{RefWorks:415}). Hence, this metric can be written concisely
as, 
\begin{equation}
\Delta\rho(\mathbf{v})=\lambda_{1}-\rho(\mathbf{v}),\label{eq:delta_rho}
\end{equation}
where $\lambda_{1}$ is the largest eigenvalue of $\bar{\mathbf{S}}$.
For the ideal, non-private case, the error would clearly be zero.

\subsubsection{MVG Mechanism Design}

As discussed in Section \ref{subsec:Equi-Modal-Directional-Noise},
Algorithm \ref{alg:mvg_design_equimodal} is appropriate for the covariance
query, so we employ it for this experiment. The $L_{2}$-sensitivity
of this query is derived as follows. The query function is $f(\mathbf{X})=\frac{1}{n}\mathbf{X}\mathbf{X}^{T}$,
where $\mathbf{X}\in[-100,100]^{4\times2021}$. For neighboring datasets
$\{\mathbf{X},\mathbf{X}'\}$, the $L_{2}$-sensitivity is 
\[
s_{2}(f)=\sup_{\mathbf{X},\mathbf{X}'}\frac{\left\Vert \mathbf{x}_{j}\mathbf{x}_{j}^{T}-\mathbf{x}_{j}'\mathbf{x}_{j}'^{T}\right\Vert _{F}}{10,176}=\frac{2\sqrt{\sum_{j=1}^{4^{2}}x_{j}(i)^{4}}}{10,176}=\frac{8\cdot10^{4}}{10,176}.
\]
To identify the informative directions to allocate the precision budget,
we inspect the data collection setup described by Bacciu et al. \cite{RefWorks:396},
and hypothesize that that two of the four anchors should be more informative
due to their proximity to the movement path (ANC0 \& ANC3). Hence,
we use the standard basis vectors as the directions (cf. Section \ref{sec:Practical-Implementation})
and allocate more precision budget to these two features using the
same strategy as in Exp. I.

Since this query function is positive semi-definite, we can use either
Theorem \ref{thm:design_general} or Theorem \ref{thm:mvg_for_psd}
to guarantee the same $(\epsilon,\delta)$-differential privacy. Therefore,
to illustrate the improvement achievable via the exploitation of the
structural characteristic of the matrix-valued query function, we
implement the MVG mechanism with both theorems.

\subsection{Experiment III: Covariance Estimation}

\subsubsection{Task and Dataset}

The third experiment considers the similar problem to Exp. II but
with a different flavor. In this experiment, we consider the task
of estimating the covariance matrix \emph{from the perturbed database}.
This differs from Exp. II in three major ways. First, for covariance
estimation, we are interested in every P.C., rather than just the
first one. Second, as mentioned in Exp. II, many previous works do
not consider every P.C. in their design, so the previous works for
comparison are different. Third, to give a different taste of our
approach, we consider the method of input perturbation for estimating
the covariance, i.e. we query the database itself and use it to compute
the covariance.

We use the Cardiotocography (CTG) dataset \cite{RefWorks:322,RefWorks:412},
which consists of 21 features from 2,126 fetal samples. All features
have the range of $[0,1]$.

\subsubsection{Query Function and Evaluation Metric}

We consider covariance estimation via input perturbation, so we use
the identity query, i.e. $f(\mathbf{X})=\mathbf{X}$. The query size
is $21\times2126$.

We adopt the captured variance as the quality metric similar to Exp.
II, but since we are interested in every P.C., we consider the \emph{residual
sum of square (RSS)} \cite{RefWorks:350} of every P.C. This is similar
to the total residual variance used by Dwork et al. \cite[p. 5]{RefWorks:249}.
Formally, given the perturbed database $\tilde{\mathbf{X}}$, the
covariance estimate is $\tilde{\mathbf{S}}=\frac{1}{n}\tilde{\mathbf{X}}\tilde{\mathbf{X}}^{T}$.
Let $\{\tilde{\mathbf{v}}_{i}\}$ be the set of P.C.'s derived from
$\tilde{\mathbf{S}}$, and the RSS is, 
\begin{equation}
RSS(\tilde{\mathbf{S}})=\sum_{i}(\lambda_{i}-\rho(\tilde{\mathbf{v}}_{i}))^{2},\label{eq:rss}
\end{equation}
where $\lambda_{i}$ is the $i^{th}$ eigenvalue of $\bar{\mathbf{S}}$
(cf. Exp. II), and $\rho(\tilde{\mathbf{v}}_{i})$ is the captured
variance of the $i^{th}$ P.C. derived from $\tilde{\mathbf{S}}$.
Clearly, in the non-private case, $RSS(\bar{\mathbf{S}})=0$.

\subsubsection{MVG Mechanism Design}

Since we consider the identity query, we employ Algorithm \ref{alg:mvg_design_unimodal}
for this experiment. The query function is the same as Exp. I, so
the $L_{2}$-sensitivity can be readily derived as $s_{2}(f)=\sup_{\mathbf{X},\mathbf{X}'}\sqrt{\sum_{i=1}^{21}(x(i)-x'(i))^{2}}=\sqrt{21}$.
To identify the informative directions to allocate the precision budget
to, we refer to the domain knowledge from Costa Santos et al. \cite{RefWorks:416},
which identifies three features to be most informative, i.e. fetal
heart rate (FHR), \%time with abnormal short term variability (ASV),
and \%time with abnormal long term variability (ALV). Hence, we use
the standard basis vectors as the directions and allocate more precision
budget to these three features using the similar strategy to that
in Exp. I.

\subsection{Comparison to Previous Works}

Since our approach falls into the category of basic mechanism, we
compare our work to four prior state-of-the-art basic mechanisms discussed
in Section \ref{subsec:Basic-Mechanisms}, namely, the Laplace mechanism,
the Gaussian mechanism, the Exponential mechanism, and the JL transform
method.

For Exp. I and III, since we consider the identity query, the four
previous works for comparison are the works by Dwork et al. \cite{RefWorks:195},
Dwork et al. \cite{RefWorks:186}, Blum et al. \cite{RefWorks:174},
and Upadhyay \cite{RefWorks:399}, for the four basic mechanisms respectively.

For Exp. II, we consider the $1^{st}$ P.C. As this problem has been
well-investigated, we compare our approach to the state-of-the-art
algorithms specially designed for this problem. These algorithms using
the four prior basic mechanisms are, respectively: Dwork et al. \cite{RefWorks:195},
Dwork et al. \cite{RefWorks:249}, Chaudhuri et al. \cite{RefWorks:178},
and Blocki et al. \cite{RefWorks:313}.

For all previous works, we use the parameter values as suggested by
the authors of the method, and vary the free variable before reporting
the best performance.

Finally, we recognize that some of these prior works have a different
privacy guarantee from ours, namely, $\epsilon$-differential privacy.
Nevertheless, we present these prior works for comprehensive coverage
of prior basic mechanisms, and we will keep this difference in mind
when discussing the results.

\section{Experimental Results \label{sec:Experimental-Results}}

\begin{table}
\begin{centering}
\begin{tabular}{|c|c|c|c|}
\hline 
Method  & $\epsilon$  & $\delta$  & RMSE ($\times10^{-2}$)\tabularnewline
\hline 
\hline 
Non-private  & -  & -  & 1.226\tabularnewline
\hline 
Random guess  & -  & -  & $\sim3.989$\tabularnewline
\hline 
MVG using the binary allocation strategy & 1.0  & $1/n$  & $1.624\pm0.026$\tabularnewline
\hline 
MVG using the max-PNR allocation strategy & 1.0 & $1/n$  & $1.611\pm0.046$\tabularnewline
\hline 
Gaussian (Dwork et al. \cite{RefWorks:186})  & 1.0  & $1/n$  & $1.913\pm0.069$\tabularnewline
\hline 
JL transform (Upadhyay \cite{RefWorks:399})  & 1.0  & $1/n$  & $1.682\pm0.015$\tabularnewline
\hline 
Laplace (Dwork et al. \cite{RefWorks:195})  & 1.0  & 0  & $2.482\pm0.189$\tabularnewline
\hline 
Exponential (Blum et al. \cite{RefWorks:174})  & 1.0  & 0  & $2.202\pm0.721$\tabularnewline
\hline 
\end{tabular}
\par\end{centering}
\caption{Results from Experiment I: regression. \label{tab:Ex1_results}}
\end{table}

\subsection{Experiment I: Regression}

Table \ref{tab:Ex1_results} reports the results for Exp. I. Here
are the key observations. 
\begin{itemize}
\item Compared to the non-private baseline, the MVG mechanisms using both
precision allocation strategies yield similar performance (difference
of .004 in RMSE). 
\item Compared to other $(\epsilon,\delta)$-basic mechanisms, i.e. Gaussian
and JL transform, the MVG mechanism provides better utility with the
same privacy guarantee (by .003 and .0006 in RMSE, respectively). 
\item Compared to other $\epsilon$-basic mechanisms, i.e. Laplace and Exponential,
the MVG mechanism provides \emph{significantly} better utility (up
to \textasciitilde{}150\%) with slightly weaker $(\epsilon,1/n)$-differential
privacy guarantee. 
\item Among the compared methods, the MVG mechanism using the directional
noise derived from SVD and the max-PNR allocation strategy has the
best performance.
\end{itemize}
Overall, the results from regression show the promise of the MVG mechanism.
Our approach can outperform all other $(\epsilon,\delta)$-basic mechanisms.
Although it provides a weaker privacy guarantee than other $\epsilon$-basic
mechanisms, it can provide considerably more utility (up to \textasciitilde{}150\%).
As advocated by Duchi et al. \cite{RefWorks:419} and Fienberg et
al. \cite{RefWorks:418}, this trade-off could be attractive in some
settings, e.g. critical medical situation. In addition, it shows that,
even when domain knowledge is not available to design the noise directions,
the process in Section \ref{par:DP-SVD+Max-PNR}, which uses differentially-private
SVD and the PNR maximization based on Section \ref{subsec:maximizing_PNR},
can provide an even better \textendash{} and, in fact, the best \textendash{}
performance.

\subsection{Experiment II: 1$^{st}$ Principal Component}

\begin{table}
\begin{centering}
\begin{tabular}{|c|c|c|c|}
\hline 
Method  & $\epsilon$  & $\delta$  & Error $\Delta\rho$ ($\times10^{-1})$\tabularnewline
\hline 
\hline 
Non-private  & -  & -  & $0$\tabularnewline
\hline 
Random guess  & -  & -  & $\sim4.370$\tabularnewline
\hline 
MVG with Theorem \ref{thm:design_general}  & 1.0  & $1/n$  & $2.138\pm0.192$\tabularnewline
\hline 
MVG with Theorem \ref{thm:mvg_for_psd}  & 1.0 & $1/n$ & $1.434\pm0.192$\tabularnewline
\hline 
Gaussian (Dwork et al. \cite{RefWorks:249})  & 1.0  & $1/n$  & $2.290\pm0.185$\tabularnewline
\hline 
JL transform (Blocki et al. \cite{RefWorks:313})  & 1.0  & $1/n$  & $2.258\pm0.186$\tabularnewline
\hline 
Laplace (Dwork et al. \cite{RefWorks:195})  & 1.0  & 0  & $2.432\pm0.177$\tabularnewline
\hline 
Exponential (Chaudhuri et al. \cite{RefWorks:178})  & 1.0  & 0  & $1.742\pm0.188$\tabularnewline
\hline 
\end{tabular}
\par\end{centering}
\caption{Results from Experiment II: $1^{st}$ principal component. The MVG
with Theorem \ref{thm:design_general} does not utilize the PSD structure
of the query function, whereas the MVG with Theorem \ref{thm:mvg_for_psd}
does utilize the PSD characteristic of the query function. \label{tab:Exp2_results}}
\end{table}

Table \ref{tab:Exp2_results} reports the results for Exp. II. Here
are the key observations. 
\begin{itemize}
\item Compared to the non-private baseline, the MVG mechanism has very small
error $\Delta\rho$ as low as $0.1434$, which is achieved when the
PSD structure of the query function is exploited by the mechanism
via Theorem \ref{thm:mvg_for_psd}. 
\item Compared to other $(\epsilon,\delta)$-basic mechanisms, i.e. the
Gaussian mechanism and the JL transform, the best MVG mechanism, which
exploits the PSD structure of the query via Theorem \ref{thm:mvg_for_psd},
provides significantly better utility (error of 0.1434 vs 0.2290/0.2258)
with the same privacy guarantee. 
\item Compared to other $\epsilon$-basic mechanisms, i.e. the Laplace and
Exponential mechanisms, the best MVG mechanism, which exploits the
PSD structure of the query via Theorem \ref{thm:mvg_for_psd}, also
yields higher utility (error of 0.1434 vs 0.2432/0.1742) with a slightly
weaker $(\epsilon,1/n)$-differential privacy guarantee.
\item Comparing the two MVG mechanism designs, the one with utilizes the
structural characteristic of the query function with Theorem \ref{thm:mvg_for_psd}
performs notably better (error of 0.1434 vs 0.2138) over its counterpart
which uses Theorem \ref{thm:design_general}.
\end{itemize}
Overall, the MVG mechanism which exploits the PSD structure of the
query via Theorem \ref{thm:mvg_for_psd} provides the best utility
among competing methods. Noticeably, without the utilization of the
PSD structure of the query function, the MVG mechanism which uses
Theorem \ref{thm:design_general} actually performs worse than the
state-of-the-art method by Chaudhuri et al. \cite{RefWorks:178}.
Though, it can be said that the method by Chaudhuri et al. \cite{RefWorks:178}
also utilizes the PSD property of the query function since the algorithm
is designed specifically for deriving the $1^{st}$ P.C. This reinforces
the notion that exploiting the structural characteristics of the matrix-valued
query function can provide utility improvement. 

\subsection{Experiment III: Covariance Estimation}

\begin{table}
\begin{centering}
\begin{tabular}{|c|c|c|c|}
\hline 
Method  & $\epsilon$  & $\delta$  & RSS ($\times10^{-2})$\tabularnewline
\hline 
\hline 
Non-private  & -  & -  & $0$\tabularnewline
\hline 
Random guess  & -  & -  & $\sim12.393$\tabularnewline
\hline 
MVG & 1.0  & $1/n$  & $6.657\pm0.193$\tabularnewline
\hline 
Gaussian (Dwork et al. \cite{RefWorks:186})  & 1.0  & $1/n$  & $7.029\pm0.216$\tabularnewline
\hline 
JL transform (Upadhyay \cite{RefWorks:399})  & 1.0  & $1/n$  & $6.718\pm0.229$\tabularnewline
\hline 
Laplace (Dwork et al. \cite{RefWorks:195})  & 1.0  & 0  & $7.109\pm0.211$\tabularnewline
\hline 
Exponential (Blum et al. \cite{RefWorks:174})  & 1.0  & 0  & $7.223\pm0.211$\tabularnewline
\hline 
\end{tabular}
\par\end{centering}
\caption{Results from Experiment III: covariance estimation. \label{tab:Exp3_results}}
\end{table}

Table \ref{tab:Exp3_results} reports the results for Exp. III. Here
are the key observations. 
\begin{itemize}
\item Compared to the non-private baseline, the MVG mechanism has very small
RSS error of $.06657$. 
\item Compared to other $(\epsilon,\delta)$-basic mechanisms, i.e. Gaussian
and JL transform, the MVG mechanism provides better utility with the
same privacy guarantee (.004 and .001 smaller RSS error, respectively). 
\item Compared to other $\epsilon$-basic mechanisms, i.e. Laplace and Exponential,
the MVG mechanism gives better utility with slightly weaker $(\epsilon,1/n)$-differential
privacy guarantee (.005 and .006 smaller RSS error, respectively). 
\end{itemize}
Overall, the MVG mechanism provides the best utility (smallest error).
When compared to other methods with stronger privacy guarantee, the
MVG mechanism can provide much higher utility. Again, we point out
that in some settings, the trade-off of weaker privacy for better
utility might be favorable \cite{RefWorks:418,RefWorks:419}, and
our approach provides the best trade-off.

\begin{figure*}
\begin{centering}
\includegraphics[scale=0.47]{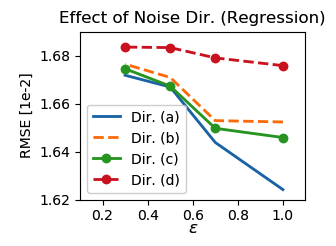}\includegraphics[scale=0.47]{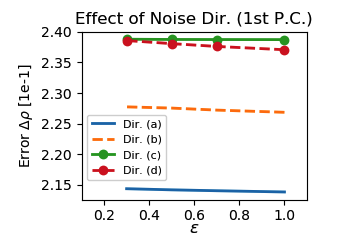}\includegraphics[scale=0.47]{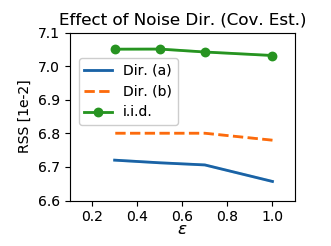} 
\par\end{centering}
\caption{Effect of noise directions on the utility under various $\epsilon$,
all with $\delta=1/n$. (Left) Exp. I: regression on the Liver dataset.
The four directions shown put more precision budget on the following
features based on the domain knowledge \cite{RefWorks:395}: (a) \{ALT,
Y\}, (b) \{ALT\}, (c) \{Y\}, (d) \{ALT, AST, Y\}. (Middle) Exp. II:
$1^{st}$ P.C. (with Theorem \ref{thm:design_general}) on the Movement
dataset. The four directions emphasize the following features: (a)
\{ANC0, ANC3\}, (b) \{ANC0\}, (c) \{ANC3\}, (d) \{ANC1, ANC2\}. (Right)
Exp. III: covariance estimation on the CTG dataset. The two directions
emphasize the two disjoint subsets of features: (a) \{FHR, ASV, ALV\},
(b) The rest of the features. \label{fig:Effect-of-noise-direction}}
\end{figure*}

\section{Discussion and Future Works}

\subsection{Effect of Noise Directions on the Utility}

\label{subsec:effect_of_noise_dir}

In Section \ref{subsec:Utility-Gain-via-dir-noise}, we discuss how
the choice of noise directions can affect the utility. Here, we investigate
this effect on the obtained utility in the three experiments. Figure
\ref{fig:Effect-of-noise-direction} depicts our results, and we discuss
our key observations here.

Figure \ref{fig:Effect-of-noise-direction}, Left, shows the direction
comparison from Exp. I with the directions derived from the domain
knowledge \cite{RefWorks:395}. We compare four choices of directions.
Direction (a), which uses the domain knowledge (ALT) and the teacher
label (Y), yields the best result when compared to: (b) using only
the domain knowledge (ALT), (c) using only the teacher label (Y),
and (d) using an arbitrary extra feature (ALT+Y+AST). In addition,
it is noticeable that the MVG mechanism still generally provides good
utility despite the less optimal directions being chosen.

Figure \ref{fig:Effect-of-noise-direction}, Middle, shows the direction
comparison from Exp. II using the MVG mechanism with Theorem \ref{thm:design_general}.
We compare four choices of directions. Direction (a), which makes
full use of the prior information (ANC0 and ANC3), performs best when
compared to: (b), (c) using only partial prior information (ANC0 or
ANC3), and (d) having the wrong priors completely (ANC1 and ANC2).

Figure \ref{fig:Effect-of-noise-direction}, Right, shows the comparison
from Exp. III. We compare three choices of directions. Direction (a),
which uses the domain knowledge (FHR, ASV, ALV3), gives the best performance
compared to: (d) using the completely wrong priors (all other features),
and (c) having no prior at all (i.i.d.).

Overall, these observations confirm the merit of both using directional
noise for the utility gain, and using prior information to properly
select the noise directions. However, we emphasize that, when prior
information is not available, we can still exploit the notion of directional
noise as shown in Section \ref{sec:Experimental-Results}.

\subsection{Directional Noise as a Generalized Subspace Analysis}

Directional noise provides utility gain by adding less noise in useful
directions and more in other directions. This has a connection to
subspace projection or dimensionality reduction, in which the non-useful
directions are simply removed. Clearly, the main difference between
the two is that, in directional noise, the non-useful directions are
kept, although are highly perturbed. However, despite being highly
perturbed, these directions may still be able to contribute to the
utility performance. With dimensionality reduction, which discards
them completely, we forfeit that additional information.

We test this hypothesis by running two additional regression experiments
(Exp. I) as follows. Given the same two important features used for
the MVG mechanism (ALT \& Y), we use the Gaussian mechanism \cite{RefWorks:186}
and the JL transform method \cite{RefWorks:399} to perform the regression
task using only these two features (i.e. discarding the rest of the
features). With $\epsilon=1$ and $\delta=\frac{1}{n}$, the results
are $(2.538\pm.065)\times10^{-2}$ and $(2.863\pm.022)\times10^{-2}$
of RMSE, respectively. Noticeably, these results are significantly
worse (i.e. larger error) than that of the MVG mechanism ($(1.624\pm.026)\times10^{-2}$
or $(1.611\pm.046)\times10^{-2}$), with the same privacy guarantee.
Specifically, by incorporating all features with directional noise
via the MVG mechanism, we can achieve over 150\% gain in utility over
the dimensionality reduction alternatives.

\subsection{Information Theoretic Perspective on the MVG Mechanism}

In this work, we analyze the MVG mechanism primarily from the perspectives
of differential privacy and signal processing. However, as discussed
by Cuff and Yu \cite{RefWorks:259}, there is a strong connection
between differential privacy and information theory. Therefore, in
this section, we discuss several advantageous properties of the MVG
mechanism from the perspective of information theory. First, we introduce
the necessary definitions, concepts, lemmas, and theorems. Second,
we discuss how maximizing PNR in Section \ref{subsec:maximizing_PNR}
is equivalent to maximizing the mutual information between the MVG
mechanism output and the true query answer. Third, we derive the optimal
type of query function to be used with the MVG mechanism. Finally,
we show that, under a certain condition, not only does the MVG mechanism
preserves ($\epsilon,\delta$)-differential privacy, but it also guarantees
the minimum correlation between the dataset and the output of the
mechanism.

\subsubsection{Background Concepts}

From the perspective of information theory, we view both the query
answer $f(\mathcal{X})$ and the mechanism output $\mathcal{MVG}(f(\mathcal{X}))$
as two random variables. One of the key measurements in information
theory concerns the mutual information, defined as follows.
\begin{defn}[Mutual information]
 Let $\mathcal{X}$ and $\mathcal{Y}$ be two matrix-valued random
variables with joint probability density given by $p_{\mathcal{X},\mathcal{Y}}(\mathcal{X},\mathcal{Y})$.
Then, the mutual information between $\mathcal{X}$ and $\mathcal{Y}$
is defined as 
\[
I(\mathcal{X};\mathcal{Y})=\mathbb{E}\left[\log\frac{p_{\mathcal{X},\mathcal{Y}}(\mathcal{X},\mathcal{Y})}{p_{\mathcal{X}}(\mathcal{X})p_{\mathcal{Y}}(\mathcal{Y})}\right].
\]
\end{defn}
The mutual information $I(\mathcal{X};\mathcal{Y}$) is a non-negative
quantity that measures the dependency between the two random variables
(cf. \cite{RefWorks:344}). In relation to our analysis, it is important
to discuss a property of the mutual information with respect to \emph{entropy}.
To the best of our knowledge, the notion of entropy for a matrix-valued
random variable has not been substantially studied. Therefore, we
define the terms here, and note that this topic in itself might be
of interest for future research in information theory and matrix analysis. 
\begin{defn}
Given a matrix-valued random variable $\mathcal{X}$ with the pdf
$p_{\mathcal{X}}(\mathcal{X})$, the \emph{differential entropy} is
\[
h(\mathcal{X})=\mathbb{E}\left[-\log p_{\mathcal{X}}(\mathcal{X})\right].
\]
For a joint pdf $p_{\mathcal{X},\mathcal{Y}}(\mathcal{X},\mathcal{Y})$,
the \emph{conditional entropy} is 
\[
h(\mathcal{X}\mid\mathcal{Y})=\mathbb{E}\left[-\log p_{\mathcal{X}\mid\mathcal{Y}}(\mathcal{X}\mid\mathcal{Y})\right].
\]
For two pdfs $p_{\mathcal{X}}(\mathcal{X})$ and $p_{\mathcal{Y}}(\mathcal{Y})$,
the \emph{relative entropy} is 
\[
D(p_{\mathcal{X}}\parallel p_{\mathcal{Y}})=\mathbb{E}\left[\log\frac{p_{\mathcal{X}}(\mathcal{X})}{p_{\mathcal{Y}}(\mathcal{Y})}\right].
\]
\end{defn}
Using the definitions of entropy, the mutual information can be written
as 
\begin{equation}
I(\mathcal{X};\mathcal{Y})=h(\mathcal{Y})-h(\mathcal{Y}\mid\mathcal{X}).\label{eq:mi_entropy}
\end{equation}
Next, let us define 
\[
Cov(vec(\mathcal{X}))=\mathbb{E}\left[[vec(\mathcal{X})-\mathbb{E}\left[vec(\mathcal{X})\right]][vec(\mathcal{X})-\mathbb{E}\left[vec(\mathcal{X})\right]]^{T}\right],
\]
and let $\preceq$ be the Loewner partial order defined as follows
\cite[chapter 7.7]{RefWorks:208}.
\begin{defn}
Given $\mathbf{A},\mathbf{B}\in\mathbb{R}^{n\times n}$, we write
$\mathbf{A}\preceq\mathbf{B}$ if $\mathbf{A}$ and $\mathbf{B}$
are symmetric and $\mathbf{B}-\mathbf{A}$ is positive semi-definite.
\end{defn}
Then, we present the following two lemmas. 
\begin{lem}[Maximum entropy principle]
\label{lem:max_entropy_principle}Given a matrix-valued random variable
$\mathcal{X}\in\mathbb{R}^{m\times n}$ with $Cov(vec(\mathcal{X}))\preceq\mathbf{K}$,
the entropy of $\mathcal{X}$ has the property 
\[
h(\mathcal{X})\leq\frac{1}{2}\log\left[(2\pi e)^{mn}\left|\mathbf{K}\right|\right],
\]
with equality if and only if $\mathcal{X}\sim\mathcal{MVG}_{m,n}(\mathbf{M},\boldsymbol{\Sigma},\boldsymbol{\Psi})$
with $\boldsymbol{\Psi}\otimes\boldsymbol{\Sigma}=\mathbf{K}$. 
\end{lem}
\begin{proof}
The proof parallels the proof in \cite{RefWorks:344}, where we use
the property that $D(f_{\mathcal{A}}\|g_{\mathcal{X}})$ is non-negative.
Let $g_{\mathcal{X}}$ be the pdf of MVG random variable such that
$Cov(vec(\mathcal{X}))=Cov(vec(\mathcal{A}))$. We observe that 
\begin{align*}
0\le D(f\|g) & =\mathbb{E}\left[\log\frac{f_{\mathcal{A}}(\mathcal{A})}{g_{\mathcal{X}}(\mathcal{A})}\right]=-h(\mathcal{A})-\mathbb{E}\left[\log g_{\mathcal{X}}(\mathcal{A})\right].
\end{align*}
Next, we show that $\mathbb{E}\left[\log g_{\mathcal{X}}(\mathcal{X})\right]=\mathbb{E}\left[\log g_{\mathcal{X}}(\mathcal{A})\right]$
as follows. 
\begin{align*}
\mathbb{E}\left[\log g_{\mathcal{X}}(\mathcal{A})\right] & =-\frac{1}{2}\mathbb{E}\left[\mathrm{tr}[\boldsymbol{\Sigma}^{-1}(\mathcal{A}-\mathbf{M})^{T}\boldsymbol{\Psi}^{-1}(\mathcal{A}-\mathbf{M})]\right]+\log\left((2\pi)^{mn/2}\left|\boldsymbol{\Sigma}\right|^{m/2}\left|\boldsymbol{\Psi}\right|^{n/2}\right)\\
 & =-\frac{1}{2}\mathbb{E}\left[(vec(\mathcal{A})-vec(\mathbf{M}))^{T}(\boldsymbol{\Psi}\otimes\boldsymbol{\Sigma})(vec(\mathcal{A})-vec(\mathbf{M}))\right]+\log\left((2\pi)^{mn/2}\left|\boldsymbol{\Psi}\otimes\boldsymbol{\Sigma}\right|\right)\\
 & =-\frac{1}{2}Cov(vec(\mathcal{X}))+\log\left((2\pi)^{\frac{mn}{2}}\left|\boldsymbol{\Psi}\otimes\boldsymbol{\Sigma}\right|\right)=\mathbb{E}\left[\log g_{\mathcal{X}}(\mathcal{X})\right].
\end{align*}
Using the equality $\mathbb{E}\left[\log g_{\mathcal{X}}(\mathcal{X})\right]=\mathbb{E}\left[\log g_{\mathcal{X}}(\mathcal{A})\right]$,
we have $h(\mathcal{A})\le-\mathbb{E}\left[\log g_{\mathcal{X}}(\mathcal{A})\right]=h(\mathcal{X})$,
which concludes the proof. 
\end{proof}
\begin{lem}
\label{lem:max_conditional_entropy} Given two matrix-valued random
variables $\mathcal{X}\in\mathbb{R}^{m\times n}$ and $\mathcal{Y}\in\mathbb{R}^{p\times q}$
with $Cov(vec(\mathcal{X})\mid\mathcal{Y})=\mathbf{K}$, the conditional
entropy of $\mathcal{X}$ upon $\mathcal{Y}$ has the property 
\[
h(\mathcal{X}\mid\mathcal{Y})\leq\frac{1}{2}\log\left[(2\pi e)^{mn}\left|\mathbf{K}\right|\right],
\]
with equality if and only if $(\mathcal{X},\mathcal{Y})$ is jointly
MVG with $Cov(vec(\mathcal{X})\mid\mathcal{Y})=\mathbf{K}$. 
\end{lem}
\begin{proof}
The proof follows that of Lemma \ref{lem:max_entropy_principle} by
working with the conditional probability version. 
\end{proof}

\subsubsection{Maximizing PNR is Maximizing Mutual Information}

Recall from Section \ref{subsec:maximizing_PNR} that we optimize
utility of the MVG mechanism using the PNR criterion. In this section,
we present how maximizing PNR is equivalent to maximizing the mutual
information between the MVG mechanism output and the true query answer.

Here, we can consider the true query answer $f(\mathcal{X})$, and
the mechanism output $\mathcal{MVG}(f(\mathcal{X}))$ as two random
variables. Then, we use the mutual information $I(f(\mathcal{X});\mathcal{MVG}(f(\mathcal{X})))$
to measure how informative the MVG mechanism output is, with respect
to the true answer. More importantly, we can optimize this measure
over the design flexibility of the MVG mechanism as follows: 
\begin{align}
\underset{\mathcal{Z}}{\arg\max}\:I(f(\mathcal{X});f(\mathcal{X})+\mathcal{Z})\label{eq:utility_optimization}\\
\text{s.t. }\ \mathcal{Z}\sim\mathcal{MVG}(\mathbf{0},{\bf \boldsymbol{\Sigma}},{\bf \boldsymbol{\Psi}}),\,\mathcal{Z}\in\mathbb{P},\nonumber 
\end{align}
where $\mathbb{P}$ is the set of all possible random variables $\mathcal{Z}$
that satisfy ($\epsilon,\delta$)-differential privacy according to
Theorem \ref{thm:design_general}, i.e. the conceptual shaded area
in Figure \ref{fig:A-conceptual-display}. If $f(\mathcal{X})$ is
Gaussian, the optimization can be simplified as follows. 
\begin{lem}
\label{lem:opt_utility_dir_noise}The optimization in Problem \ref{prob:pnr_constrained_opt}
is equivalent to 
\begin{align}
\underset{\boldsymbol{\Sigma},\boldsymbol{\Psi}}{\arg\max}\:\frac{|{\bf \boldsymbol{\Psi}}_{f}\otimes\boldsymbol{\Sigma}_{f}+\boldsymbol{\Psi}\otimes\boldsymbol{\Sigma}|}{|\boldsymbol{\Psi}\otimes\boldsymbol{\Sigma}|}\label{eq:utility_optimization2}\\
\text{s.t. }\ \left\Vert \boldsymbol{\sigma}(\boldsymbol{\Sigma}^{-1})\right\Vert _{2}\left\Vert \boldsymbol{\sigma}(\boldsymbol{\Psi}^{-1})\right\Vert _{2}\leq & \frac{(-\beta+\sqrt{\beta^{2}+8\alpha\epsilon})^{2}}{4\alpha^{2}},\nonumber 
\end{align}
where $\left|\cdot\right|$ is the matrix determinant; $\beta$, $\alpha$,
and $\epsilon$ are defined in Theorem \ref{thm:design_general};
and $f(\mathcal{X})\sim\mathcal{MVG}(\mathbf{0},{\bf \boldsymbol{\Sigma}}_{f},{\bf \boldsymbol{\Psi}}_{f})$. 
\end{lem}
\begin{proof}
Because $f(\mathcal{X})$ and $\mathcal{Z}$ are Gaussian, their mutual
information can be explicitly evaluated. With the relationship in
between the mutual information and entropy in Eq. \eqref{eq:mi_entropy}
and Lemma \ref{lem:max_entropy_principle}, the optimization in Eq.
\eqref{eq:utility_optimization} becomes, 
\begin{align*}
\underset{\boldsymbol{\Sigma},\boldsymbol{\Psi}}{\arg\max}\frac{1}{2}\log\left(\frac{|{\bf \boldsymbol{\Psi}}_{f}\otimes\boldsymbol{\Sigma}_{f}+\boldsymbol{\Psi}\otimes\boldsymbol{\Sigma}|}{|\boldsymbol{\Psi}\otimes\boldsymbol{\Sigma}|}\right)\\
\mathrm{s.t.}\ \left\Vert \boldsymbol{\sigma}(\boldsymbol{\Sigma}^{-1})\right\Vert _{2}\left\Vert \boldsymbol{\sigma}(\boldsymbol{\Psi}^{-1})\right\Vert _{2}\leq & \frac{(-\beta+\sqrt{\beta^{2}+8\epsilon})^{2}}{4\alpha^{2}}.
\end{align*}
By using the monotonicity property of the log function, this concludes
the proof. 
\end{proof}
Noticeably, comparing Lemma \ref{lem:opt_utility_dir_noise} to Problem
\ref{prob:pnr_constrained_opt}, if we replace $\boldsymbol{\Psi}_{f}\otimes\boldsymbol{\Sigma}_{f}$
with $\mathbf{K}_{f}$, we readily see that the two optimizations
are equivalent. In other words, maximizing PNR of the MVG mechanism
output is equivalent to maximizing the mutual information between
the MVG mechanism output and the true query answer.

\subsubsection{Optimal Query Function for the MVG Mechanism}

With an alternative view on the formulation of mutual information
in Eq. \eqref{eq:utility_optimization}, we can also study what type
of query function would yield the maximum mutual information between
the MVG mechanism output and the true query answer. Here, we prove
that such query functions have the form $f(\mathbf{X})=\sum_{i}g_{i}(\mathbf{x}_{i})$,
where $\mathbf{x}_{i}$ is the $i^{th}$ column (sample) of $\mathbf{X}$.
\begin{thm}
\label{thm:max_util} The mutual information between the true query
answer $f(\mathcal{X})$ and the query answer via the MVG mechanism
$\mathcal{MVG}(f(\mathcal{X}))$ is maximum when the query function
has the form $f(\mathbf{X})=\sum_{i}g_{i}(\mathbf{x}_{i})$, where
$\mathbf{x}_{i}$ is the $i^{th}$ column (sample) in $\mathbf{X}$. 
\end{thm}
\begin{proof}
For clarity, let $\mathcal{Y}=\mathcal{MVG}(f(\mathcal{X}))=f(\mathcal{X})+\mathcal{Z}$.
For any query function with $Cov(vec(f(\mathcal{X}))\preceq\mathbf{K}$,
write the mutual information in terms of the entropy, 
\[
\max_{f}I(f(\mathbf{X});\mathcal{Y})=\max_{f}h(\mathcal{Y})-h(\mathcal{Z}).
\]
From Lemma \ref{lem:max_entropy_principle}, the maximum is reached
when both $\mathcal{Y}$ and $\mathcal{Z}$ are Gaussian. Since $\mathcal{Z}\sim\mathcal{MVG}_{m,n}(\mathbf{0},\boldsymbol{\Sigma},\boldsymbol{\Psi})$,
$\mathcal{Y}$ is Gaussian when $f(\mathbf{X})$ is Gaussian. Then,
with the multivariate central limit theorem \cite{RefWorks:342},
and a mild condition on $g_{i}(\cdot)$, e.g. $g_{i}(\cdot)$ is bounded,
we arrive at the latter condition $f(\mathbf{X})=\sum_{i}g_{i}(\mathbf{x}_{i})$. 
\end{proof}
Theorem \ref{thm:max_util} states that to obtain the highest amount
of information on the true query answer from the MVG mechanism, the
query function should be in the form $f(\mathbf{X})=\sum_{i}g_{i}(\mathbf{x}_{i})$.
Conceptually, this condition states that the query function can be
decomposed into functions of individual records in the dataset. We
note that this can cover a wide range of real-world query functions.
Several statistical queries can be formulated in this form including
the average, covariance, and correlation. In machine learning, the
kernel classifiers and the loss functions can often be formulated
in this form \cite{RefWorks:33,RefWorks:350}.

\subsubsection{Proof of Minimum Privacy Leakage}

Apart from considering the mutual information between the MVG mechanism
output and the true query answer as in Eq. \eqref{eq:utility_optimization},
we can alternatively consider the mutual information between the MVG
mechanism output and the \emph{dataset} itself. This can be thought
of as the correlation between the dataset and the output of the MVG
mechanism. From the privacy point of view, clearly, we want this correlation
to be minimum. Here, we show that, if the query function has the form
$f(\mathbf{X})=\sum g_{i}(\mathbf{x}_{i})$, the MVG mechanism also
guarantees the minimum correlation between the dataset and the output
of the mechanism.

First, we treat the dataset $\mathcal{X}$ as a random variable. Then,
the claim is readily proved by the following theorem.
\begin{thm}
\label{thm:min_mi_priv}If the matrix-valued query function is of
the form $f(\mathbf{X})=\sum g_{i}(\mathbf{x}_{i})$, then for an
arbitrary positive definite matrix $\mathbf{K}$, 
\[
\min_{\mathcal{Z}:Cov(vec(\mathcal{Z}))=\mathbf{K}}I(\mathcal{X};f(\mathcal{X})+\mathcal{Z})=I(\mathcal{X};f(\mathcal{X})+\mathcal{Z}_{\mathcal{MVG}}),
\]
where $\mathcal{Z}_{\mathcal{MVG}}\sim\mathcal{MVG}_{m,n}(\mathbf{0},\boldsymbol{\Sigma},\boldsymbol{\Psi})$
such that $\boldsymbol{\Psi}\otimes\boldsymbol{\Sigma}=\mathbf{K}$. 
\end{thm}
\begin{proof}
First, we use the multivariate central limit theorem \cite{RefWorks:342}
to claim that $f(\mathcal{X})$ is Gaussian. Then, using the decomposition
in Eq. (\ref{eq:mi_entropy}), we have that 
\begin{align*}
I(\mathcal{X};f(\mathcal{X})+\mathcal{Z}) & =h(f(\mathcal{X})+\mathcal{Z})-h(f(\mathcal{X})+\mathcal{Z}|\mathcal{X})\\
 & =h(f(\mathcal{X})+\mathcal{Z})-h(f(\mathcal{X})+\mathcal{Z}|f(\mathcal{X}))\\
 & =I(f(\mathcal{X});f(\mathcal{X})+\mathcal{Z})\\
 & =h\left(f(\mathcal{X})\right)-h\left(f(\mathcal{X})|f(\mathcal{X})+\mathcal{Z}\right),
\end{align*}
where the second line follows from the independence of $f(\mathcal{X})$
and $\mathcal{Z}$, and from the fact that by conditioning on $\mathcal{X}$,
we have $h(f(\mathcal{X})+\mathcal{Z}\mid\mathcal{X})=h(\mathcal{Z})$.
Therefore, we have 
\begin{align*}
\min_{\mathcal{Z}:Cov(vec(\mathcal{Z}))=\mathbf{K}}I(\mathcal{X};f(\mathcal{X})+\mathcal{Z}) & =h\left(f(\mathcal{X})\right)-\max_{\mathcal{Z}:Cov(vec(\mathcal{Z}))=\mathbf{K}}h\left(f(\mathcal{X})|f(\mathcal{X})+\mathcal{Z}\right)\\
 & =h\left(f(\mathcal{X})\right)-h\left(f(\mathcal{X})|f(\mathcal{X})+\mathcal{Z}_{\mathcal{MVG}}\right),
\end{align*}
where the last step follows from Lemma \ref{lem:max_conditional_entropy}.
\end{proof}
Combining result of Theorem \ref{thm:design_general} and Theorem
\ref{thm:min_mi_priv}, we have a remarkable observation that not
only does the MVG mechanism guarantees $(\epsilon,\delta)$-differential
privacy but it also guarantees the minimum possible correlation (in
terms of mutual information) between the database and the output of
the mechanism. Finally, we note that the result of Theorem \eqref{thm:min_mi_priv}
is very general as it does not assume anything about the distribution
of the dataset $\mathcal{X}$ at all.

\section{Conclusion}

In this work, we study the matrix-valued query function in differential
privacy. We present the MVG mechanism that can fully exploit the structural
nature of the matrix-valued analysis. We present two sufficient conditions
for the MVG mechanism to guarantee $(\epsilon,\delta)$-differential
privacy \textendash{} one for the general matrix-valued query function,
and another for the positive semi-definite query function. As a result
of the sufficient conditions, we introduce the novel concept of directional
noise, which can be used to reduce the impact of the noise on the
utility of the query answer. Furthermore, we show how design the directional
noise such that the power-to-noise ratio (PNR) of the MVG mechanism
output is maximum. Finally, we evaluate our approach experimentally
for three matrix-valued query functions on three privacy-sensitive
datasets. The results show that our approach can provide significant
utility improvement over existing methods, and yield the utility performance
close to that of the non-private baseline while preserving $(1.0,1/n)$-differential
privacy. We hope that our work may inspire further development in
the problem of differential privacy under matrix-valued query functions.

\bibliographystyle{IEEEtran}
\bibliography{mvg_references}

\end{document}